\documentclass{article}

\usepackage[icml]{optional}
\usepackage[accepted]{icml_style/icml2021}

\usepackage[hidelinks]{hyperref}
\icmltitlerunning{Online Learning with Optimism and Delay}

\usepackage{microtype}
\usepackage{graphicx}
\usepackage{subfigure}
\usepackage{booktabs} %

\usepackage{graphicx}
\usepackage{amsthm,amsmath,amssymb,amsfonts}
\usepackage{mathtools}
\usepackage{algorithm,algorithmic}
\usepackage{xspace}
\usepackage{url}
\usepackage{enumitem}
\usepackage[dvipsnames]{xcolor}
\usepackage[font=small,skip=0pt]{caption} %
\usepackage{quoting}
\usepackage{subfigure}
\quotingsetup{vskip=0pt}
\usepackage[capitalize]{cleveref}  
\crefname{assumption}{Assumption}{Assumptions}
\creflabelformat{name}{#2\textup{#1}#3} %
\crefname{name}{}{} %
\crefname{equation}{}{}
\crefname{corollary}{Cor.}{Cors.}
\crefname{lemma}{Lem.}{Lems.}
\crefname{theorem}{Thm.}{Thms.}
\crefname{proposition}{Prop.}{Props.}
\crefname{assumption}{Assump.}{Assumps.}
\crefname{equation}{}{}
\crefname{section}{Sec.}{Secs.}
\crefname{appendix}{App.}{Apps.}
\usepackage{autonum} 
\newcommand{\lam}{\lambda}

\def\balign#1\ealign{\begin{align}#1\end{align}}
\def\baligns#1\ealigns{\begin{align*}#1\end{align*}}
\def\balignat#1\ealign{\begin{alignat}#1\end{alignat}}
\def\balignats#1\ealigns{\begin{alignat*}#1\end{alignat*}}
\def\bitemize#1\eitemize{\begin{itemize}#1\end{itemize}}
\def\benumerate#1\eenumerate{\begin{enumerate}#1\end{enumerate}}

\newenvironment{talign*}
 {\let\displaystyle\textstyle\csname align*\endcsname}
 {\endalign}
\newenvironment{talign}
 {\let\displaystyle\textstyle\csname align\endcsname}
 {\endalign}

\def\balignst#1\ealignst{\begin{talign*}#1\end{talign*}}
\def\balignt#1\ealignt{\begin{talign}#1\end{talign}}
\newcommand{\qtext}[1]{\quad\text{#1}\quad}

\let\originalleft\left
\let\originalright\right
\renewcommand{\left}{\mathopen{}\mathclose\bgroup\originalleft}
\renewcommand{\right}{\aftergroup\egroup\originalright}

\def\Holder{H\"older\xspace}

\def\tinycitep*#1{{\tiny\citep*{#1}}}
\def\tinycitealt*#1{{\tiny\citealt*{#1}}}
\def\tinycite*#1{{\tiny\cite*{#1}}}
\def\smallcitep*#1{{\scriptsize\citep*{#1}}}
\def\smallcitealt*#1{{\scriptsize\citealt*{#1}}}
\def\smallcite*#1{{\scriptsize\cite*{#1}}}

\def\mbf#1{\mathbf{#1}}
\def\mbb#1{\mathbb{#1}}

\def\tbf#1{\textbf{#1}}

\def\reals{\mathbb{R}} %
\def\R{\mathbb{R}}
\def\<{\left\langle} %
\def\>{\right\rangle}

\def\defeq{\triangleq} %
\def\half{\frac{1}{2}}
\def\texthalf{{\textstyle\frac{1}{2}}}

\newcommand{\boldone}{\mbf{1}} %
\newcommand{\boldzero}{\mbf{0}} %
\newcommand{\norm}[1]{\left\|{#1}\right\|} %
\newcommand{\onenorm}[1]{\norm{#1}_1} %
\newcommand{\infnorm}[1]{\norm{#1}_{\infty}} %
\newcommand{\dualnorm}[1]{\norm{#1}_{*}} %
\newcommand{\pnorm}[1]{\norm{#1}_{p}} %
\newcommand{\qnorm}[1]{\norm{#1}_{q}} %
\def\staticnorm#1{\|{#1}\|} %
\newcommand{\inner}[2]{\langle{#1},{#2}\rangle} %
\def\indic#1{\mbb{I}\left[{#1}\right]} %

\def\bigO#1{\mathcal{O}(#1)} %
\newcommand{\grad}{\nabla} %
\newcommand{\iid}{\textrm{i.i.d.}\xspace}

\newcommand{\subdiff}{\partial} %

\providecommand{\argmax}{\mathop\mathrm{arg max}} %
\providecommand{\argmin}{\mathop\mathrm{arg min}}

\providecommand{\sign}{\mathop\mathrm{sign}}
\providecommand{\conv}{\mathop\mathrm{conv}} %

\ifdefined\nonewproofenvironments\else
\ifdefined\ispres\else
\fi
\newcommand{\para}[1]{\tbf{#1\quad}}
\newcommand{\bi}{\begin{itemize}}
\newcommand{\ei}{\end{itemize}}

\newcommand{\be}{\begin{enumerate}}
\newcommand{\ee}{\end{enumerate}}

\newcommand{\ww}{\mathbf{w}}
\newcommand{\aww}{\ww^*} %
\newcommand{\wbar}{\bar{\ww}}

\newcommand{\worth}{\tilde{\ww}}

\newcommand{\omegaorth}{\tilde{\omega}}

\newcommand{\uu}{\mathbf{u}}

\newcommand{\wset}{\mathbf{W}} %
\newcommand{\uset}{\mathbf{U}} %
\newcommand{\vset}{\mathbf{V}} %
\newcommand{\g}{\mathbf{g}}
\newcommand{\h}{\mathbf{h}}
\newcommand{\rr}{\mathbf{r}}
\newcommand{\bb}{\mathbf{b}} %
\newcommand{\ab}{\mathbf{a}} %
\newcommand{\bbomd}[1]{\bb_{#1, O}} %
\newcommand{\bbftrl}[1]{\bb_{#1, F}} %
\newcommand{\abftrl}[1]{\ab_{#1, F}} %
\DeclareMathOperator{\diameter}{diam} %
\newcommand{\diam}[1]{\diameter({#1})} %

\newcommand{\M}{\mathcal{M}}

\newcommand{\x}{\mathbf{x}}
\newcommand{\y}{\mathbf{y}}
\newcommand{\X}{\mathbf{X}}
\newcommand{\vv}{\mathbf{v}}

\newcommand{\reg}{\psi}
\newcommand{\basis}{\mathbf{e}} %
\newcommand{\loss}{\ell} 
\newcommand{\surrloss}{\hat{\ell}} %
\newcommand{\pseudoloss}{\tilde{\ell}} %
\newcommand{\tildeg}{\tilde{\g}} %
\newcommand{\atildeg}{\tilde{\g}^*} %
\newcommand{\tildeq}{\tilde{q}} 
\newcommand{\metagrad}{\gamma} %

\newcommand{\simplex}{\triangle}
\newcommand{\simplexd}{\triangle_{d-1}}
\newcommand{\simplexm}{\triangle_{m-1}}
\newcommand{\orthantd}{\reals_+^d}

\newcommand{\obj}{F} %
\newcommand{\oobj}{\tilde{\obj}} %
\newcommand{\aoobj}{\tilde{\obj}^*} %
\newcommand{\regret}{\textup{Regret}}
\newcommand{\hregret}{\textup{HintRegret}}

\newcommand{\Breg}{\mathcal{B}}

\renewcommand{\norm}[1]{\left\Vert #1 \right\Vert}             %
\newcommand{\normt}[1]{\left\Vert #1 \right\Vert_2}             %
\newcommand{\maxentrynorm}[1]{\infnorm{#1}} %

\newcommand{\huber}{\textup{huber}}

\newtheorem{theorem}{Theorem}
\newtheorem{lemma}[theorem]{Lemma}
\newtheorem{corollary}[theorem]{Corollary}
\newtheorem{proposition}[theorem]{Proposition}

\newtheorem{assumption}{Assumption}

\theoremstyle{definition}
\newcommand{\SOOMD}{\cref{soomd}\xspace}
\newcommand{\ASOOMD}{\cref{asoomd}\xspace}
\newcommand{\DOOMD}{\cref{doomd}\xspace}
\newcommand{\varDOOMD}{\cref{var-doomd}\xspace}

\newcommand{\DORMP}{\cref{dorm+}\xspace}
\newcommand{\DORM}{\cref{dorm}\xspace}
\newcommand{\ODAFTRL}{\cref{odaftrl}\xspace}
\newcommand{\varODAFTRL}{\cref{var-odaftrl}\xspace}
\newcommand{\OAFTRL}{\cref{oaftrl}\xspace}
\newcommand{\ODFTRL}{\cref{odftrl}\xspace}
\newcommand{\OFTRL}{\cref{oftrl}\xspace}
\newcommand{\AdaHedgeD}{\cref{adahedged}\xspace}
\newcommand{\varAdaHedgeD}{\cref{var-adahedged}\xspace}
\newcommand{\DUB}{\cref{dub}\xspace}
\newcommand{\varDUB}{\cref{var-dub}\xspace}
\renewcommand{\norm}[1]{\|{#1}\|} %
\begin{document}
\twocolumn[
\icmltitle{Online Learning with Optimism and Delay}

\icmlsetsymbol{equal}{*}

\begin{icmlauthorlist}
\icmlauthor{Genevieve Flaspohler}{mit,whoi}
\icmlauthor{Francesco Orabona}{bu}
\icmlauthor{Judah Cohen}{aer}
\icmlauthor{Soukayna Mouatadid}{toronto}
\icmlauthor{Miruna Oprescu}{msr}
\icmlauthor{Paulo Orenstein}{impa}
\icmlauthor{Lester Mackey}{msr}
\end{icmlauthorlist}

\icmlaffiliation{mit}{%
Dept.\ of EECS, %
Massachusetts Institute of Technology
}
\icmlaffiliation{whoi}{%
Dept.\ of AOSE, 
Woods Hole Oceanographic Institution
}
\icmlaffiliation{bu}{%
Dept.\ of ECE, %
Boston University%
}
\icmlaffiliation{aer}{%
Atmospheric and Environmental Research%
}
\icmlaffiliation{toronto}{%
Dept.\ of CS, %
University of Toronto%
}
\icmlaffiliation{msr}{%
Microsoft Research New England%
}
\icmlaffiliation{impa}{%
Instituto de Matem\'{a}tica Pura e Aplicada%
}

\icmlcorrespondingauthor{Genevieve Flaspohler}{geflaspo@mit.edu}

\icmlkeywords{Online learning}

\vskip 0.3in
]

\printAffiliationsAndNotice{}  %

\begin{abstract}

Inspired by the demands of real-time climate and weather forecasting, we develop optimistic online learning algorithms that require no parameter tuning and have optimal regret guarantees under delayed feedback. Our algorithms---\DORM, \DORMP, and \AdaHedgeD---arise from a novel reduction of delayed online learning to optimistic online learning that reveals how optimistic hints can mitigate the regret penalty caused by delay. We pair this delay-as-optimism perspective with a new analysis of optimistic learning that exposes its robustness to hinting errors and a new meta-algorithm for learning effective  hinting strategies in the presence of delay. We conclude by benchmarking our algorithms on four subseasonal climate forecasting tasks, demonstrating low regret relative to state-of-the-art forecasting models.

\end{abstract}

\section{Introduction}
Online learning is a sequential decision-making paradigm in which a learner is pitted against a potentially adversarial environment \cite{shalev2007online,Orabona2019AMI}. At time $t$, the learner must select a play $\ww_t$ from some set of possible plays $\wset$. The environment then reveals the loss function $\loss_t$ and the learner pays the cost $\loss_t(\ww_t)$. The learner uses information collected in previous rounds to improve its plays in subsequent rounds. \textit{Optimistic} online learners additionally make use of side-information or ``hints'' about expected future losses to improve their plays. Over a period of length $T$, the goal of the learner is to minimize \textit{regret}, an objective that quantifies the performance gap between the learner and the best possible constant play in retrospect in some competitor set $\uset$: $\regret_T = \sup_{\uu \in \uset} \sum_{t=1}^{T} \loss_t(\ww_t) - \loss_t(\uu)$. Adversarial online learning algorithms provide robust performance in many complex real-world online prediction problems such as climate or weather forecasting.

In traditional online learning paradigms, the loss for round $t$ is revealed to the learner immediately at the end of round $t$. %
However, many real-world applications produce delayed feedback, i.e., the loss for round $t$ is not available until round $t+D$ for some delay period $D.$\footnote{Our initial presentation will assume constant delay $D$, but we provide extensions to variable and unbounded delays in \cref{sec:variable_delays}.}
Existing delayed online learning algorithms achieve optimal worst-case regret rates against adversarial loss sequences,  %
but each has drawbacks when deployed for real applications with short horizons $T$.
Some use only a small fraction of the data to train each learner \cite{weinberger2002delayed,joulani2013online}; others tune their parameters using uniform bounds on future gradients that are often challenging to obtain or overly conservative in applications \cite{mcmahan2014delay,quanrud2015online,joulani2016delay,korotin2020adaptive,hsieh2020multi}.
Only the concurrent work of \citet[Thm.~13]{hsieh2020multi} 
can make use of optimistic hints and only for the special case of unconstrained online gradient descent.

In this work, we aim to develop robust and practical algorithms for real-world delayed online learning.
To this end, we introduce three novel algorithms---\DORM, \DORMP, and \AdaHedgeD---that 
use every observation to train the learner, 
have no parameters to tune, 
exhibit optimal worst-case regret rates under delay, 
\emph{and} enjoy improved performance 
when accurate hints for unobserved losses are available.
We begin by formulating delayed online learning as a special case of optimistic online learning and use this ``delay-as-optimism'' perspective to develop: 
\begin{enumerate}[itemsep=-0.5ex]
    \item A formal reduction of delayed online learning to optimistic online learning (\cref{odftrl_is_oftrl,doomd_is_soomd}),
    \item The first optimistic tuning-free and self-tuning algorithms with optimal regret guarantees under delay (\DORM, \DORMP, and \AdaHedgeD),
    \item A tightening of standard optimistic online learning regret bounds that reveals the robustness of optimistic algorithms to inaccurate hints (\cref{oftrl_regret,soomd_regret}),
    \item The first general analysis of follow-the-regularized-leader (\cref{odftrl_regret,odaftrl_regret}) and online mirror descent algorithms (\cref{doomd_regret}) with optimism and delay, and
    \item The first meta-algorithm for learning a low-regret optimism strategy under delay (\cref{base-and-hinter-regret}).
\end{enumerate}

We validate our algorithms on the problem of subseasonal forecasting in \cref{sec:experiments}. Subseasonal forecasting---predicting precipitation and temperature 2-6 weeks in advance---is a crucial task for allocating water resources and preparing for weather extremes \citep{white2017potential}. Subseasonal forecasting presents several challenges for online learning algorithms.
First, real-time subseasonal forecasting suffers from delayed feedback: multiple forecasts are issued before receiving feedback on the first. Second, the regret horizons are short: a common evaluation period for semimonthly forecasting is one year, resulting in 26 total forecasts. Third, forecasters cannot have difficult-to-tune parameters in real-time, practical deployments. We demonstrate that our algorithms \DORM, \DORMP, and \AdaHedgeD sucessfully overcome these challenges and achieve consistently low regret compared to the best forecasting models. 

Our Python library for Optimistic Online Learning under Delay (PoolD) and experiment code are 
available at\\  \textcolor{blue}{\url{https://github.com/geflaspohler/poold}}.

\para{Notation} 
For integers $a,b$, we use the shorthand $[b] \defeq \{1,\dots, b\}$
and $\g_{a:b} \defeq \sum_{i=a}^b \g_i$.
We say a function $f$ is \emph{proper} if it is somewhere finite and never $-\infty$.
We let $\partial f(\ww) = \{\g \in \R^d :  f(\uu)\geq f(\ww) + \langle \g, \uu-\ww\rangle, \ \forall \uu \in \R^d\}$ denote the set of \emph{subgradients} of $f$ at $\ww \in \R^d$ 
and say $f$ %
is \emph{$\mu$-strongly convex} over a convex set $\wset \subseteq \mathop{\mathrm{int}} \mathop{\mathrm{dom}} f$ with respect to $\norm{\cdot}$ with dual norm $\dualnorm{\cdot}$ if $\forall \ww, \uu \in \wset$ and $\g \in \partial f(\ww)$, we have $f(\uu) \geq f(\ww) + \langle \g , \uu - \ww \rangle + \frac{\mu}{2} \| \ww - \uu \|^2$.
For differentiable $\psi$, we define the Bregman divergence 
$\Breg_{\reg}(\ww,\uu) \defeq \psi(\ww) - \psi(\uu) - \inner{\grad \psi(\uu)}{\ww - \uu}$.
We define $\diam{\wset}=\inf_{\ww,\ww'\in\wset}\norm{\ww-\ww'}$, $(r)_+ \defeq \max(r,0)$, and $\min(r,s)_+ \defeq (\min(r,s))_+$.

\section{Preliminaries: Optimistic Online Learning}
Standard online learning algorithms, such as follow the regularized leader (FTRL) and online mirror descent (OMD) achieve optimal worst-case regret against adversarial loss sequences \citep{Orabona2019AMI}. However, many loss sequences encountered in applications are not truly adversarial. \textit{Optimistic} online learning algorithms aim to improve performance when loss sequences are partially predictable, while remaining robust to adversarial sequences \citep[see, e.g.,][]{azoury2001relative, pmlr-v23-chiang12,rakhlin2013optimization, steinhardt2014adaptivity}.  In optimistic online learning, the learner is provided with a ``hint'' in the form of a pseudo-loss $\pseudoloss_t$ at the start of round $t$ that represents a guess for the true unknown loss. The online learner can incorporate this hint before making play $\ww_t$. 

In standard formulations of optimistic online learning, the convex pseudo-loss $\pseudoloss_t(\ww_t)$ is added to the standard FTRL or OMD regularized objective function and leads to optimistic variants of these algorithms: optimistic FTRL \citep[\OFTRL,][]{rakhlin2013online} and single-step optimistic OMD \citep[\SOOMD,][Sec.~7.2]{joulani2017modular}. Let $\tildeg_{t} \in \subdiff \pseudoloss_{t}(\ww_{t-1})$ and $\g_{t} \in \subdiff \loss_{t}(\ww_{t})$ denote subgradients of the pseudo-loss and true loss respectively. The inclusion of an optimistic hint leads to the following linearized update rules for play $\ww_{t+1}$: 
\begin{align}
    \ww_{t+1}& = \argmin_{\ww\in\wset} \,\inner{ \g_{1:t} + \tildeg_{t+1}}{\ww} + \lam \reg(\ww), \label[name]{oftrl}\tag{OFTRL} \\
    \ww_{t+1}& = \argmin_{\ww\in\wset} \,\inner{\g_t+\tildeg_{t+1}-\tildeg_t}{\ww} + \Breg_{\lam\reg}(\ww,\ww_t) \\
    \label[name]{soomd}\tag{SOOMD}
    &\qtext{with} \tildeg_0 = \boldzero \qtext{and arbitrary} \ww_0    
\end{align}
where $\tildeg_{t+1} \in \R^d$ is the hint subgradient, $\lam \geq 0$ is a regularization parameter, and $\reg$ is proper regularization function that is $1$-strongly convex with respect to a norm $\norm{\cdot}$. 
The optimistic learner enjoys reduced regret whenever the hinting error $\staticnorm{\g_{t+1}-\tildeg_{t+1}}_*$ is small 
\citep{rakhlin2013online,joulani2017modular}.
Common choices of optimistic hints include the last observed subgradient or average of previously observed subgradients~\citep{rakhlin2013online}. We note that the standard FTRL and OMD updates can be recovered by setting the optimistic hints to zero.
\section{Online Learning with Optimism and Delay} \label{odftrl_doomd}
In the delayed feedback setting with constant delay of length $D$, the learner only observes $(\loss_i)_{i=1}^{t-D}$ before making play $\ww_{t+1}$. In this setting, we propose counterparts of the OFTRL and SOOMD online learning algorithms, which we call \textit{optimistic delayed FTRL (\ODFTRL)} and \emph{delayed optimistic online mirror descent (\DOOMD)} respectively:
\begin{align}
    &\ww_{t+1} = \argmin_{\ww\in\wset}\, \inner{\g_{1:t-D} + \h_{t+1}}{\ww} + \lam \reg(\ww) \label[name]{odftrl}\tag{ODFTRL} \\
    &\ww_{t+1} = \argmin_{\ww\in\wset} \,\inner{\g_{t-D}+\h_{t+1}-\h_t}{\ww} + \Breg_{\lam\reg}(\ww,\ww_t) \\
    \label[name]{doomd}\tag{DOOMD}
    &\qtext{with} \h_0 \defeq \boldzero \qtext{and arbitrary} \ww_0,
\end{align}
for hint vector $\h_{t+1}$. Our use of the notation $\h_{t+1}$ instead of $\tildeg_{t+1}$ for the optimistic hint here is suggestive. Our regret analysis in \cref{odftrl_regret,doomd_regret} reveals that, instead of hinting only for the ``future`` missing loss $\g_{t+1}$, delayed online learners should uses hints $\h_t$ that guess at the summed subgradients of all delayed and future losses: $\h_t = \sum_{s=t-D}^t \tildeg_s$.

\subsection{Delay as Optimism} \label{sec:delay_as_opt}
To analyze the regret of the \ODFTRL and \DOOMD algorithms, we make use of the first key insight of this paper:
\begin{quoting}
    \emph{Learning with delay is a special case of learning with optimism.}
\end{quoting}
In particular, \ODFTRL and \DOOMD are instances of \OFTRL and \SOOMD respectively with a particularly ``bad'' choice of optimistic hint $\tildeg_{t+1}$  that deletes the unobserved loss subgradients $\g_{t-D+1:t}$. 

\begin{lemma}[\ODFTRL is \OFTRL with a bad hint]\label{odftrl_is_oftrl}
\ODFTRL is \OFTRL with $\tildeg_{t+1} = \h_{t+1} - \sum_{s=t-D+1}^t\g_s$.
\end{lemma}

\begin{lemma}[\DOOMD is \SOOMD with a bad hint]\label{doomd_is_soomd}
\DOOMD is \SOOMD with 
$\tildeg_{t+1} 
        = \tildeg_t + \g_{t-D} - \g_t + \h_{t+1}-\h_t
        = \h_{t+1} - \sum_{s=t-D+1}^t\g_s.$
\end{lemma}
The implication of this reduction of delayed online learning to optimistic online learning is that \emph{any} regret bound shown for undelayed \OFTRL or \SOOMD immediately yields a regret bound for \ODFTRL and \DOOMD under delay. As we demonstrate in the remainder of the paper, this novel connection between delayed and optimistic online learning allows us to bound the regret of optimistic, self-tuning, and tuning-free algorithms for the first time under delay.

Finally, it is worth reflecting on the key property of \OFTRL and \SOOMD that enables the delay-to-optimism reduction: each algorithm depends on $\g_t$ and $\tildeg_{t+1}$ only through the sum  $\g_{1:t} + \tildeg_{t+1}$.\footnote{For \SOOMD, $\g_t+\tildeg_{t+1}-\tildeg_t=\g_{1:t} + \tildeg_{t+1}-(\g_{1:t-1} + \tildeg_{t})$.} For the ``bad'' hints of  \cref{odftrl_is_oftrl,doomd_is_soomd}, these sums are observable even though $\g_{t}$ and $\tildeg_{t+1}$ are not separately observable at time $t$ due to delay.
A number of alternatives to \SOOMD have been proposed for optimistic OMD \citep{pmlr-v23-chiang12,rakhlin2013online,rakhlin2013optimization,kamalaruban2016improved}.  Unlike \SOOMD, these procedures all incorporate optimism in two steps, as in the updates
\begin{talign}
    &\ww_{t+1/2} = \argmin_{\ww\in\wset} \,\inner{\g_t}{\ww} + \Breg_{\lam\reg}(\ww,\ww_{t-1/2})
    \qtext{and}\\
    &\ww_{t+1} = \argmin_{\ww\in\wset} \,\inner{\tildeg_{t+1}}{\ww} + \Breg_{\lam\reg}(\ww,\ww_{t+1/2})\label{eq:two-step-oomd}
\end{talign}
described in \citet[Sec. 2.2]{rakhlin2013online}.
It is unclear how to reduce delayed OMD to an instance of one of these two-step procedures, as knowledge of the unobserved $\g_t$ is needed to carry out the first step. 

\subsection{Delayed and Optimistc Regret Bounds}

To demonstrate the utility of our delay-as-optimism perspective, we first present the following new regret bounds for \OFTRL and \SOOMD, proved in \cref{proof_oftrl_regret,proof_soomd_regret} respectively.
\begin{theorem}[\OFTRL regret] \label{oftrl_regret}
If $\reg$ is nonnegative, then, for all $\uu\in\wset$, the \OFTRL iterates $\ww_t$ satisfy
\begin{talign}
    &\regret_T(\uu) 
        \leq \lam \reg(\uu)+ 
        \frac{1}{\lam} \sum_{t=1}^T \huber(\dualnorm{\g_t - \tildeg_t}, \dualnorm{\g_t}).
\end{talign}
\end{theorem}

\begin{theorem}[\SOOMD regret]\label{soomd_regret}
If $\reg$ is differentiable and $\tildeg_{T+1} \defeq \boldzero$, then, $\forall\uu\in\wset$, the \SOOMD iterates $\ww_t$ satisfy
\begin{talign}\label{soomd_regret_bound}
\regret_T(\uu) 
    &\leq \Breg_{\lam \reg}(\uu,\ww_0) \,+ \\
    \frac{1}{\lam}&\sum_{t=1}^T
    \huber(\norm{\g_t - \tildeg_t}_*,
    \norm{\g_t + \tildeg_{t+1} - \tildeg_t}_*).
\end{talign}
\end{theorem}
Both results feature the robust Huber penalty \citep{huber1964}
\begin{talign}
\huber(x,y) 
    \defeq \half x^2 - \half(|x| - |y|)_+^2
    \leq \min(\half x^2, |y| |x|)
\end{talign}
in place of the more common squared error term $\half \dualnorm{\g_t - \tildeg_t}^2$.
As a result, 
\cref{oftrl_regret,soomd_regret} 
strictly improve the rate-optimal \OFTRL and \SOOMD regret bounds of %
\citet[Thm. 7.28]{rakhlin2013online,mohri2016accelerating,Orabona2019AMI} and \citet[Sec. 7.2]{joulani2017modular} 
by revealing a previously undocumented robustness to inaccurate hints $\tildeg_t$.
We will use this robustness to large hint error $\norm{\g_t - \tildeg_t}_*$ to establish optimal regret bounds under delay.

As an immediate consequence of this regret analysis and our delay-as-optimism perspective, we obtain the first general analyses of FTRL and OMD with optimism and delay.
\begin{theorem}[\ODFTRL regret] \label{odftrl_regret}
If $\reg$ is nonnegative, then, for all $\uu\in\wset$, the \ODFTRL iterates $\ww_t$ satisfy
\begin{talign}
&\regret_T(\uu)
    \leq \lam \reg(\uu) 
    +  \frac{1}{\lam} \sum_{t=1}^T \bbftrl{t} \qtext{for} \\
    &
    \bbftrl{t}
    \defeq \huber(\staticnorm{\h_t - \sum_{s=t-D}^t\g_s}_*, \dualnorm{\g_{t}}).  
    \label{bbf_def}
\end{talign}
\end{theorem}

\begin{theorem}[\DOOMD regret]\label{doomd_regret}
If $\reg$ is differentiable
and
$\h_{T+1} 
\defeq \g_{T-D+1:T}$, 
then, for all $\uu\in\wset$, the \DOOMD iterates $\ww_t$ satisfy
\begin{talign}
&\regret_T(\uu) 
    \leq \Breg_{\lam \reg}(\uu,\ww_0)
    + \frac{1}{\lam} \sum_{t=1}^T \bbomd{t} 
 \qtext{for} \\
\label{bbo_def}
&\bbomd{t}
\defeq\huber(\staticnorm{\h_t - \sum_{s=t-D}^t\g_s}_*, \dualnorm{\g_{t-D}+\h_{t+1}-\h_t}). 
\end{talign}
\end{theorem}
Our results show a compounding of regret due to delay: 
the $\bbftrl{t}$ term of \cref{odftrl_regret} is of size $\bigO{D+1}$ whenever $\dualnorm{\h_t} = \bigO{D+1}$, and the same holds for $\bbomd{t}$ of \cref{doomd_regret} if $\dualnorm{\h_{t+1}-\h_t}=\bigO{1}$.
An optimal setting of $\lambda$ therefore delivers $\bigO{\sqrt{(D+1)T}}$ regret, yielding the minimax optimal rate for adversarial learning under delay \citep{weinberger2002delayed}.
\cref{odftrl_regret,doomd_regret} also reveal the heightened value of optimism in the presence of delay: in addition to providing an effective guess of the future subgradient $\g_t$, an optimistic hint can approximate the missing delayed feedback ($\sum_{s=t-D}^{t-1} \g_s$) and thereby significantly reduce the penalty of delay. If, on the other hand, the hints are a poor proxy for the missing loss subgradients, the novel $\huber$ term ensures that we still only pay the minimax optimal $\sqrt{D+1}$ penalty for delayed feedback.

\para{Related work}  
A classical approach to delayed feedback in online learning is the so-called ``replication'' strategy in which $D+1$ distinct learners take turns observing and responding to feedback  \citep{weinberger2002delayed,joulani2013online,agarwal2012distributed,mesterharm2005line}. While minimax optimal in adversarial settings, this strategy has the disadvantage that each learner only sees $\frac{T}{D+1}$ losses and is completely isolated from the other replicates, exacerbating the problem of short prediction horizons. In contrast, we develop and analyze non-replicated delayed online learning strategies that use a combination of optimistic hinting and self-tuned regularization to mitigate the effects of delay while retaining optimal worst-case behavior. %

We are not aware of prior analyses of \DOOMD, and, 
to our knowledge, \cref{odftrl_regret} and its adaptive generalization \cref{odaftrl_regret} provide the first general analysis of delayed FTRL, apart from the concurrent work of \citet[Thm.~1]{hsieh2020multi}.
\citet[Thm.~13]{hsieh2020multi} and \citet[Thm.~2.1]{quanrud2015online} focus only on delayed gradient descent, \citet{korotin2020adaptive} study General Hedging, and
\citet[Thm.~4]{joulani2016delay} and \citet[Thm.~A.5]{quanrud2015online} study non-optimistic OMD under delay. 
\cref{odftrl_regret,doomd_regret,odaftrl_regret} strengthen these results from the literature which feature a sum of subgradient norms ($\sum_{s=t-D}^{t-1}\staticnorm{\g_s}_*$ or $D\staticnorm{\g_t}_*$) in place of $\staticnorm{\h_t - \sum_{s=t-D}^{t-1}\g_s}_*$.  Even in the absence of optimism, the latter can be significantly smaller: e.g., if the gradients $\g_s$ are \iid mean-zero vectors, the
former has size $\Omega(D)$ while the latter has expectation $\bigO{\sqrt{D}}$. 
In the absence of optimism, \citet{mcmahan2014delay} obtain a bound comparable to \cref{odftrl_regret} for the special case of one-dimensional unconstrained online gradient descent.

In the absence of delay, \citet{Cutkosky19} introduces meta-algorithms for imbuing learning procedures with optimism while remaining robust to inaccurate hints; however, unlike \OFTRL and \SOOMD, 
the procedures of \citeauthor{Cutkosky19} require separate observation of $\tildeg_{t+1}$ and each $\g_t$, making them unsuitable for our delay-to-optimism reduction.

\subsection{Tuning Regularizers with Optimism and Delay} \label{sec:tuning}
The online learning algorithms introduced so far all include a regularization parameter $\lam$. In theory and in practice, these algorithms only achieve low regret if the regularization parameter $\lam$ is chosen appropriately.  In standard FTRL, for example, one such setting that achieves optimal regret is $\lam = \sqrt{\frac{\sum_{t=1}^T \dualnorm{\g_t}^2}{\sup_{\uu\in\uset} \reg(\uu)}}$. This choice, however, cannot be used in practice as it relies on knowledge of all future unobserved loss subgradients. To make use of online learning algorithms, the tuning parameter $\lam$ is often set using coarse upper bounds on, e.g., the maximum possible subgradient norm. 
However, these bounds are often very conservative and lead to poor real-world performance. 

In the following sections, we introduce two strategies for tuning 
regularization with optimism and delay.
\cref{sec:dorm_dormp} introduces the \DORM and \DORMP algorithms, variants of \ODFTRL and \DOOMD that are \textit{entirely tuning-free}. \cref{sec:adahedged} introduces the \AdaHedgeD algorithm, an adaptive variant of \ODFTRL that is \textit{self-tuning}; a sequence of regularization parameters $\lam_t$ are set automatically using new, tighter bounds on algorithm regret. All three algorithms achieve the minimax optimal regret rate under delay, support optimism, and have strong real-world performance as shown in \cref{sec:experiments}. 

\section{Tuning-free Learning with Optimism \hspace{\columnwidth} and Delay} \label{sec:dorm_dormp}
Regret matching (RM) \citep{blackwell1956analog,hart2000simple} and regret matching+ (RM+) \citep{tammelin2015solving} are online learning algorithms that have strong empirical performance. RM was developed to find correlated equilibria in two-player games and is commonly used to minimize regret over the simplex. RM+ is a modification of RM designed to accelerate convergence and used to effectively solve the game of Heads-up Limit Texas Hold'em poker~\citep{Bowling145}.  
RM and RM+ support neither optimistic hints nor delayed feedback, and known regret bounds have a suboptimal scaling with respect to the problem dimension $d$ \citep{cesa2006prediction,orabona2015optimal}. To extend these algorithms to the delayed and optimistic setting and recover the optimal regret rate, we introduce our generalizations, \emph{delayed optimistic regret matching} (\DORM) 
\begin{talign} 
    \label[name]{dorm}\tag{DORM}
\ww_{t+1} 
    &= \worth_{t+1}/\inner{\boldone}{\worth_{t+1}} 
    \qtext{for} \\
\worth_{t+1}
    &\defeq \max(\mathbf{0}, (\rr_{1:t-D} + \h_{t+1}) / \lam)^{q-1}
\end{talign}      
and \emph{delayed optimistic regret matching+} (\DORMP) 
\begin{talign}
\label[name]{dorm+}\tag{DORM+}
\ww_{t+1} 
    &= \worth_{t+1}/\inner{\boldone}{\worth_{t+1}} 
    \text{\ for \ } 
\h_0 
    = %
\worth_0 
    \defeq \boldzero, 
    \\
\worth_{t+1}
    &\defeq \max\big(\mathbf{0}, \worth_{t}^{p-1} +  (\rr_{t-D}+\h_{t+1} - \h_t)/\lam\big)^{q-1},  
\end{talign}
Each algorithm makes use of an {instantaneous regret} vector $\rr_{t} \defeq \boldone\inner{\g_{t}}{\ww_{t}} - \g_{t}$ that quantifies the relative performance of each expert with respect to the play $\ww_t$ and the linearized loss subgradient $\g_t$.
The updates also include a parameter $q \geq 2$ and its conjugate exponent $p = q/(q-1)$ that is set to recover the minimax optimal scaling of regret with the number of experts (see \cref{dorm_dorm+_regret}). 
We note that \DORM and \DORMP recover the standard RM and RM+ algorithms when $D=0$, $\lam=1$, $q=2$, and $\h_t = \boldzero, \ \forall t$. 

\subsection{Tuning-free Regret Bounds}
To bound the regret of the \DORM and \DORMP plays, we prove that \DORM is an instance of \ODFTRL and \DORMP is an instance of \DOOMD. This connection enables us to immediately provide regret guarantees for these regret-matching algorithms under delayed feedback and with optimism.
We first highlight a remarkable property of \DORM and \DORMP that is the basis of their tuning-free nature. Under mild conditions: 
\begin{quote}
    The normalized \DORM and \DORMP iterates $\ww_t$ are \emph{independent} of the choice of regularization parameter $\lam$. 
\end{quote}
\begin{lemma}[\DORM and \DORMP are independent of $\lam$]
\label{dorm_dorm+_lambda_independent}
If the subgradient $\g_t$ and hint $\h_{t+1}$ only depend on $\lam$ through $(\ww_s,\lam^{q-1}\worth_s,\g_{s-1},\h_{s})_{s\leq t}$
and $(\ww_s,\lam^{q-1}\worth_s,\g_{s},\h_{s})_{s\leq t}$ respectively, 
then the \DORM and \DORMP iterates $(\ww_t)_{t\geq 1}$ are independent of the choice of $\lambda > 0$.
\end{lemma}
\cref{dorm_dorm+_lambda_independent}, proved in \cref{proof_dorm_dorm+_lambda_independent}, implies that \DORM and \DORMP are \textit{automatically} optimally tuned with respect to $\lam$, even when run with a default value of $\lam = 1$. Hence, these algorithms are  tuning-free, a very appealing property for real-world deployments of online learning. %

To show that \DORM and \DORMP also achieve optimal regret scaling under delay, we connect them to \ODFTRL and \DOOMD operating on the nonnegative orthant with a special surrogate loss $\surrloss_{t}$ (see \cref{proof_dorm_is_odftrl_dorm+_is_doomd} for our proof):
\begin{lemma}[\DORM is \ODFTRL and \DORMP is \DOOMD]\label{dorm_is_odftrl_dorm+_is_doomd}
The \DORM and \DORMP iterates are proportional to 
\ODFTRL and \DOOMD iterates respectively with $\wset \defeq \orthantd$, $\psi(\worth) = \half \pnorm{\worth}^2$, and loss $\surrloss_{t}(\worth) = \inner{\worth}{-\rr_{t}}$. 
\end{lemma}

\cref{dorm_is_odftrl_dorm+_is_doomd} enables the following optimally-tuned regret bounds for \DORM and \DORMP run with any choice of $\lam$: %
\begin{corollary}[\DORM and \DORMP regret]\label{dorm_dorm+_regret}
Under the assumptions of \cref{dorm_dorm+_lambda_independent}, 
for all $\uu \in \simplexd$ and any choice of $\lam > 0$, the \DORM and \DORMP iterates $\ww_t$ satisfy
\begin{align}
&\regret_T(\uu) 
    \leq
    \inf_{\lam > 0}
    \textstyle
        \frac{\lam}{2} \pnorm{\uu}^2
    + \frac{1}{\lam(p-1)}\sum_{t=1}^T\bb_{t,q} \\
    &=
    \textstyle
        \sqrt{\frac{\pnorm{\uu}^2}{2(p-1)}\sum_{t=1}^T \bb_{t,q}} 
    \leq 
    \textstyle
        \sqrt{\frac{d^{2/q}(q-1)}{2}\sum_{t=1}^T\bb_{t,\infty}}
\end{align}
where $\h_{T+1} \defeq \rr_{T-D+1:T}$ and, for each $c \in [2,\infty]$,
\begin{talign}
\bb_{t,c}
    \stackrel{(\tiny\DORM)}{=}
    &\ \huber(\staticnorm{\h_t - \sum_{s=t-D}^t\rr_s}_c,\staticnorm{\rr_{t}}_c) \qtext{and} \\
\bb_{t,c}
    \stackrel{(\tiny\DORMP)}{=}
    &\ \huber(\staticnorm{\h_t - \sum_{s=t-D}^t\rr_s}_c^2, \\
    &\hspace{.9cm} \staticnorm{\rr_{t-D}+\h_{t+1}-\h_t}_c).
\end{talign}
If, in addition, $q = \argmin_{q' \geq 2} d^{2/q'}(q'-1)$, then
$\regret_T(\uu) \leq \sqrt{(2\log_2(d)-1)\sum_{t=1}^T\bb_{t,\infty}}$.
\end{corollary}

\cref{dorm_dorm+_regret}, proved in \cref{proof_dorm_dorm+_regret}, suggests a natural hinting strategy for reducing the regret of \DORM and \DORMP: predict the sum of unobserved instantaneous regrets $\sum_{s=t-D}^t\rr_s$. We explore this strategy empirically in \cref{sec:experiments}. \cref{dorm_dorm+_regret} also highlights the value of the $q$ parameter in \DORM and \DORMP: using the easily computed value $q = \argmin_{q' \geq 2} d^{2/q'}(q'-1)$ yields the minimax optimal $\sqrt{\log_2(d)}$ dependence of regret on dimension \citep{cesa2006prediction,orabona2015optimal}. By \cref{dorm_is_odftrl_dorm+_is_doomd}, setting $q$ in this way is equivalent to selecting a robust $\half\pnorm{\cdot}^2$ regularizer \cite{Gentile03} for the underlying \ODFTRL and \DOOMD problems.

\para{Related work}
Without delay, \citet{Farina_Kroer_Sandholm_2021} independently developed optimistic versions of RM and RM+ by reducing them to \OFTRL and a two-step variant of optimistic OMD \cref{eq:two-step-oomd}. Unlike \SOOMD, this two-step optimistic OMD requires separate observation of $\tildeg_{t+1}$ and $\g_t$, making it unsuitable for our delay-as-optimism reduction and resulting in a different algorithm from \DORMP even when $D=0$.  In addition, their regret bounds and prior bounds for RM and RM+ (special cases of \DORM and \DORMP with $q=2$) have suboptimal regret when the dimension $d$ is large \citep{Bowling145,zinkevich2007regret}.

\section{Self-tuned Learning with Optimism  \hspace{\columnwidth} and Delay}  \label{sec:adahedged}

In this section, we analyze an adaptive version of \ODFTRL 
with time-varying regularization $\lam_t \reg$ and develop strategies for setting $\lam_t$ appropriately in the presence of optimism and delay.  
We begin with a new general regret analysis of optimistic delayed \emph{adaptive} FTRL (\ODAFTRL)
\begin{align}\label[name]{odaftrl}\tag{ODAFTRL}
    \ww_{t+1} = \argmin_{\ww\in\wset} \,\inner{\g_{1:t-D} + \h_{t+1}}{\ww} + \lam_{t+1} \reg(\ww)
\end{align}
where $\h_{t+1} \in \R^d$ is an arbitrary hint vector revealed before  $\ww_{t+1}$ is generated, $\reg$ is  $1$-strongly convex with respect to a norm $\norm{\cdot}$, and $\lam_t \geq 0$ is a regularization parameter.

\begin{theorem}[\ODAFTRL regret]\label{odaftrl_regret}
If $\psi$ is nonnegative and $\lam_t$ is non-decreasing in $t$, then,  $\forall\uu \in \wset$, the \ODAFTRL iterates $\ww_t$ satisfy
\begin{talign}
&\regret_T(\uu) 
    \leq 
        \lambda_{T}\reg(\uu) + 
        \sum_{t=1}^T 
    \min(\frac{ \bbftrl{t}}{\lam_{t}}, \abftrl{t}) \qtext{with} \\\label{abf_def}
    &\quad\bbftrl{t} 
        \defeq \huber(\staticnorm{\h_t - \sum_{s=t-D}^t\g_s}_*, \dualnorm{\g_{t}})
        \qtext{and}
        \\
    &\quad\abftrl{t} \defeq \diam{\wset} \min\big(\staticnorm{\h_t - \sum_{s=t-D}^t\g_s}_*, \dualnorm{\g_{t}} \big). 
\end{talign}
\end{theorem}

The proof of this result in \cref{proof_odaftrl_regret} builds on a new regret bound for undelayed optimistic adaptive FTRL (\OAFTRL).
In the absence of delay ($D=0$), \cref{odaftrl_regret} strictly improves  existing regret bounds \citep{rakhlin2013online,mohri2016accelerating,joulani2017modular} for \OAFTRL by providing tighter guarantees whenever the hinting error $\dualnorm{\h_t - \sum_{s=t-D}^t\g_t}$ is larger than the subgradient magnitude $\dualnorm{\g_t}$.
In the presence of delay, \cref{odaftrl_regret} benefits both from robustness to hinting error in the worst case and the ability to exploit accurate hints in the best case.
The bounded-domain factors $\abftrl{t}$ strengthen both standard \OAFTRL regret bounds and the concurrent bound of \citet[Thm.~1]{hsieh2020multi} when $\diam{\wset}$ is small and will enable us to design practical $\lam_t$-tuning strategies under delay without any prior knowledge of unobserved subgradients. We now turn to these self-tuning protocols.

\subsection{Conservative Tuning with Delayed Upper Bound}  \label{sec:upper-bound-lam}
Setting aside the $\ab_{t,F}$ bounded-domain factors in \cref{odaftrl_regret} for now, the adaptive sequence $\lam_{t} = \sqrt{\frac{\sum_{s=1}^{t} \bb_{s,F}}{\sup_{\uu\in\uset}\reg(\uu)}}$ is known to be a near-optimal minimizer of the \ODAFTRL regret bound \citep[Lemma 1]{mcmahan2017survey}.  However, this value is unobservable at time $t$. A common strategy is to play the conservative value $\lam_t = \sqrt{\frac{(D+1)B_0+\sum_{s=1}^{t-D-1} \bb_{s,F}}{\sup_{\uu\in\uset}\reg(\uu)}}$, where $B_0$ is a uniform upper bound on the unobserved $\bb_{s,F}$ terms \citep{joulani2016delay, mcmahan2014delay}.
In practice, this requires computing an \textit{a priori} upper bound on any subgradient norm that could possibly arise and often leads to extreme over-regularization (see \cref{sec:experiments}).  

As a preliminary step towards fully adaptive settings of $\lam_t$, we analyze in \cref{proof_dub_regret} a new \emph{delayed upper bound} (\DUB) tuning  strategy which relies only on observed $\bbftrl{s}$ terms and does not require upper bounds for future losses.  

\begin{theorem}[\DUB regret] \label{dub_regret}
Fix $\alpha > 0$, and, 
for $\abftrl{t},\bbftrl{t}$ as in \cref{abf_def}, 
consider the \emph{delayed upper bound} (\DUB) sequence 
\begin{talign}  \label[name]{dub}\tag{DUB}
    \lambda_{t+1}
        &= \frac{2}{\alpha} \max_{j\leq t-D-1} \abftrl{j-D+1:j} \\
        &+\frac{1}{\alpha} \sqrt{\sum_{i=1}^{t-D} \abftrl{i}^2 +2  \alpha \bbftrl{i}}.
\end{talign}
If $\psi$ is nonnegative, then, for all $\uu \in \wset$, the \ODAFTRL iterates $\ww_t$ satisfy
\begin{talign}
&\regret_T(\uu) 
    \leq \big(\frac{\reg(\uu)}{\alpha}+1\big) \\
&\big(2\max_{t\in[T]} \abftrl{t-D:t-1} +\sqrt{\sum_{t=1}^{T} \abftrl{t}^2 +2  \alpha \bbftrl{t}}\big).
\end{talign}
\end{theorem}
As desired, the \DUB setting of $\lam_t$ depends only on previously observed $\abftrl{t}$ and $\bbftrl{t}$ terms and achieves optimal regret scaling with the delay period $D$. However, the terms $\abftrl{t}$, $\bbftrl{t}$ are themselves potentially loose upper bounds for the instantaneous regret at time $t$. In the following section, we show how the \DUB regularization setting can be refined further to produce \AdaHedgeD adaptive regularization.

\subsection{Refined Tuning with AdaHedgeD}  \label{sec:adahedge-lam}
As noted by \citet{erven2011adaptive, rooij14a, Orabona2019AMI}, the effectiveness of an adaptive regularization setting $\lam_t$ that uses an upper bound on regret (such as $\bbftrl{t}$) relies heavily on the tightness of that bound. In practice, we want to set $\lam_t$ using as tight a bound as possible.  Our next result introduces a new tuning sequence that can be used with delayed feedback and is inspired by the popular AdaHedge algorithm \citep{erven2011adaptive}. It makes use of the tightened regret analysis underlying \cref{odaftrl_regret} to enable tighter settings of $\lam_t$ compared to \DUB, while still controlling algorithm regret (see proof in \cref{proof_adahedged_regret}).

\begin{theorem}[\AdaHedgeD regret] \label{adahedged_regret}
Fix $\alpha > 0$, and consider the
\emph{delayed AdaHedge-style} (\AdaHedgeD) sequence
\begin{talign}
\label[name]{adahedged}\tag{AdaHedgeD}
&\textstyle\lam_{t+1} 
    = \frac{1}{\alpha}\sum_{s=1}^{t-D} \delta_s 
    \quad\qtext{for}\\
&\textstyle\delta_t 
    \defeq \min(\obj_{t+1}(\ww_t,\lam_t) - \obj_{t+1}(\bar\ww_t,\lam_t),\ \ \inner{\g_t}{\ww_t - \bar\ww_t}, \\
    &\qquad\qquad\obj_{t+1}(\hat{\ww}_t,\lam_t) -    \obj_{t+1}(\bar\ww_t,\lam_t)
    +
    \inner{\g_t}{\ww_t - \hat{\ww}_t} )_+ \\
&\textstyle\text{with\quad}  \label{def-deltat}
\bar\ww_t 
    \defeq \argmin_{\ww\in\wset} \obj_{t+1}(\ww,\lam_t), \\
&\quad\quad\ \ \ \hat{\ww}_t
    \defeq \argmin_{\ww\in\wset}\textstyle
    \obj_{t+1}(\ww,\lam_t)\ + \\ 
&\hspace{1.8cm}\min(\frac{\dualnorm{\g_t}}{\dualnorm{\h_{t} - \g_{t-D:t}}},1)\inner{\h_{t} - \g_{t-D:t}}{\ww}, \\
&\textstyle\text{and\quad}
\obj_{t+1}(\ww,\lam_t) 
    \defeq \lam_t \reg(\ww) + \inner{\g_{1:t}}{\ww}.
\end{talign}
If $\psi$ is nonnegative, then, for all $\uu \in \wset$, the \ODAFTRL iterates satisfy
\begin{talign}
&\regret_T(\uu) 
    \leq \big( \frac{\reg(\uu)}{\alpha} + 1 \big) \\
& \big(2\max_{t \in [T]} \abftrl{t-D:t-1}
            + \sqrt{\sum_{t=1}^{T} \abftrl{t}^2 + 2 \alpha \bbftrl{t}}  \big).
\end{talign}
\end{theorem}
Remarkably, \cref{adahedged_regret}  yields a minimax optimal $\bigO{\sqrt{(D+1)T}+D}$ dependence on the delay parameter and nearly matches the \cref{odftrl_regret} regret of the optimal constant $\lam$ tuning. Although this regret bound is identical to that in  \cref{dub_regret}, in practice the $\lam_t$ values produced by \AdaHedgeD can be orders of magnitude smaller than those of \DUB, granting additional adaptivity. We evaluate the practical implications of these $\lam_t$ settings in \cref{sec:experiments}.

As a final note, when $\reg$ is bounded on $\uset$, we recommend choosing $\alpha = \sup_{\uu\in\uset} \reg(\uu)$ so that $\frac{\reg(\uu)}{\alpha} \leq 1$.
For negative entropy regularization $\reg(\uu) = \sum_{j=1}^d \uu_j \ln(\uu_j) + \ln(d)$ on the simplex $\uset=\wset=\simplexd$, this yields $\alpha = \ln(d)$ and a regret bound with minimax optimal $\sqrt{\ln(d)}$ dependence on $d$ \citep{cesa2006prediction,orabona2015optimal}.

\para{Related work}
Our \AdaHedgeD $\delta_t$ terms differ from standard AdaHedge increments \citep[see, e.g.,][Sec.~7.6]{Orabona2019AMI} due to the accommodation of delay, the incorporation of optimism, and the inclusion of the final two terms in the $\min$. These non-standard terms are central to reducing the impact of delay on our regret bounds.
Prior and concurrent approaches to adaptive tuning under delay do not incorporate optimism and require an explicit upper bound on all future subgradient norms, a quantity which is often difficult to obtain or very loose %
\citep{mcmahan2014delay,joulani2016delay,hsieh2020multi}.
Our optimistic algorithms, \DUB and \AdaHedgeD, admit comparable regret guarantees (\cref{dub_regret,adahedged_regret}) but require no prior knowledge of future subgradients.
\section{Learning to Hint with Delay} \label{sec:hinting}
As we have seen, optimistic hints play an important role in online learning under delay: effective hinting can counteract the increase in regret under delay.  In this section, we consider the problem of choosing amongst several competing hinting strategies. We show that this problem can again be treated as a delayed online learning problem.  In the following, we will call the original online learning problem the ``base problem'' and the learning-to-hint problem the ``hinting  problem.''

Suppose that, at time $t$, we observe the hints $\tildeg_t$ of $m$ different hinters arranged into a $d \times m$ matrix $H_t$.
Each column of $H_t$ is one hinter's best estimate of the sum of missing loss subgradients $\g_{t-D:t}$.
Our aim is to output a sequence of combined hints $\h_t(\omega_t) \defeq H_t \omega_t$ with low regret relative to the best constant combination strategy $\omega \in \Omega \defeq \simplexm$ in hindsight.
To achieve this using delayed online learning, we make use of a convex loss function $l_t(\omega)$ for the hint learner that upper bounds the base learner regret.

\begin{assumption}[Convex regret bound] \label{base_assumptions}
For any hint sequence $(\h_t)_{t=1}^T$ and $\uu\in\Omega$,
the base problem admits the regret bound  $\regret_T(\uu) \leq C_0(\uu) + C_1(\uu) \sqrt{\sum_{t=1}^T f_t(\h_t)}$
 for $C_1(\uu)\geq 0$ and convex functions $f_t$ independent of $\uu$.
\end{assumption}
As we detail in \cref{proof_adahedged-example}, \cref{base_assumptions} holds for all of the learning algorithms introduced in this paper. For example, by \cref{dorm_dorm+_regret}, if the base learner is \DORM, we may choose
$C_0(\uu) = 0$, $C_1(\uu) = \sqrt{\frac{\pnorm{\uu}^2}{2(p-1)}},$
and the $\bigO{D+1}$ convex function
$
f_t(\h_t) = \norm{\rr_{t}}_{q}  \norm{\h_t -\sum_{s=t-D}^t\rr_s}_{q}
    \geq \bb_{t,q}
$.\footnote{The alternative choice $f_t(\h_t) = \half\staticnorm{\h_t - \sum_{s=t-D}^t\rr_s}^2_{q}$ also bounds regret but may have size $\Theta((D+1)^2)$.}

For any base learner satisfying \cref{base_assumptions}, 
we choose $l_t(\omega) = f_t(H_t \omega)$ as our hinting loss, use the tuning-free \DORMP algorithm to output the combination weights $\omega_t$ on each round, and provide the hint $\h_t(\omega_t) = H_t \omega_t$ to the base learner.
The following result, proved in \cref{adaptive-hinting-proofs}, shows that this learning to hint strategy performs nearly as well as the best constant hint combination strategy in restrospect.
\begin{theorem}[Learning to hint regret] \label{base-and-hinter-regret}
Suppose the base problem satisfies \cref{base_assumptions} and the hinting problem is solved with \DORMP hint iterates $\omega_t$, hinting losses $l_t(\omega) = f_t(H_t \omega)$, no meta-hints for the hinting problem, and $q=\argmin_{q' \geq 2} m^{2/q'}(q'-1)$. Then the base problem with hints $\h_t(\omega_t) = H_t \omega_t$ satisfies
\begin{talign}
&\regret_T(\uu) 
    \leq C_0(\uu)  + C_1(\uu) \sqrt{\inf_{\omega\in\Omega} \sum_{t=1}^T f_t(\h_t(\omega))} \\
    &+  C_1(\uu) \big({(2\log_2(m)-1) (\half\xi_T + \sum_{t=1}^{T-1} \huber(\xi_t, \zeta_t)})\big)^{1/4} \\
&\qtext{for} 
\ \,\xi_t 
    \defeq 4(D+1)\sum_{s=t-D}^{t} \norm{\gamma_{s}}_{\infty}^2,
    \quad
\gamma_t 
    \in \subdiff l_t(\omega_t),
    \\
&\qtext{and}
\zeta_t 
    \defeq 4  \norm{\gamma_{t-D}}_{\infty} \sum_{s=t-D}^t \norm{\gamma_{s}}_{\infty}.
\end{talign}
\end{theorem}

To quantify the size of this regret bound, consider again the \DORM base learner with $f_t(\h_t) = \norm{\rr_{t}}_{q}  \norm{\h_t -\sum_{s=t-D}^t\rr_s}_{q}$.  
By \cref{hint-gradient-bound} in \cref{proof_adahedged-example}, $\norm{\gamma_t}_{\infty} \leq d^{1/q}\maxentrynorm{H_t} \norm{\rr_t}_{q}$ for $\maxentrynorm{H_t}$ the maximum absolute entry of $H_t$.
Each column of $H_t$ is a sum $D+1$ subgradient hints, so $\maxentrynorm{H_t}$ is $\bigO{D+1}$. %
Thus, for this choice of hinter loss, the $\huber(\xi_t, \zeta_t)$ term is $\bigO{(D+1)^3}$, and the hint learner suffers only $\bigO{T^{1/4}(D+1)^{3/4}}$ additional regret from learning to hint.
Notably, this additive regret penalty is $\bigO{\sqrt{(D+1)T}}$
if $D = \bigO{T}$ (and $o(\sqrt{(D+1)T})$
when $D = o(T)$), so the learning to hint strategy of \cref{base-and-hinter-regret} preserves minimax optimal regret rates.

\para{Related work}
\citet[Sec.~4.1]{rakhlin2013online} propose and analyze a method to learn optimism strategies for a two-step OMD base learner.
Unlike \cref{base-and-hinter-regret}, the approach does not accommodate delay, and the analyzed regret is only with respect to single hinting strategies $\omega\in \{\basis_j\}_{j \in [m]}$ rather than combination strategies, $\omega \in \simplexm$. 
\section{Experiments} \label{sec:experiments}

\begin{table*}[bt]
\caption{\textbf{Average RMSE of the 2011-2020 semimonthly forecasts}: The average RMSE for online learning algorithms (left) and input models (right) over a $10$-year evaluation period with the top-performing learners and input models bolded and blue. %
In each task, the online learners compare favorably with the best input model and learn to downweight the lower-performing candidates, like the worst models italicized in red.}

\label{exp-zoo-table}
\vskip 0.15in
\centering
\begin{small}
\begin{sc}
\begin{tabular*}{\linewidth}{lrrr|rrrrrr}
\toprule

\multicolumn{1}{c}{{}} &  \multicolumn{1}{c}{AdaHedgeD} &   \multicolumn{1}{c}{DORM} &  \multicolumn{1}{c}{DORM+} &   \multicolumn{1}{c}{Model1}  &  \multicolumn{1}{c}{Model2} &  \multicolumn{1}{c}{Model3} &  \multicolumn{1}{c}{Model4}  &  \multicolumn{1}{c}{Model5}  &  \multicolumn{1}{c}{Model6}  \\
\midrule
Precip. 3-4w &     21.726 & 21.731 & \textcolor{RoyalBlue}{\textbf{21.675}} &         \textcolor{RoyalBlue}{\textbf{21.973}} &    22.431 &         22.357 &   21.978 &         21.986 &     \textcolor{Maroon}{\textit{23.344}} \\
Precip. 5-6w &     21.868 & 21.957 & \textcolor{RoyalBlue}{\textbf{21.838}} &         22.030 &    22.570 &         22.383 &   22.004 &         \textcolor{RoyalBlue}{\textbf{21.993}} &     \textcolor{Maroon}{\textit{23.257}} \\
Temp. 3-4w   &      2.273 &  2.259 &  \textcolor{RoyalBlue}{\textbf{2.247}} &         \textcolor{RoyalBlue}{\textbf{2.253}} &     2.352 &          2.394 &    2.277 &          2.319 &      \textcolor{Maroon}{\textit{2.508}} \\
Temp. 5-6w   &      2.316 &  2.316 &  \textcolor{RoyalBlue}{\textbf{2.303}} &          \textcolor{RoyalBlue}{\textbf{2.270}} &     2.368 &          2.459 &    2.278 &          2.317 &      \textcolor{Maroon}{\textit{2.569}} \\
\bottomrule
\end{tabular*}
\end{sc}
\end{small}
\end{table*}

We now apply the online learning techniques developed in this paper to the problem of adaptive ensembling for subseasonal forecasting. 
Our experiments are based on the subseasonal forecasting data of \citet{frii} that provides the forecasts of $d=6$ machine learning and physics-based models for both temperature and precipitation at two forecast horizons: 3-4 weeks and 5-6 weeks.
In operational subseasonal forecasting, feedback is delayed; models make $D=2$ or $3$ forecasts (depending on the forecast horizon) before receiving feedback. We use delayed, optimistic online learning to play a time-varying convex combination of input models and compete with the best input model over a year-long prediction period ($T=26$ semimonthly dates). The loss function is the geographic root-mean squared error (RMSE) across $514$ locations in the Western United States.

We evaluate the relative merits of the delayed online learning techniques presented by computing yearly regret and mean RMSE for the ensemble plays made by the online leaner in each year from 2011-2020.  Unless otherwise specified, all online learning algorithms use the \texttt{recent\_g} hint $\tildeg_s$, which approximates each unobserved subgradient at time $t$ with the most recent observed subgradient $\g_{t-D-1}$.
See \cref{sec:experiment_details} for full experimental details, \cref{sec:algorithm_details} for algorithmic details, and \cref{sec:experiment_extended_results} for extended experimental results.

\para{Competing with the best input model} \label{all_comparison}
The primary benefit of online learning in this setting is its ability to achieve small average regret, i.e., to perform nearly as well as the best input model in the competitor set $\uset$ without knowing which is best in advance. 
We run our three delayed online learners---\DORM, \DORMP, and \AdaHedgeD---on all four subseasonal prediction tasks and measure their average loss. 

\begin{figure}[bth]
    \centering
    \includegraphics[width=\columnwidth]{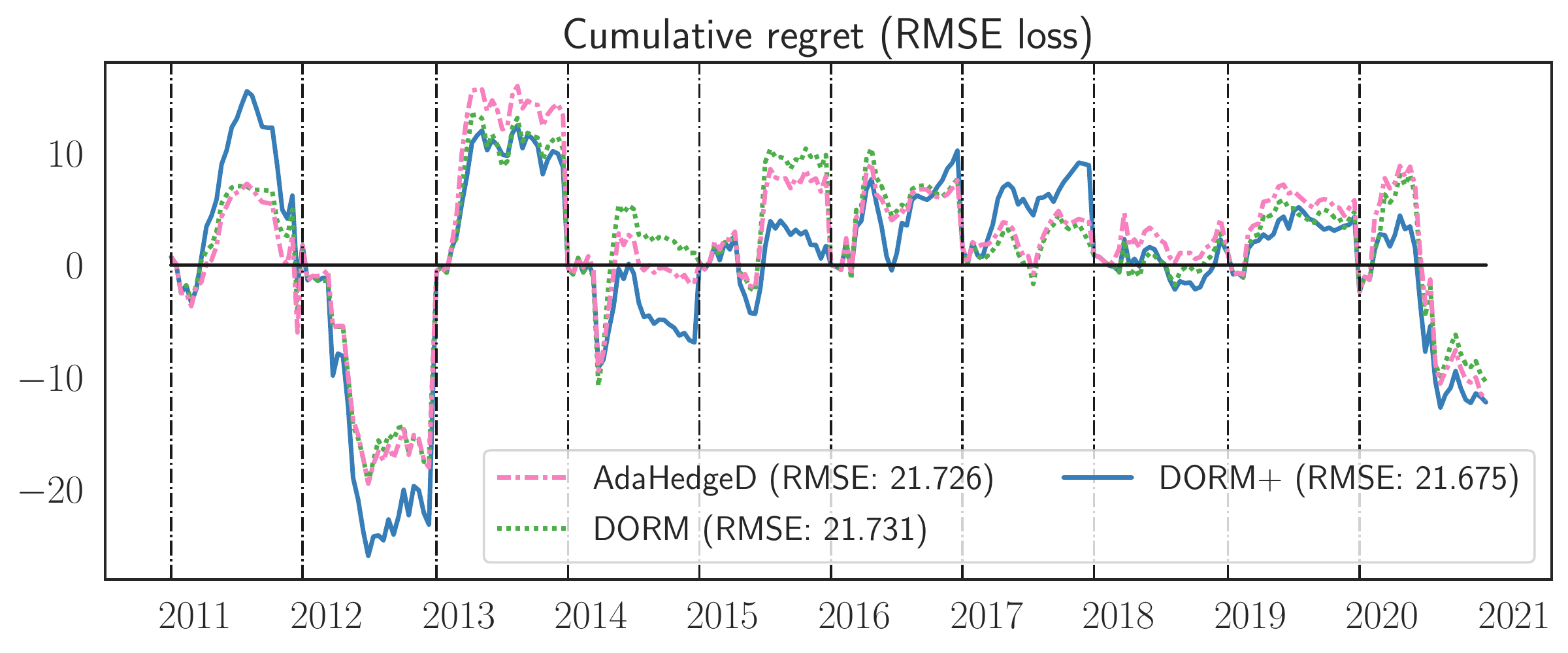}
    \caption{\textbf{Overall performance}: Yearly cumulative regret under RMSE loss for the 
    the Precip.\ 3-4w task. The zero line corresponds to the performance of the best input model in a given year.}
    \label{fig:zoo-experts}
\end{figure}

The average yearly RMSE for the three online learning algorithms and the six input models is shown in \cref{exp-zoo-table}. The \DORMP algorithm tracks the performance of the best input model for all tasks except Temp.~5-6w. All online learning algorithms achieve negative regret for both precipitation tasks. \cref{fig:zoo-experts} shows the yearly cumulative regret (in terms of the RMSE loss) of the online learning algorithms over the $10$-year evaluation period. There are several years (e.g., 2012, 2014, 2020) in which all online learning algorithms substantially outperform the best input forecasting model. The consistently low regret year-to-year of \DORMP compared to \DORM and \AdaHedgeD makes it a promising candidate for real-world delayed subseasonal forecasting. Notably, RM+ (a special case of \DORMP) is known to have small \emph{tracking regret}, i.e., it competes well even with strategies that switch between input models a bounded number of times \citep[Thm.\ 2]{tammelin2015solving}.  We suspect that this is one source of \DORMP's superior performance. We also note that the self-tuned \AdaHedgeD performs comparably to the the optimally-tuned \DORM, demonstrating the effectiveness of our self-tuning strategy.

\para{Impact of regularization}
We evaluate the impact of the three regularization strategies developed in this paper: 1) the upper bound \DUB strategy, 2) the tighter \AdaHedgeD strategy, and 3) the \DORMP algorithm that is tuning-free. This tuning-free property has evident practical benefits, as this section demonstrates. 

\begin{figure}[b]
    \centering
    \includegraphics[width=\columnwidth]{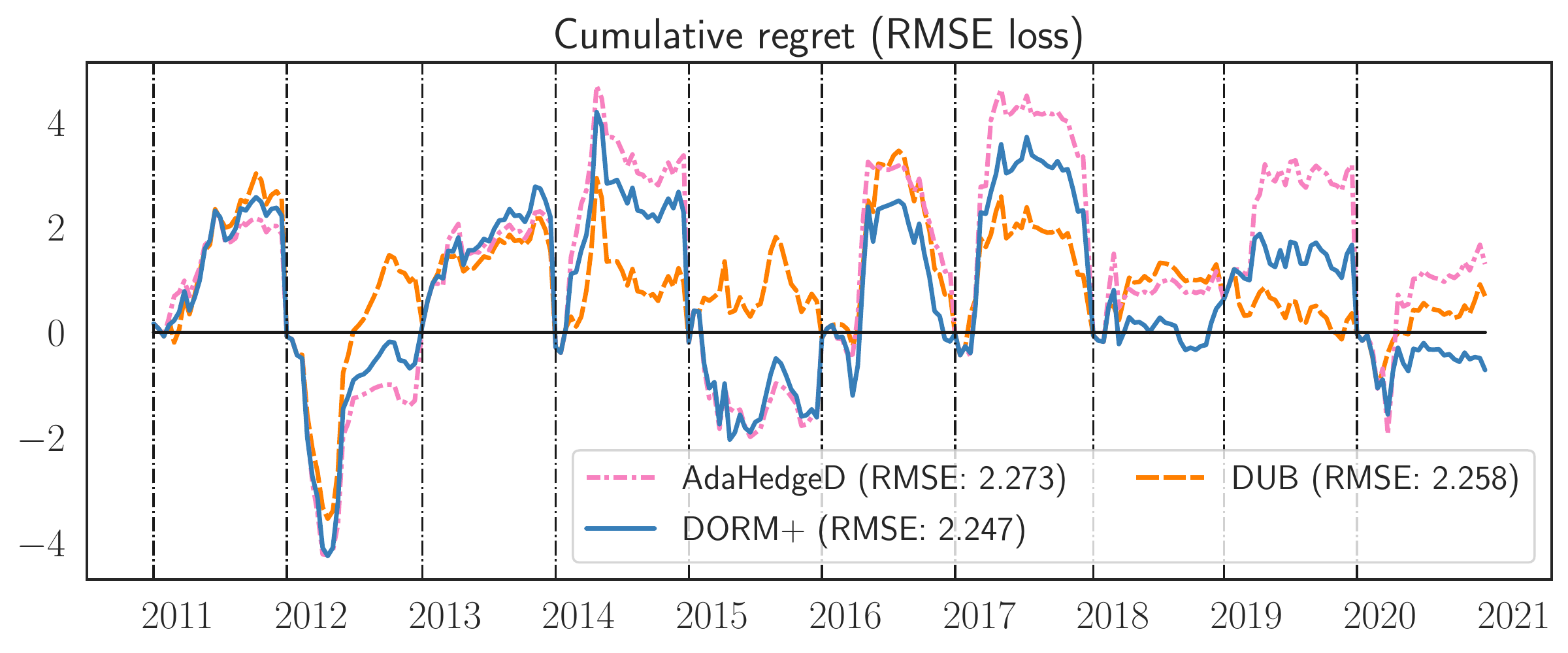}
    \caption{\textbf{Regret of regularizers}: Yearly cumulative regret (in terms of the RMSE loss) for the three regularization strategies for the Temp.~3-4w task.  } 
    \label{fig:regularization_regret}
\end{figure}

\cref{fig:regularization_regret} shows the yearly regret of the \DUB, \AdaHedgeD, and \DORMP algorithms. A consistent pattern appears in the yearly regret: \DUB has moderate positive regret, \AdaHedgeD has both the largest positive and negative regret values, and  \DORMP sits between these two extremes. 
If we examine the weights played by each algorithm (\cref{fig:regularization}), the weights of \DUB and \AdaHedgeD appear respectively over- and under-regularized compared to \DORMP (the top model for this task). \DUB's use of the upper bound $\bb_{t,F}$ results in a very large regularization setting ($\lam_T = 142.881$) and a virtually uniform weight setting. \AdaHedgeD's tighter bound $\delta_t$ produces a value for $\lam_T=3.005$ that is two orders of magnitude smaller. However, in this short-horizon forecasting setting, \AdaHedgeD's aggressive plays result in higher average RMSE. By nature of it's $\lam_t$-free updates, \DORMP produces more moderately regularized plays $\ww_t$ and negative regret. 
\begin{figure*}[hbt!]
    \centering
    \includegraphics[height=1.5in]{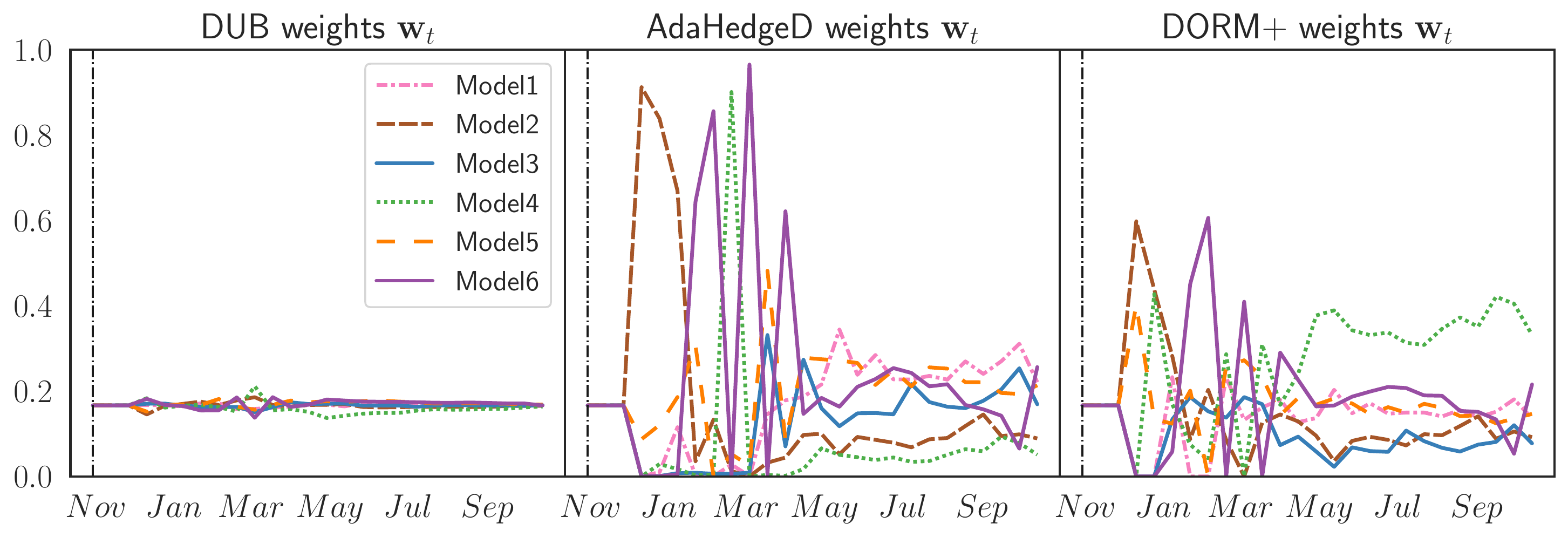} 
    \includegraphics[height=1.512in]{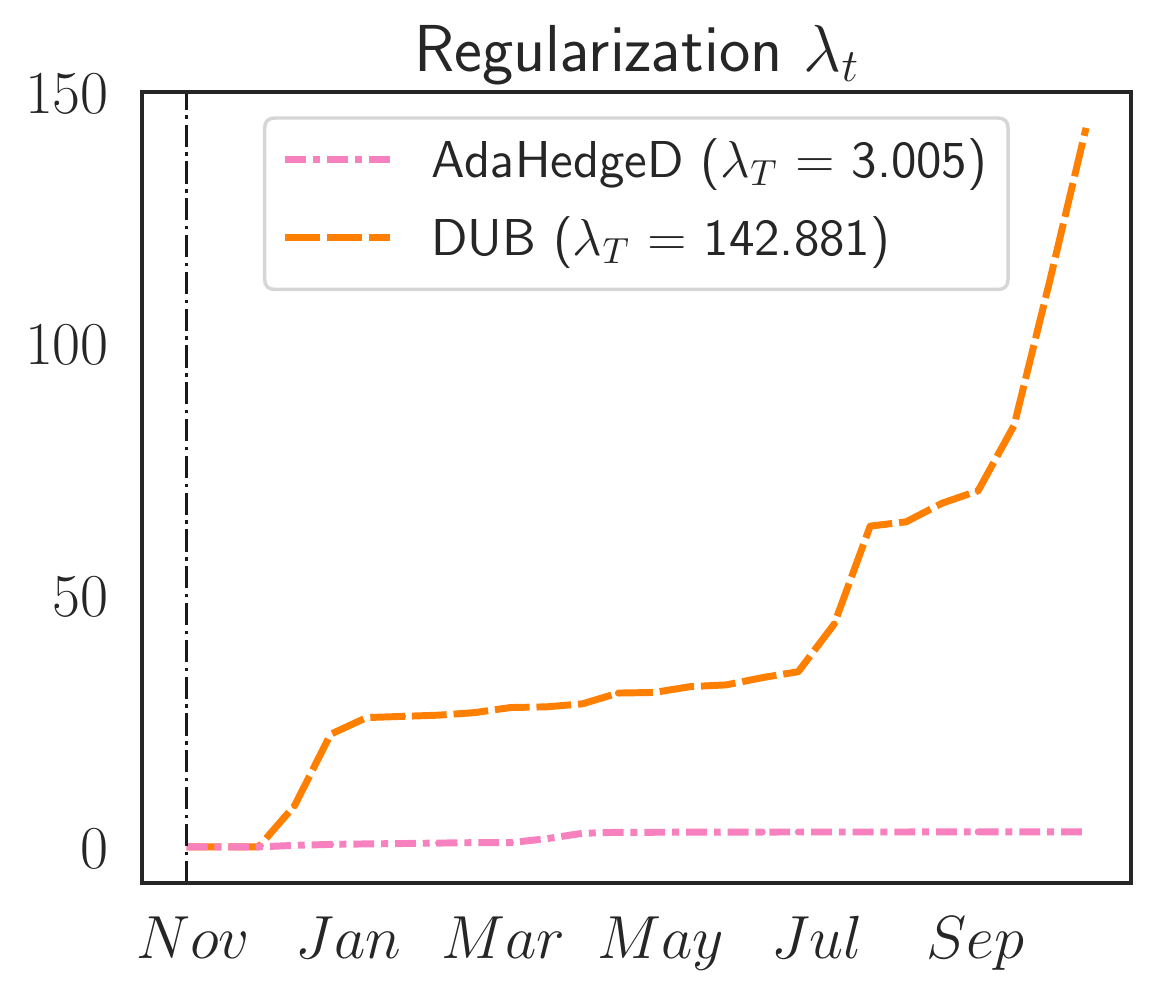}
    \caption{\textbf{Impact of regularization:} The plays $\ww_t$ of online learning algorithms used to combine the input models for the Temp.~3-4w task in the 2020 evaluation year. The weights of \DUB and \AdaHedgeD appear respectively over and under regularized compared to \DORMP (the top model for this task) due to their selection of regularization strength $\lam_t$ (right).}
    \label{fig:regularization}
\end{figure*}

\para{To replicate or not to replicate}
In this section, we compare the performance of replicated and non-replicated variants of our \DORMP algorithm. 
Both algorithms perform well (see \cref{sec:replicate}), but in all tasks, \DORMP outperforms replicated \DORMP (in which $D+1$ independent copies of \DORMP make staggered predictions). \cref{fig:exp_replication} provides an example of the  weight plots produced by the replication strategy in the Temp.~5-6w task with $D=3$. 
The separate nature of the replicated learner's plays is evident in the weight plots and leads to an average RMSE of $2.315$, versus $2.303$ for \DORMP in the Temp.~5-6w task.

\begin{figure}[htb!]
    \centering
    \includegraphics[width=\columnwidth]{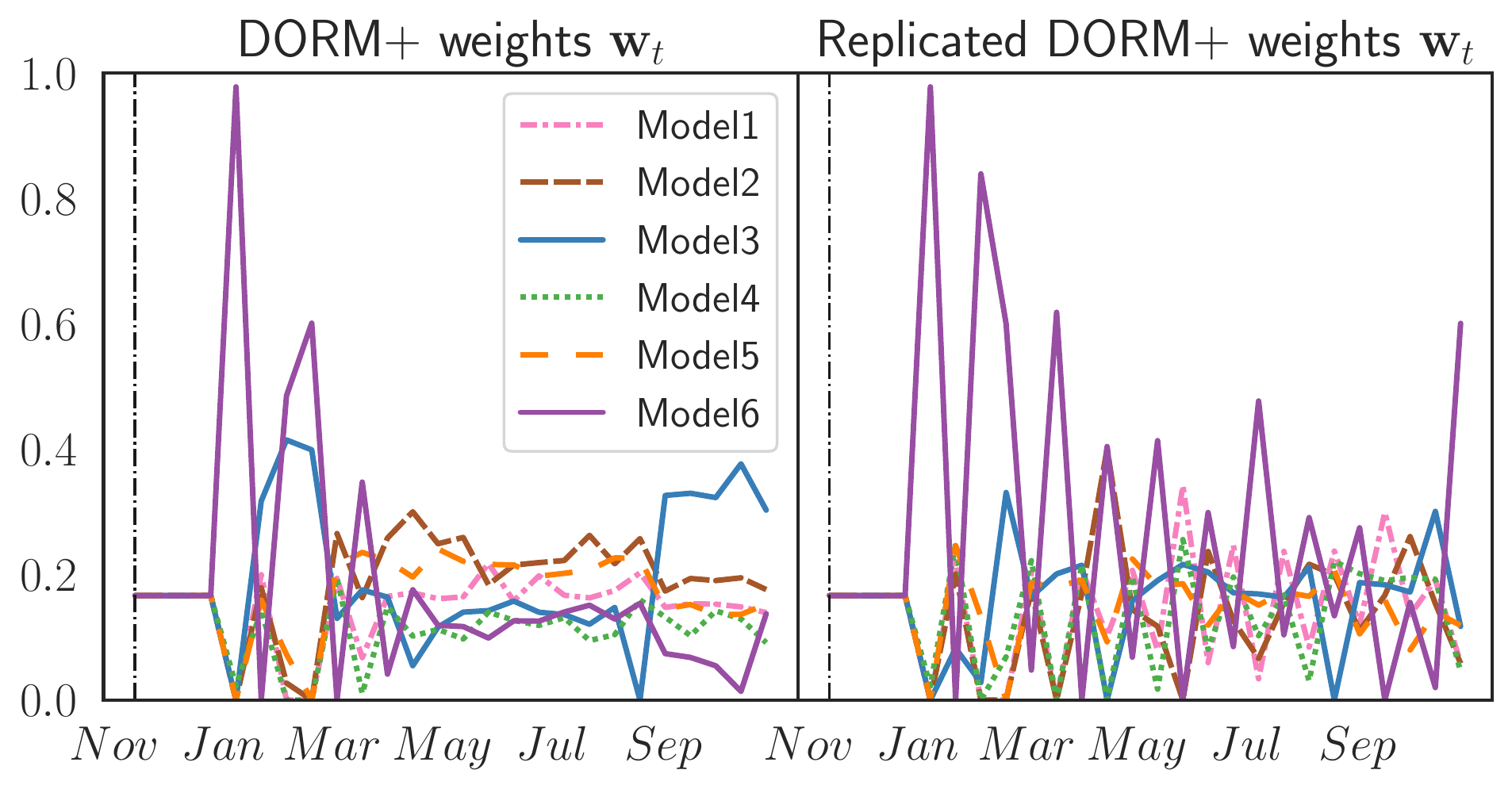}   
    \caption{\textbf{To replicate or not to replicate}: The plays $\ww_t$ of standard \DORMP and replicated \DORMP algorithms for the Temp.~5-6w task in the final evaluation year.}
    \label{fig:exp_replication}
\end{figure}

\para{Learning to hint}
Finally, we examine the effect of optimism on the \DORMP algorithms and the ability of our ``learning to hint'' strategy to recover the performance of the best optimism strategy in retrospect.  
Following the hint construction protocol in  \cref{sec:algorithm_details_dorm_dormp}, 
we run the \DORMP base algorithm with 
$m=4$ subgradient hinting strategies: %
$\tildeg_s = \g_{t-D-1}$ (\texttt{recent\_g}), $\tildeg_s = \g_{s-D-1}$ (\texttt{prev\_g}), $\tildeg_s = \frac{D+1}{t-D-1} \g_{1:t-D-1}$ (\texttt{mean\_g}), or $\tildeg_s = \boldzero$ (\texttt{none}).
We also use \DORMP as the meta-algorithm for hint learning to produce the \texttt{learned} optimism strategy that plays a convex combination of the four hinters. In \cref{fig:hinting}, we first note that several optimism strategies outperform the \texttt{none} hinter, confirming the value of optimism in reducing regret. The \texttt{learned} variant of \DORMP avoids the worst-case performance of the individual hinters in any given year (e.g., 2015), while staying competitive with the best strategy (although it does not outperform the dominant \texttt{recent\_g} strategy overall). We believe the performance of the online hinter could be further improved by developing tighter convex bounds on the regret of the base problem in the spirit of \cref{base_assumptions}. 

\begin{figure}[bt!]
    \centering
    \includegraphics[width=\columnwidth]{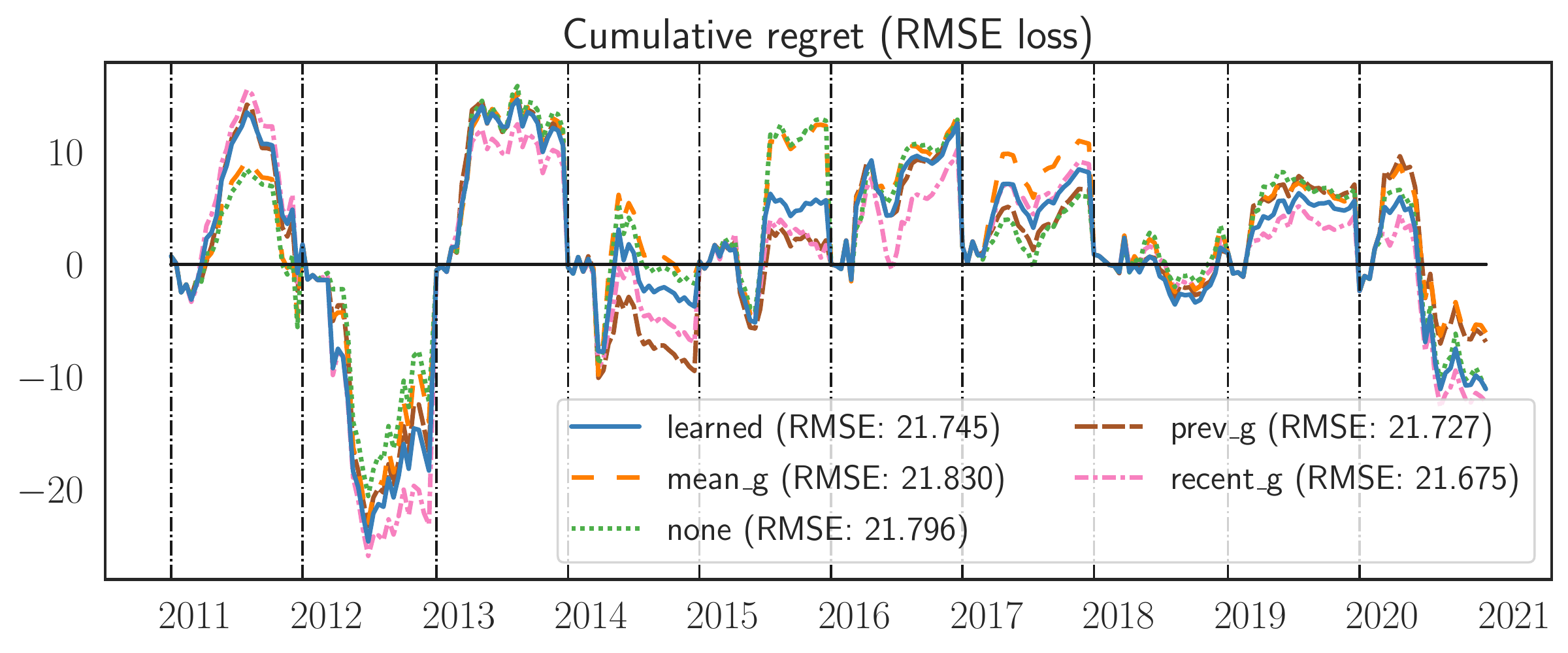}
    \caption{\textbf{Learning to hint}: Yearly cumulative regret (in terms of the RMSE loss) for the adaptive hinting and four constant hinting strategies for the Precip. 3-4w task.}
    \label{fig:hinting}
\end{figure}

\section{Conclusion}
In this work, we confronted the challenges of delayed feedback and short regret horizons in online learning with optimism, developing practical non-replicated, self-tuned and tuning-free algorithms with optimal regret guarantees. Our ``delay as optimism'' reduction and our refined analysis of optimistic learning produced novel regret bounds for both optimistic and delayed online learning and elucidated the connections between these two problems. Within the subseasonal forecasting domain, we demonstrated that delayed online learning methods can produce zero-regret forecast ensembles that perform robustly from year-to-year. Our results highlighted \DORMP as a particularly promising candidate due to its tuning-free nature and small tracking regret. 

In future work, we are excited to further develop optimism strategies under delay by 1) employing tighter convex loss bounds on the regret of the base algorithm to improve the learning to hint algorithm, 2) exploring the relative impact of hinting for ``past'' ($\g_{t-D:t-1}$) versus ``future'' ($\g_t$) missing subgradients (see \cref{sec:future_past} for an initial exploration), and
 3) developing adaptive self-tuning variants of the \DOOMD algorithm.  
 Within the subseasonal domain, we plan to leverage the flexibility of our optimism formulation to explore hinting strategies that use meteorological expertise to improve beyond the generic mean and past subgradient hints and to deploy our open-source subseasonal forecasting algorithms operationally. 
\section*{Acknowledgements}
This work was supported by Microsoft AI for Earth, an NSF GRFP, and the NSF 
grants no.\ 1925930 ``Collaborative Research: TRIPODS Institute for Optimization and Learning'', no.\ 1908111 ``AF: Small: Collaborative Research: New Representations for Learning Algorithms and Secure Computation'', and no.\ 2022446 ``Foundations of Data Science Institute''. FO also thanks Nicol\`o Cesa-Bianchi and Christian Kroer for discussions on RM and RM+.
\clearpage\newpage
{
\small
\bibliography{online_learning}
\bibliographystyle{icml_style/icml2021}
}
\appendix\onecolumn
\clearpage
\newpage
\section{Extended Literature Review}\label{lit_review}
We review here additional prior work not detailed in the main paper. 
\subsection{General online learning}
We recommend the monographs of \citet{shalev2012online,Orabona2019AMI} and the textbook of \citet{cesa2006prediction} for surveys of the field of online learning and 
\citet{joulani2017modular,mcmahan2017survey} for widely applicable and modular analyses of online learning algorithms.

\subsection{Online learning with optimism but without delay}
\citet{syrgkanis2015fast} analyzed optimistic FTRL and two-step variant of optimistic MD without delay.  The work focuses on a particular form of optimism (using the last observed subgradient as a hint) and shows improved rates of convergence to correlated equilibria in multiplayer games.
In the absence of delay, \citet{steinhardt2014adaptivity} combined optimism and adaptivity to obtain improvements over standard optimistic regret bounds.
\subsection{Online learning with delay but without optimism}

\paragraph{Overview}  \citet{joulani2013online,joulani2016delay,mcmahan2014delay} provide broad reviews of progress on delayed online learning. 

\paragraph{Delayed stochastic optimization} \citet{recht2011hogwild, agarwal2012distributed, nesterov2012efficiency, liu2014asynchronous, liu2015asynchronous,sra2016adadelay}
studied the effects of delay on stochastic optimization but do not treat the adversarial setting studied here.

\paragraph{FTRL-Prox vs. FTRL} \citet{joulani2016delay} analyzed the delayed feedback regret of the \emph{FTRL-Prox} algorithm, which regularizes toward the last played iterate as in online mirror descent, but did not study the standard FTRL algorithms (sometimes called \emph{FTRL-Centered}) analyzed in this work.

\subsection{Self-tuned online learning without delay or optimism}
In the absence of optimism and delay, 
\citet{rooij14a,OrabonaP15,koolen2014learning} developed alternative variants of FTRL algorithms that self-tune their learning rates.

\subsection{Online learning without delay for climate forecasting}
\citet{monteleoni2011tracking} applied the Learn-$\alpha$ online learning algorithm of \citet{monteleonijaakkola2004} to the task of ensembling climate models.
The authors considered historical temperature data from 20 climate models and tracked the changing sequence of which model predicts best at any given time. In this context, the algorithm used was based on a set of generalized Hidden Markov Models, in which the identity of the current best model is the hidden variable and the updates are derived as Bayesian updates. This work was extended to take into account the influence of regional neighboring locations when performing updates~\citep{mcquade2012global}. These initial results demonstrated the promise of applying online learning to climate model ensembling, but both methods rely on receiving feedback without delay.

\section{Proof of \cref{oftrl_regret}: OFTRL regret} \label{proof_oftrl_regret}

We will prove the following more general result for optimistic adaptive FTRL (\OAFTRL)
\begin{talign}\label[name]{oaftrl}\tag{OAFTRL}
    \ww_{t+1} = \argmin_{\ww\in\wset} \,\inner{\g_{1:t} + \tildeg_{t+1}}{\ww} + \lam_{t+1} \reg(\ww),
\end{talign}
from which \cref{oftrl_regret} will follow with the choice $\lam_t = \lam$ for all $t \geq 1$.

\begin{theorem}[\OAFTRL regret] \label{oaftrl_regret}
If $\reg$ is nonnegative and $(\lam_t)_{t\geq 1}$ is non-decreasing, then, $\forall \uu\in\wset$, the \OAFTRL iterates $\ww_t$ satisfy, 
\begin{talign}
\regret_T(\uu)
    \label{oaftrl_delta_regret_bound}
    &\leq 
        \lam_{T}\reg(\uu) + \sum_{t=1}^T 
        \delta_t
    \\
    &\leq
        \lam_{T}\reg(\uu) + \sum_{t=1}^T 
        \min\big(\frac{1}{\lam_t}\huber(\dualnorm{\g_t-\tildeg_t}, \dualnorm{\g_t}), 
        \diam{\wset}\min(\dualnorm{\g_t-\tildeg_t},\dualnorm{\g_t})\big)
\end{talign}
for 
\begin{talign}
\delta_t 
    &\defeq \min (\obj_{t+1}(\ww_t,\lam_t) - \obj_{t+1}(\bar\ww_t,\lam_t),\  \ \inner{\g_t}{\ww_t - \bar\ww_t},\\
    &\qquad\quad\ \obj_{t+1}(\hat{\ww}_t,\lam_t) -    \obj_{t+1}(\bar\ww_t,\lam_t)
    +
    \inner{\g_t}{\ww_t - \hat{\ww}_t} )_+ \qtext{with}  \\ 
\bar\ww_t 
    &\defeq \argmin_{\ww\in\wset} \obj_{t+1}(\ww,\lam_t),
    \quad
\obj_{t+1}(\ww,\lam_t) 
    \defeq \lam_t \reg(\ww) + \inner{\g_{1:t}}{\ww},
\qtext{and} \\
\hat{\ww}_t
    &\defeq \argmin_{\ww\in\wset}
    \lam_t \reg(\ww) + \inner{\g_{1:t}+\min(\frac{\dualnorm{\g_t}}{\dualnorm{\tildeg_t - \g_t}},1)(\tildeg_t - \g_t)}{\ww}.
\end{talign}
\end{theorem}
\begin{proof}
Consider a sequence of arbitrary auxiliary subgradient hints $\atildeg_1, \dots, \atildeg_{T} \in \reals^d$ and the auxiliary OAFTRL sequence
\begin{talign}\label{aux_oaftrl}
    \aww_{t+1} = \argmin_{\aww\in\wset} \,\inner{\g_{1:t} + \atildeg_{t+1}}{\aww} + \lam_{t+1} \reg(\aww)
    \qtext{for} 0 \leq t \leq T
    \qtext{with}
    \atildeg_{T+1}\defeq \boldzero
    \qtext{and}
    \lam_{T+1} = \lam_T.
\end{talign}
Generalizing the forward regret decomposition of \citet{joulani2017modular} and the prediction drift decomposition of \citet{joulani2016delay}, we will decompose the regret of our original $(\ww_t)_{t=1}^T$ sequence into the regret of the auxiliary sequence $(\aww_t)_{t=1}^T$ and the drift between $(\ww_t)_{t=1}^T$ and $(\aww_t)_{t=1}^T$.

For each time $t$, define 
the auxiliary optimistic objective function $\aoobj_t(\ww) = \obj_t(\ww) + \inner{\atildeg_t}{\ww}$.
Fixing any $\uu \in \wset$, we have the regret bound
\begin{talign}
    \regret_T(\uu) 
         &= \sum_{t=1}^T \loss_t(\ww_t) - \loss_t(\uu) 
         \leq \sum_{t=1}^T \inner{\g_t}{\ww_t-\uu} 
            \qtext{(since each $\loss_t$ is convex with $\g_t \in \partial \loss_t(\ww_t)$)}\\   
         &= \underbrace{\sum_{t=1}^T \inner{\g_t}{\ww_t-\aww_t}}_{\textup{drift}} +  \underbrace{\sum_{t=1}^T \inner{\g_t}{\aww_t-\uu}}_{\textup{auxiliary regret}}.
\end{talign}
 
To control the drift term we employ the following lemma, proved in \cref{proof_oaftrl_difference_bound}, which bounds the difference between two \OAFTRL optimizers with different losses but common regularizers.
\begin{lemma}[\OAFTRL difference bound] \label{oaftrl_difference_bound}  The \OAFTRL and auxiliary \textup{OAFTRL} iterates \cref{aux_oaftrl}, $\ww_t$ and $\aww_t$, satisfy
    \begin{talign}
        \norm{\ww_t - \aww_t} 
        \leq \min(\frac{1}{\lam_t}\dualnorm{\tildeg_t - \atildeg_t}, \diam{\wset}).
    \end{talign}
\end{lemma}

Letting $a = \diam{\wset} \in \R \cup \{\infty\}$, we now bound each drift term summand using the Fenchel-Young inequality for dual norms and  \cref{oaftrl_difference_bound}: 
\begin{talign}
\inner{\g_t}{\ww_t-\aww_t}
    \leq
        \dualnorm{\g_t}\norm{\ww_t - \aww_t}
        \leq \min \big( \frac{1}{\lam_t}\dualnorm{\g_t}\dualnorm{\tildeg_t - \atildeg_t}, a\dualnorm{\g_t} \big).
\end{talign}

To control the auxiliary regret, we begin by invoking the OAFTRL regret bound of \citet[proof of Thm.~7.28]{Orabona2019AMI}, the nonnegativity of $\reg$, and the assumption that $(\lam_t)_{t\geq 1}$ is non-decreasing:
\begin{talign}
\sum_{t=1}^T \inner{\g_t}{\aww_t-\uu} 
     &\leq 
        \lam_{T+1} \reg(\uu) - \lam_1 \reg(\aww_1)     
        +  \sum_{t=1}^T \obj_{t+1}(\aww_t,\lam_t) -  \obj_{t+1}(\bar\ww_t,\lam_t)
     + (\lam_t - \lam_{t+1}) \reg(\aww_{t+1}) \\
     &\leq 
        \lam_{T+1} \reg(\uu) - \lam_1 \reg(\aww_1)     
        + \sum_{t=1}^T \obj_{t+1}(\aww_t,\lam_t) -  \obj_{t+1}(\bar\ww_t,\lam_t).
\end{talign}
We next bound the summands in this expression in two ways.
Since $\aww_t$ is the minimizer of $\aoobj_t$, we may apply the Fenchel-Young inequality for dual norms to conclude that
\begin{talign}
\obj_{t+1}(\aww_t,\lam_t) -  \obj_{t+1}(\bar\ww_t,\lam_t)
    &=
    \aoobj_t(\aww_t) + \inner{\aww_t}{\g_t-\atildeg_t} - (\aoobj_{t}(\bar\ww_t) + \inner{\bar\ww_t}{\g_t-\atildeg_t}) \\
    &\leq 
    \inner{\aww_t-\bar\ww_t}{\g_t-\atildeg_t} 
    \leq 
    \norm{\aww_t-\bar\ww_t}\dualnorm{\g_t-\atildeg_t} 
    \leq
    a\dualnorm{\g_t-\atildeg_t}.
\end{talign}
Moreover, by \citet[proof of Thm.~7.28]{Orabona2019AMI} and the fact that $\bar\ww_t$ minimizes $\obj_{t+1}(\cdot,\lam_t)$ over $\wset$,
\begin{talign}
\obj_{t+1}(\aww_t,\lam_t) -  \obj_{t+1}(\bar\ww_t,\lam_t)
    &\leq  \frac{\dualnorm{\g_t-\atildeg_t}^2}{2\lam_t}.
\end{talign}

Our collective bounds establish that
\begin{talign} \label{def-deltat-general}
\delta_t(\atildeg_t)
    &\defeq
\obj_{t+1}(\aww_t,\lam_t) - \obj_{t+1}(\bar\ww_t,\lam_t)
    +
\inner{\g_t}{\ww_t - \aww_t} \\
    &\leq
\min(\frac{1}{2\lam_t}\dualnorm{\g_t-\atildeg_t}^2, a\dualnorm{\g_t-\atildeg_t}) + \min(\frac{1}{\lam_t}\dualnorm{\g_t}\dualnorm{\tildeg_t-\atildeg_t}, a\dualnorm{\g_t}) \\
    &\leq
\frac{1}{2\lam_t}\dualnorm{\g_t-\atildeg_t}^2 + \frac{1}{\lam_t}\dualnorm{\g_t}\dualnorm{\tildeg_t-\atildeg_t}.
\end{talign}
To obtain an interpretable bound on regret, we will minimize the final expression over all convex combinations $\atildeg_t$ of $\g_t$ and $\tildeg_t$.
The optimal choice is given by 
\begin{align}
\hat{\g}_t 
    &= 
\g_t + c_* (\tildeg_t - \g_t) 
    \qtext{for} \\
c_* 
    &\defeq\textstyle
\min(\frac{\dualnorm{\g_t}}{\dualnorm{\tildeg_t - \g_t}},1)
    = 
\displaystyle\argmin_{c\leq 1, \atildeg_t = \g_t + c (\tildeg_t - \g_t)}\textstyle \frac{1}{2\lam_t}\dualnorm{\g_t-\atildeg_t}^2 + \frac{1}{\lam_t}\dualnorm{\g_t}\dualnorm{\tildeg_t-\atildeg_t} \\
    &=\textstyle
\argmin_{c\leq 1} \frac{c^2}{2\lam_t}\dualnorm{\g_t-\tildeg_t}^2 + \frac{1-c}{\lam_t}\dualnorm{\g_t}\dualnorm{\tildeg_t-\g_t}.
\end{align}
For this choice, we obtain the bound
\begin{talign}
(\delta_t(\hat{\g}_t))_+
    &\leq 
        \frac{1}{2\lam_t}\dualnorm{\g_t-\hat{\g}_t}^2
    +
        \frac{1}{\lam_t}\dualnorm{\g_t}\dualnorm{\hat{\g}_t-\tildeg_t} \\
    &=
        \frac{c_*^2}{2\lam_t}\dualnorm{\g_t-\tildeg_t}^2
    + 
        \frac{1-c_*}{\lam_t}\dualnorm{\g_t}\dualnorm{\g_t-\tildeg_t} \\
    &=
        \frac{1}{2\lam_t}\min(\dualnorm{\g_t-\tildeg_t},\dualnorm{\g_t})^2
    + 
         \frac{1}{\lam_t}\dualnorm{\g_t}(\dualnorm{\g_t-\tildeg_t}-\dualnorm{\g_t})_+ \\
    &= 
        \frac{1}{2\lam_t}(\dualnorm{\g_t-\tildeg_t}^2
    - 
        (\dualnorm{\g_t-\tildeg_t}-\dualnorm{\g_t})_+^2) \\
    &= 
        \frac{1}{\lam_t}\huber(\dualnorm{\g_t-\tildeg_t},\dualnorm{\g_t}) 
\end{talign}
and therefore
\begin{talign}
\label{delta_min_bound}
\delta_t = \min(\delta_t(\tildeg_t),\delta_t(\g_t), \delta_t(\hat{\g}_t))_+
    &\leq 
        \min(\frac{1}{\lam_t}\huber(\dualnorm{\g_t-\tildeg_t},\dualnorm{\g_t}),
            a\min(\dualnorm{\g_t-\tildeg_t}, \dualnorm{\g_t})).
\end{talign} 
Since $\atildeg_t$ is arbitrary, 
 the advertised regret bounds follow as
\begin{talign}
\regret_T(\uu)
    &\leq 
        \inf_{\atildeg_1,\dots,\atildeg_T \in\reals^d} \lam_{T+1} \reg(\uu)
        + \sum_{t=1}^T \delta_t(\atildeg_t) \\
    &= 
        \lam_{T+1} \reg(\uu)
        + \sum_{t=1}^T \inf_{\atildeg_t \in\reals^d}\delta_t(\atildeg_t) \\
    &\leq 
        \lam_{T+1} \reg(\uu)
        + \sum_{t=1}^T \min(\delta_t(\tildeg_t),\delta_t(\g_t),\delta_t(\hat{\g}_t))_+. 
\end{talign}
\end{proof}

\subsection{Proof of \cref{oaftrl_difference_bound}: \OAFTRL difference bound}\label{proof_oaftrl_difference_bound}
Fix any time $t$, and define the optimistic objective function 
$\oobj_t(\ww) = \lam_t \reg(\ww) + \sum_{i=1}^{t-1} \inner{\g_i}{\ww} + \inner{\tildeg_t}{\ww}$ 
and the auxiliary optimistic objective function $\aoobj_t(\ww) = \lam_t \reg(\ww) + \sum_{i=1}^{t-1} \inner{\g_i}{\ww} + \inner{\atildeg_t}{\ww}$ so that
$\ww_t \in \argmin_{\ww\in\wset} \oobj_t(\ww)$ and $\aww_t \in \argmin_{\ww\in\wset} \aoobj_t(\ww)$.
We have
\begin{talign}
    \aoobj_t(\ww_{t}) - \aoobj_t(\aww_{t})  
    &\geq \frac{\lam_t}{2} \norm{\ww_t - \aww_{t}}^2 \qtext{by the strong convexity of $\aoobj_t$ and} \\
    \oobj_{t}(\aww_{t}) - \oobj_{t}(\ww_{t})  
    &\geq \frac{\lam_{t}}{2} \norm{\ww_{t} - \aww_{t}}^2 \qtext{by the strong convexity of $\oobj_{t}$.}
\end{talign}
    
Summing the above inequalities and applying the Fenchel-Young inequality for dual norms, we obtain
\begin{align}
    \lam_t \norm{\ww_{t} - \aww_t}^2
        &\leq \inner{\atildeg_{t} - \tildeg_t}{\ww_t - \aww_t} 
        \leq \dualnorm{\tildeg_t-\atildeg_t}\norm{\ww_t - \aww_t},
\end{align}
which yields the first half of our target bound after rearrangement.
The second half follows from the definition of diameter, as
$\norm{\ww_{t} - \aww_t} \leq \diam{\wset}$.

\section{Proof of \cref{soomd_regret}: \SOOMD regret}
\label{proof_soomd_regret}

We will prove the following more general result for adaptive SOOMD (\ASOOMD)
\begin{align}\label[name]{asoomd}\tag{ASOOMD}
    \ww_{t+1} = \argmin_{\ww\in\wset} \,\inner{\g_t+\tildeg_{t+1}-\tildeg_t}{\ww} + \lam_{t+1}\Breg_{\reg}(\ww,\ww_t)
    \qtext{with arbitrary} \ww_0
    \qtext{and} \g_0 = \tildeg_0 = \boldzero
\end{align}
from which \cref{soomd_regret} will follow with the choice $\lam_t = \lam$ for all $t \geq 1$.

\begin{theorem}[\ASOOMD regret]\label{asoomd_regret}
Fix any $\lam_{T+1}\geq 0$.
If each $(\lam_{t+1}-\lam_t)\reg$ is proper 
and  differentiable, $\lam_0 \defeq 0$, and $\tildeg_{T+1} \defeq \boldzero$, then, for all $\uu\in\wset$, the \ASOOMD iterates $\ww_t$ satisfy
\begin{talign}\label{soomd_regret_bound}
\regret_T(\uu) 
    \leq 
    &\sum_{t=0}^T (\lam_{t+1} - \lam_{t}) \Breg_{\reg}(\uu,\ww_t) +\\
    &\sum_{t=1}^T
    \min\big(
        \diam{\wset}\norm{\g_t - \tildeg_t}_*,
        \frac{1}{\lam_{t+1}}\huber(\norm{\g_t - \tildeg_t}_*,\norm{\g_t + \tildeg_{t+1} - \tildeg_t}_*
        )
        \big).
\end{talign}
\end{theorem}

\begin{proof}
Fix any $\uu\in\wset$, instantiate the notation of \citet[Sec. 7.2]{joulani2017modular}, and consider the choices 
\begin{itemize}
    \item $r_1 = \lam_2 \reg$, $r_t = (\lam_{t+1} - \lam_{t}) \reg$ for $t \geq 2$, so that $r_{1:t} = \lam_{t+1} \reg$ for $t \geq 1$,
    \item $q_t = \tildeq_t + \inner{\tildeg_{t+1} - \tildeg_t}{\cdot}$ for $t \geq 0$, 
    \item $\tilde{q}_0(\ww) = \lam_1\Breg_{\reg}(\ww, \ww_0)$ and $\tilde{q}_t \equiv 0$ for all $t \geq 1$,
    \item $p_1 \defeq r_1 - q_{0} = r_1 - \tilde{q}_{0} - \inner{\tildeg_1 - \tildeg_{0}}{\cdot} = \lam_2 \reg - \lam_1\Breg_{\reg}(\cdot, \ww_0) - \inner{\tildeg_1 - \tildeg_{0}}{\cdot}$,
    \item $p_t \defeq r_t - q_{t-1} = r_t - \tilde{q}_{t-1} - \inner{\tildeg_t - \tildeg_{t-1}}{\cdot} = (\lam_{t+1} - \lam_{t}) \reg - \inner{\tildeg_t - \tildeg_{t-1}}{\cdot}$ for all $t\geq 2$.
\end{itemize}
Since, for each $t$, $\delta_t = 0$ and $\loss_t$ is convex, the \textsc{Ada-MD} regret inequality of \citet[Eq. (24)]{joulani2017modular} and 
the choice $\tildeg_{T+1} = 0$
imply that
\begin{align}
\regret_T(\uu) 
    &= 
        \sum_{t=1}^T\loss_t(\ww_{t}) - \sum_{t=1}^T\loss_t(\uu) \\
    &\leq 
        - \sum_{t=1}^T \Breg_{\loss_t}(\uu, \ww_t)
        + \sum_{t=0}^T q_t(\uu) - q_t(\ww_{t+1})
        + \sum_{t=1}^T \Breg_{p_t}(\uu, \ww_t) \\
        &\quad- \sum_{t=1}^T \Breg_{r_{1:t}}(\ww_{t+1}, \ww_t)
        + \sum_{t=1}^T \inner{\g_t}{\ww_t - \ww_{t+1}}
        + \sum_{t=1}^T \delta_t \\
    &\leq 
        \lam_1(\Breg_{\reg}(\uu, \ww_0) - \Breg_{\reg}(\ww_1, \ww_0)) 
        + \sum_{t=0}^T \inner{\tildeg_{t+1} - \tildeg_t}{\uu - \ww_{t+1}} \\
        &\quad+ 
        \sum_{t=1}^T (\lam_{t+1} - \lam_{t})  \Breg_{\reg}(\uu,\ww_t) 
        +\sum_{t=1}^T \inner{\g_t}{\ww_t - \ww_{t+1}} - \lam_{t+1} \Breg_{\reg}(\ww_{t+1}, \ww_t)\\
    &=
        \sum_{t=0}^T (\lam_{t+1} - \lam_{t})  \Breg_{\reg}(\uu,\ww_t) 
            + \sum_{t=0}^T\inner{\g_t-\tildeg_t}{\ww_t - \ww_{t+1}} 
            - \lam_{t+1} \Breg_{\reg}(\ww_{t+1}, \ww_t). \label{baseline_soomd_regret_bound} 
\end{align}
To obtain our advertised bound, we begin with the expression \cref{baseline_soomd_regret_bound} and invoke the $1$-strong convexity of $\reg$ and the nonnegativity of $\Breg_{\lam \reg}(\ww_{1}, \ww_0)$ to find
\begin{talign}
\regret_T(\uu) 
    &\leq 
        \sum_{t=0}^T (\lam_{t+1} - \lam_{t})  \Breg_{\reg}(\uu,\ww_t) 
            + \sum_{t=0}^T\inner{\g_t-\tildeg_t}{\ww_t - \ww_{t+1}} 
            - \lam_{t+1} \Breg_{\reg}(\ww_{t+1}, \ww_t)\\
    &\leq 
        \sum_{t=0}^T (\lam_{t+1} - \lam_{t}) \Breg_{\reg}(\uu,\ww_t)             + \sum_{t=1}^T\inner{\g_t-\tildeg_t}{\ww_t - \ww_{t+1}} - \frac{\lam_{t+1}}{2}\norm{\ww_t - \ww_{t+1}}^2. \label{intermediate_soomd_regret_bound}
\end{talign}
We will bound the final sum in this expression using two lemmas.
The first is a bound on the difference between subsequent \ASOOMD iterates distilled from \citet[proof of Prop. 2]{joulani2016delay}. 
\begin{lemma}[\ASOOMD iterate bound {\citep[proof of Prop. 2]{joulani2016delay}}]\label{omd_iterate_bound}
If $\reg$ is differentiable and $1$-strongly convex with respect to $\norm{\cdot}$, then the \ASOOMD iterates satisfy
\begin{talign}
\norm{\ww_t - \ww_{t+1}} \leq \frac{1}{\lam_{t+1}}\norm{\g_t + \tildeg_{t+1} - \tildeg_t}_*.
\end{talign}
\end{lemma}
The second, proved in \cref{proof_norm_constrained_conjugate}, is a general bound on $\inner{\g}{\vv} - \frac{\lam}{2} \norm{\vv}^2$ under a norm constraint on $\vv$.
\begin{lemma}[Norm-constrained conjugate] \label{norm_constrained_conjugate}
For any $\g \in \reals^d$ and $\lam, c, b > 0$,
\begin{align} 
    \sup_{\vv \in \reals^d: \norm{\vv} \leq \min(\frac{c}{\lam},b)} 
    \textstyle
        \inner{\g}{\vv} - \frac{\lam}{2} \norm{\vv}^2 
    &= 
    \textstyle
        \frac{1}{\lam} \min(\dualnorm{\g}, c,b\lam)(\dualnorm{\g} - \frac{1}{2} \min(\dualnorm{\g}, c,b\lam)) \\
    &\leq 
    \textstyle
        \min(b\dualnorm{\g}, \frac{1}{\lam} \min(\dualnorm{\g}, c)(\dualnorm{\g} - \frac{1}{2} \min(\dualnorm{\g}, c))) \\
    &= 
    \textstyle
        \min(b\dualnorm{\g},\frac{1}{2\lam} (\dualnorm{\g}^2 - (\dualnorm{\g}- \min(\dualnorm{\g},c))^2))\\
    &= 
    \textstyle
        \min(b\dualnorm{\g},\frac{1}{2\lam} (\dualnorm{\g}^2 - (\dualnorm{\g}- c)_+^2))\\
    &\leq
    \textstyle
         \min(\frac{1}{2\lam}\dualnorm{\g}^2, \frac{1}{\lam}c\dualnorm{\g}, b\dualnorm{\g}).
\end{align}
\end{lemma}
By \cref{omd_iterate_bound,norm_constrained_conjugate} and the definition of $a\defeq \diam{\wset}$, each summand in our regret bound \cref{intermediate_soomd_regret_bound} satisfies
\begin{align}
&\inner{\g_t-\tildeg_t}{\ww_t - \ww_{t+1}} - \textstyle
\frac{\lam_{t+1}}{2}\norm{\ww_t - \ww_{t+1}}^2 
    \leq 
    \,\displaystyle\sup_{\vv \in \reals^d : \norm{\vv}\leq \min(\frac{1}{\lam_{t+1}}\norm{\g_t + \tildeg_{t+1} - \tildeg_t}_*,a)}
    \textstyle\inner{\g_t-\tildeg_t}{\vv} - \frac{\lam_{t+1}}{2}\norm{\vv}^2 \\
    =
    \,&\textstyle
        \min\big(
        a\norm{\g_t - \tildeg_t}_*,
        \frac{1}{2\lam_{t+1}}(\norm{\g_t - \tildeg_t}_*^2 
    - 
        (\norm{\g_t - \tildeg_t}_* 
    - 
        \norm{\g_t + \tildeg_{t+1} - \tildeg_t}_*
        )_+^2)
        \big) 
\end{align}
yielding the advertised result.
\end{proof}

\subsection{Proof of \cref{norm_constrained_conjugate}: Norm-constrained conjugate}\label{proof_norm_constrained_conjugate}
By the definition of the dual norm, 
\begin{align}
    \sup_{\vv \in \reals^d : \norm{\vv}\leq \min(\frac{c}{\lam},b)} 
    \textstyle
        \inner{\g}{\vv} - \frac{\lam}{2}\norm{\vv}^2  
    &= 
        \sup_{a \leq \min(\frac{c}{\lam},b)} \sup_{\vv \in \reals^d : \norm{\vv}\leq a} 
    \textstyle
        \inner{\g}{\vv} - \frac{\lam}{2}a^2 
    = 
    \displaystyle
        \sup_{a \leq \min(\frac{c}{\lam},b)}  
    \textstyle
        a \dualnorm{\g} - \frac{\lam}{2}a^2 \\
    &= 
    \textstyle
        \frac{1}{\lam} \min(\dualnorm{\g}, c,b\lam)(\dualnorm{\g} - \frac{1}{2} \min(\dualnorm{\g}, c,b\lam))
    \leq
        \min(\frac{1}{\lam} c \dualnorm{\g}, b\dualnorm{\g}).
\end{align}
We compare to the values of less constrained optimization problems to obtain the final inequalities:
\begin{align}
    \displaystyle
        \sup_{a \leq \min(\frac{c}{\lam},b)}  
    \textstyle
        a \dualnorm{\g} - \frac{\lam}{2}a^2 
    &\leq
    \displaystyle
        \sup_{a \leq \frac{c}{\lam}}  
    \textstyle
        a \dualnorm{\g} - \frac{\lam}{2}a^2 
    = 
    \textstyle
        \frac{1}{\lam} \min(\dualnorm{\g}, c)(\dualnorm{\g} - \frac{1}{2} \min(\dualnorm{\g}, c))\\
    &\leq
    \displaystyle
        \sup_{a > 0}  
    \textstyle
        a \dualnorm{\g} - \frac{\lam}{2}a^2 
        = \frac{1}{\lam} \half \dualnorm{\g}^2.
\end{align}

\section{Proof of \cref{dorm_is_odftrl_dorm+_is_doomd}: \DORM is \ODAFTRL and \DORM+ is \DOOMD}\label{proof_dorm_is_odftrl_dorm+_is_doomd}
Our derivations will make use of several facts about $\ell^p$ norms, summarized in the next lemma.
\begin{lemma}[$\ell^p$ norm facts]\label{pnorm-facts}
For $p\in (1,\infty)$, $\reg(\ww) = \half \pnorm{\ww}^2$, and any vectors $\ww, \vv \in \reals^d$ and $\worth_0\in\orthantd$,
\begin{align}
\grad \reg(\ww)
    &= \grad \texthalf \pnorm{\ww}^2 
    = \sign(\ww) |\ww|^{p-1}/\pnorm{\ww}^{p-2}  \label{grad_pnorm_squared}\\
\inner{\ww}{\grad \reg(\ww)}
    &= \pnorm{\ww}^2 = 2\reg(\ww) \\
\reg^*(\vv) 
    &= \sup_{\ww\in\reals^d} \inner{\ww}{\vv} - \reg(\ww) = \texthalf \qnorm{\vv}^2 \qtext{for} 1/q = 1 - 1/p \label{pnorm_squared_conjugate}\\
\grad \reg^*(\vv) 
    &= \sign(\vv) |\vv|^{q-1}/\qnorm{\vv}^{q-2} \\
\reg_+^*(\vv) 
    &= \sup_{\ww\in\orthantd} \inner{\ww}{\vv} - \reg(\ww) 
    = \sup_{\ww\in\reals^d} \inner{\ww}{(\vv)_+} - \reg(\ww) 
    = \texthalf \qnorm{(\vv)_+}^2 \\
\grad \reg_+^*(\vv) 
    &= \argmax_{\ww\in\orthantd} \inner{\ww}{\vv} - \reg(\ww) = \argmin_{\ww\in\orthantd} \reg(\ww) - \inner{\ww}{\vv} = (\vv)_+^{q-1}/\qnorm{(\vv)_+}^{q-2} \label{argmin_pnorm_reg}\\
\min_{\worth \in \orthantd} \Breg_{\lam\reg}(\worth, \worth_0) - \inner{\vv}{\worth} 
    &= \lam(\inner{\worth_0}{\grad \reg(\worth_0)} -\reg(\worth_0) - \sup_{\worth\in\orthantd}  \inner{\worth}{\grad \reg(\worth_0) + \vv/\lam} - \reg(\worth)) \\
    &= \lam(\inner{\worth_0}{\grad \reg(\worth_0)} - \reg(\worth_0) - \reg_+^*(\grad \reg(\worth_0) + \vv/\lam)) \\
    &= \lam(\reg(\worth_0) - \reg_+^*(\grad \reg(\worth_0) + \vv/\lam)) \\
    &= \lam(\reg(\worth_0) - \texthalf\qnorm{(\grad \reg(\worth_0) + \vv/\lam)_+}^2) \\
    &= \lam(\texthalf\pnorm{\worth_0}^2 - \texthalf\qnorm{(\worth_0^{p-1}/\pnorm{\worth_0}^{p-2} + \vv/\lam)_+}^2).
\end{align}
\end{lemma}
\begin{proof}
The fact \cref{grad_pnorm_squared} follows from the chain rule as
\begin{talign}
    \grad_j \half \pnorm{\ww}^2
        &= \half \grad_j (\pnorm{\ww}^p)^{2/p}
        = \frac{1}{p} (\pnorm{\ww}^p)^{(2/p)-1} \grad_j \pnorm{\ww}^p 
        = \frac{1}{p} \pnorm{\ww}^{2-p} \grad_j \sum_{j'=1}^d |\ww_{j'}|^p \\
        &= \frac{1}{p} \pnorm{\ww}^{2-p} p \sign(\ww_{j}) |\ww_{j}|^{p-1}
        = \sign(\ww_{j}) |\ww_{j}|^{p-1}/\pnorm{\ww}^{p-2}.
\end{talign}

The fact \cref{pnorm_squared_conjugate} follows from \cref{norm_constrained_conjugate} as $\qnorm{\cdot}$ is the dual norm of $\pnorm{\cdot}$.
\end{proof}
We now prove each claim in turn.
\subsection{\DORM  is \ODAFTRL}
Fix $p \in (1,2]$, $\lam > 0$, and $t\geq 0$.
The \ODAFTRL iterate with hint $-\h_{t+1}$, $\wset \defeq \orthantd$, $\psi(\worth) = \half \pnorm{\worth}^2$, loss subgradients $\g_{1:t-D}^{\ODAFTRL} = -\rr_{1:t-D}$, and regularization parameter $\lam$ takes the form
\begin{align}
\argmin_{\worth\in\orthantd}\ &\lam\reg(\worth)- \inner{\worth}{\h_{t+1} + \rr_{1:t-D}} \\
        &= \argmin_{\worth\in\orthantd}  \reg(\worth)- \inner{\worth}{(\h_{t+1} + \rr_{1:t-D})/\lam} \\
        &= ((\rr_{1:t-D} + \h_{t+1})/\lam)_+^{q-1} / \qnorm{((\rr_{1:t-D} + \h_{t+1})/\lam)_+}^{q-2} \qtext{by \cref{argmin_pnorm_reg}} \\
        &= ((\rr_{1:t-D} + \h_{t+1})/\lam)_+^{q-1} \norm{((\rr_{1:t-D} + \h_{t+1})/\lam)_+^{q-1}}_p^{p-2} \qtext{since $(p-1)(q-1) = 1$}\\
        &= \worth_{t+1} \norm{\worth_{t+1}}_p^{p-2}
\end{align}
proving the claim. 

\subsection{\DORMP  is \DOOMD}
Fix $p \in (1,2]$ and $\lam > 0$,
and let $(\worth_{t})_{t\geq 0}$ denote the unnormalized iterates generated by \DORMP with hints $\h_t$, instantaneous regrets $\rr_{t}$,
regularization parameter $\lambda$, and hyperparameter $q$.
For $p = q/(q-1)$, let $(\wbar_{t})_{t\geq 0}$ denote the sequence generated by \DOOMD with 
$\wbar_0 = \boldzero$, hints $-\h_t$, $\wset \defeq \orthantd$, $\psi(\worth) = \half \pnorm{\worth}^2$, loss subgradients $\g_{t}^{\DOOMD} = -\rr_{t}$, and regularization parameter $\lam$. 
We proceed by induction to show that, for each $t$, $\wbar_{t} = \worth_{t} \pnorm{\worth_{t}}^{p-2}$.

\paragraph{Base case} By assumption, $\wbar_0 = \boldzero = \worth_0 \pnorm{\worth_0}^{p-2}$, confirming the base case.

\paragraph{Inductive step}
Fix any $t \geq 0$ and assume that for each $s \leq t$, $\wbar_s = \worth_{s} \pnorm{\worth_{s}}^{p-2}$. Then, by the definition of \DOOMD and our $\ell^p$ norm facts, 
\begin{align}
\bar{\ww}_{t+1} 
    &= \argmin_{\wbar\in\orthantd}\, \inner{-\h_{t+1}+\h_t-\rr_{t-D}}{\wbar} + \Breg_{\lam\reg}(\wbar,\wbar_t) \\
    &= \argmin_{\wbar\in\orthantd} \lam (\reg(\wbar)-\reg(\wbar_{t}) - \inner{\wbar-\wbar_{t}}{\grad \reg(\wbar_{t})}) + \inner{-\h_{t+1}+\h_t-\rr_{t-D}}{\wbar}  \\
    &= \argmin_{\wbar\in\orthantd}  \reg(\wbar)- \inner{\wbar}{\grad \reg(\wbar_{t}) + (\rr_{t-D}-\h_t+\h_{t+1})/\lam} \\
    &= \argmin_{\wbar\in\orthantd}  \reg(\wbar)- \inner{\wbar}{\wbar_t^{p-1} / \pnorm{\wbar_t}^{p-2} + (\rr_{t-D}-\h_t+\h_{t+1})/\lam} \qtext{by \cref{grad_pnorm_squared}}\\
    &= \argmin_{\wbar\in\orthantd}  \reg(\wbar)- \inner{\wbar}{\worth_{t}^{p-1} + (\rr_{t-D}-\h_t+\h_{t+1})/\lam} \qtext{by the inductive hypothesis}\\
    &= (\worth_{t}^{p-1} + (\rr_{t-D}-\h_t+\h_{t+1})/\lam)_+^{q-1} / \qnorm{(\worth_{t}^{p-1} + (\rr_{t-D}-\h_t+\h_{t+1})/\lam)_+}^{q-2} \qtext{by \cref{argmin_pnorm_reg}}\\
    &= (\worth_{t}^{p-1} + (\rr_{t-D}-\h_t+\h_{t+1})/\lam)_+^{q-1} \pnorm{(\worth_{t}^{p-1} + (\rr_{t-D}-\h_t+\h_{t+1})/\lam)_+^{q-1}}^{p-2}
    \qtext{since $(p-1)(q-1)=1$}\\
    &= \worth_{t+1} \pnorm{\worth_{t+1}}^{p-2},
\end{align}
completing the inductive step.
\section{Proof of \cref{dorm_dorm+_lambda_independent}: \DORM and \DORMP  are independent of $\lam$}\label{proof_dorm_dorm+_lambda_independent}
We will prove the following more general result, from which the stated result follows immediately.

\begin{lemma}[\DORM and \DORMP are independent of $\lam$]
\label{dorm_dorm+_lambda_independent_detailed}
Consider either \DORM or \DORMP plays $\worth_t$ as a function of $\lam > 0$, and suppose that for all time points $t$, the observed subgradient $\g_t$ and chosen hint $\h_{t+1}$ only depend on $\lam$ through $(\ww_s,\lam^{q-1}\worth_s,\g_{s-1},\h_{s})_{s\leq t}$
and $(\ww_s,\lam^{q-1}\worth_s,\g_{s},\h_{s})_{s\leq t}$ respectively.
Then if $\lam^{q-1}\worth_0$ is independent of the choice of $\lambda > 0$, then so is $\lam^{q-1} \worth_{t}$ for all time points $t$.
As a result, $\ww_t \propto \lam^{q-1} \worth_{t}$ is also independent of the choice of $\lambda > 0$ at all time points.
\end{lemma}
\begin{proof}
We prove each result by induction on $t$.

\subsection{Scaled \DORM iterates $\lam^{q-1} \worth_{t}$ are independent of $\lam$}
\paragraph{Base case}
By assumption, $\h_1$ is independent of the choice of $\lam > 0$.
Hence $\lam^{q-1}\worth_1 = (\h_1)_+^{q-1}$ is independent of $\lam > 0$, confirming the base case.

\paragraph{Inductive step}
Fix any $t \geq 0$, suppose $\lam^{q-1}\worth_{s}$ is independent of the choice of $\lambda > 0$ for all $s \leq t$, and consider
\begin{talign}
\lam^{q-1} \worth_{t+1} = (\rr_{1:t-D} + \h_{t+1})_+^{q-1}.
\end{talign}  

Since $\rr_{1:t-D}$ depends on $\lam$ only through $\ww_{s}$ and $\g_{s}$ for $s \leq t-D$, our $\lam$ dependence assumptions for $(\g_{s},\h_{s+1})_{s\leq t}$; the fact that, for each $s$, $\ww_s \propto \lam^{q-1}\worth_{s}$; and our inductive hypothesis together imply that $\lam^{q-1} \worth_{t+1}$ is independent of $\lam > 0$.
\subsection{Scaled \DORMP iterates $\lam^{q-1} \worth_{t}$ are independent of $\lam$}
\paragraph{Base case}
By assumption, $\lam^{q-1}\worth_0$ is independent of the choice of $\lambda > 0$, confirming the base case.

\paragraph{Inductive step} 
Fix any $t \geq 0$ and suppose $\lam^{q-1}\worth_{s}$ is independent of the choice of $\lambda > 0$ for all $s \leq t$. Since $(p-1)(q-1)=1$,
\begin{align}
\lam^{q-1} \worth_{t+1} 
    = (\lam \worth_{t}^{p-1} + \rr_{t-D} - \h_t + \h_{t+1})_+^{q-1}
    = ((\lam^{q-1} \worth_{t})^{p-1} + \rr_{t-D} - \h_t + \h_{t+1})_+^{q-1}.
\end{align}  
Since $\rr_{t-D}$ depends on $\lam$ only through $\ww_{t-D}$ and $\g_{t-D}$, our $\lam$ dependence assumptions for $(\g_{s},\h_{s+1})_{s\leq t}$; the fact that, for each $s \leq t$, $\ww_s \propto \lam^{q-1}\worth_{s}$; and our inductive hypothesis together imply that $\lam^{q-1} \worth_{t+1}$ is independent of $\lam > 0$.
\end{proof}
\section{Proof of \cref{dorm_dorm+_regret}: \DORM and \DORMP regret}\label{proof_dorm_dorm+_regret}

Fix any $\lambda > 0$ and $\uu\in\simplexd$, consider the unnormalized \DORM or \DORMP iterates $\worth_t$, and define $\wbar_t = \worth_t \pnorm{\worth_t}^{p-2}$ for each $t$. 
For either algorithm, we will bound our regret in terms of the surrogate losses 
\begin{talign}
\surrloss_t(\worth) 
    \defeq 
        -\inner{\rr_t}{\worth}
    = 
        \inner{\g_t}{\worth}-\inner{\worth}{\boldone} \inner{\g_t}{\ww_t}
\end{talign}
defined for $\worth \in \orthantd$.
Since $\surrloss_t(\uu) = \inner{\g_t}{\uu-\ww_t}$, $\surrloss_t(\wbar_t) = 0$, and each $\loss_t$ is convex, we have
\begin{talign}
\regret_T(\uu) 
    = \sum_{t=1}^T \loss_t(\ww_t) - \loss_t(\uu)
    \leq \sum_{t=1}^T \inner{\g_t}{\ww_t-\uu}
    = \sum_{t=1}^T \surrloss_t(\wbar_t) - \surrloss_t(\uu).
\end{talign}
For \DORM, \cref{dorm_is_odftrl_dorm+_is_doomd} implies that $(\wbar_t)_{t\geq 1}$ are \ODFTRL iterates, so the \ODFTRL regret bound (\cref{odftrl_regret}) and the fact that $\psi$ is $1$-strongly convex with respect to $\norm{\cdot} = \sqrt{p-1}\pnorm{\cdot}$ \citep[see][Lemma 17]{shalev2007online} with $\norm{\cdot}_* = \frac{1}{\sqrt{p-1}}\qnorm{\cdot}$ imply
\begin{talign}
\regret_T(\uu) 
    \leq \frac{\lam}{2}\pnorm{\uu}^2
    + \frac{1}{\lam(p-1)}\sum_{t=1}^T\bb_{t,q}.
\end{talign}
Similarly, for \DORMP, \cref{dorm_is_odftrl_dorm+_is_doomd} implies that $(\wbar_t)_{t\geq 0}$ are \DOOMD iterates with $\wbar_0 = \boldzero$, so the \DOOMD regret bound (\cref{doomd_regret}) and the strong convexity of $\psi$ yield
\begin{talign}
\regret_T(\uu) 
    &\leq \Breg_{\frac{\lam}{2}\pnorm{\cdot}^2}(\uu,\boldzero) 
    + \frac{1}{\lam(p-1)}\sum_{t=1}^T\bb_{t,q}
    = \frac{\lam}{2}\pnorm{\uu}^2
    + \frac{1}{\lam(p-1)}\sum_{t=1}^T\bb_{t,q}.
\end{talign}
Since, by \cref{dorm_dorm+_lambda_independent},  the choice of $\lambda$ does not impact the iterate sequences played by \DORM and \DORMP, we may take the infimum over $\lambda > 0$ in these regret bounds.
The second advertised inequality comes from the identity $\frac{1}{p-1} = q-1$ and the norm equivalence relations $\qnorm{\vv}\leq d^{1/q}\infnorm{\vv}$ and $\pnorm{\vv} \leq \onenorm{\vv}=1$ for $\vv\in\reals^d$, as shown in \cref{norm-equality} below.
The final claim follows as 
\begin{talign}
\inf_{q' \geq 2} d^{2/q'}(q'-1)
    = \inf_{q' \geq 2} 2^{2\log_2(d)/q'}(q'-1)
    \leq 2^{2\log_2(d)/(2\log_2(d))}(2\log_2(d)-1)
    = 2(2\log_2(d)-1)
\end{talign}
since $d > 1$.

\begin{lemma}[Equivalence of $p$-norms] \label{norm-equality}
If $\x \in \R^n$ and $q > q' \geq 1$, then  $\norm{\x}_q \leq \norm{\x}_{q'} \leq n^{(1/q' - 1/q)} \norm{\x}_q$. 
\end{lemma}
\begin{proof}
To show $\|\mathbf{x}\|_q \leq \|\mathbf{x}\|_{q'}$ for $q > q' \geq 1$, suppose without loss of generality that $\|\mathbf{x}\|_{q'}=1$. Then, $\|\mathbf{x}\|_q^q = \sum_{i=1}^n |x_i|^q \leq \sum_{i=1}^n |x_i|^{q'} = \|\mathbf{x}\|_{q'}^{q'}=1$. Hence $\|\mathbf{x}\|_q \leq 1 = \|\mathbf{x}\|_{q'}$.

For the inequality $\|\mathbf{x}\|_{q'} \leq n^{1/q'-1/q} \|\mathbf{x}\|_q$, applying \Holder's inequality yields
\begin{talign}
   \|\mathbf{x}\|_{q'}^{q'}= \sum_{i=1}^n 1 \cdot |x_i|^{q'} \leq \left(\sum_{i=1}^n 1\right)^{1 - \frac{q'}{q}}  \left(\sum_{i=1}^n |x_i|^{q}\right)^{\frac{q'}{q}} = n^{1 - \frac{q'}{q}} \|\mathbf{x}\|_q^{q'},
\end{talign}
so $\|\mathbf{x}\|_{q'} \leq n^{1/q'-1/q}\|\mathbf{x}\|_q$.
\end{proof}

\section{Proof of \cref{odaftrl_regret}: \ODAFTRL regret}\label{proof_odaftrl_regret}

Since \ODAFTRL is an instance of \OAFTRL with 
$\tildeg_{t+1} = \h_{t+1} - \sum_{s=t-D+1}^t\g_s$,
the \ODAFTRL result follows immediately from the \OAFTRL regret bound, \cref{oaftrl_regret}.
\section{Proof of \cref{dub_regret}: \DUB Regret} \label{proof_dub_regret}
Fix any $\uu\in\wset$.  By \cref{odaftrl_regret}, \ODAFTRL admits the regret bound 
\begin{talign}
\regret_T(\uu) 
    \leq 
        \lambda_{T} \reg(\uu) + \sum_{t=1}^T \min(\frac{1}{\lam_t} \bbftrl{t}, \abftrl{t}).
\end{talign}
To control the second term in this bound, we apply the following lemma proved in \cref{proof_upper-bound-lam-bound}.

\begin{lemma}[\DUB-style tuning bound] \label{upper-bound-lam-bound}
Fix any $\alpha > 0$ and any non-negative sequences $(a_t)_{t=1}^T$, $(b_t)_{t=1}^T$.
If 
\begin{talign} 
\Delta_{t+1}^* 
    \defeq
        2\max_{j\leq t-D-1} a_{j-D+1:j} + \sqrt{\sum_{i=1}^{t-D} a_{i}^2 +2  \alpha b_i} 
    \leq
        \alpha \lambda_{t+1}
        \qtext{for each} t
\end{talign}
then
\begin{talign}
\sum_{t=1}^T \min(b_t^2 / \lambda_t, a_{t})
    \leq 
        \Delta_{T+D+1}^*
    \leq
        \alpha \lambda_{T+D+1}.
\end{talign}
\end{lemma}

Since 
$\lam_T \leq \lam_{T+D+1}$, 
the result now follows by setting $a_t = \abftrl{t}$ and $b_t = \bbftrl{t}$, so that 
\begin{talign}
\regret_T(\uu) \leq  \lambda_{T} \reg(\uu) + \alpha \lam_{T+D+1} \leq (\reg(\uu) + \alpha) \lam_{T+D+1}.
\end{talign}

\subsection{Proof of \cref{upper-bound-lam-bound}: \DUB-style tuning bound}  \label{proof_upper-bound-lam-bound}
We prove the claim 
\begin{talign}
\Delta_t \defeq \sum_{i=1}^t \min(b_i / \lambda_i, a_i)
    \leq
\Delta_{t+D+1}^*
    \leq 
        \alpha \lambda_{t+D+1}
\end{talign}
by induction on $t$. 

\paragraph{Base case} 
For $t\in [D+1]$, 
\begin{talign}
\sum_{i=1}^t \min(b_i / \lambda_i, a_{i}) 
    \leq
        a_{1:t-1} + a_t
    \leq 
        2\max_{j\leq t-1} a_{j-D+1:j} + \sqrt{\sum_{i=1}^{t} a_i^2 + 2\alpha b_i}
    = 
        \Delta_{t+D+1}^*
    \leq 
        \alpha\lam_{t+D+1}
\end{talign} 
confirming the base case. 

\paragraph{Inductive step}
Now fix any $t + 1\geq D+2$ and suppose that 
\begin{talign}
\Delta_{i} 
\leq 
    \Delta_{i+D+1}^*
\leq 
    \alpha \lambda_{i+D+1}
\end{talign}
for all $1\leq i \leq t$. 
We apply this inductive hypothesis to deduce that, for each $0 \leq i \leq t$, 
\begin{align}
\Delta_{i+1}^2 - \Delta_i^2
    &= \left(\Delta_i + \min(b_{i+1} / \lambda_{i+1} ,a_{i+1}) \right)^2 - \Delta_i^2
    = 2 \Delta_i \min(b_{i+1} / \lambda_{i+1},a_{i+1}) + \min(b_{i+1} / \lambda_{i+1},a_{i+1})^2\\
    &= 2 \Delta_{i-D} \min(b_{i+1} / \lambda_{i+1},a_{i+1}) + 2 (\Delta_i-\Delta_{i-D})\min(b_{i+1} / \lambda_{i+1},a_{i+1}) + \min(b_{i+1} / \lambda_{i+1},a_{i+1})^2 \\
    &= 2 \Delta_{i-D} \min(b_{i+1} / \lambda_{i+1},a_{i+1}) + 2 \sum_{j=i-D+1}^i \min(b_j/\lam_j, a_j)\min(b_{i+1} / \lambda_{i+1},a_{i+1}) + \min(b_{i+1} / \lambda_{i+1},a_{i+1})^2 \\
    &\leq 2 \alpha \lam_{i+1} \min(b_{i+1} / \lambda_{i+1},a_{i+1}) + 2 a_{i-D+1:i}\min(b_{i+1} / \lambda_{i+1},a_{i+1}) + a_{i+1}^2 \\    
    &\leq 2 \alpha b_{i+1} + a_{i+1}^2 + 2 a_{i-D+1:i}\min(b_{i+1} / \lambda_{i+1},a_{i+1}). %
\end{align}
Now, we sum this inequality over $i=0, \dots, t$, to obtain
\begin{talign}
\Delta^2_{t+1} 
    &\leq 
        \sum_{i=0}^{t} (2  \alpha b_{i+1} + a_{i+1}^2) + 2  \sum_{i=0}^{t} a_{i-D+1:i}\min(b_{i+1} / \lambda_{i+1},a_{i+1}) \\
    &=
        \sum_{i=1}^{t+1} (2  \alpha b_{i} + a_{i}^2) + 2  \sum_{i=1}^{t+1} a_{i-D:i-1}\min(b_{i} / \lambda_{i},a_{i}) \\
    &\leq 
        \sum_{i=1}^{t+1} (a_{i}^2 + 2  \alpha b_i) + 2  \max_{j\leq t} a_{j-D+1:j} \sum_{i=1}^{t+1} \min(b_i / \lambda_{i}, a_{i}) \\
    &=
        \sum_{i=1}^{t+1} (a_{i}^2 + 2  \alpha b_i) + 2  \Delta_{t+1} \max_{j\leq t} a_{j-D+1:j} .
\end{talign}
Solving this quadratic inequality and applying the triangle inequality, we have
\begin{talign}
\Delta_{t+1} 
    &\leq 
        \max_{j\leq t} a_{j-D+1:j}+ \half \sqrt{(2\max_{j\leq t} a_{j-D+1:j})^2 + 4\sum_{i=1}^{t+1} a_{i}^2 + 2  \alpha b_i} \\
    &\leq 
        2\max_{j\leq t} a_{j-D+1:j} + \sqrt{\sum_{i=1}^{t+1} a_{i}^2 + 2  \alpha b_i}
    = 
        \Delta_{t+D+2}^*
    \leq 
        \alpha \lambda_{t+D+2}.
\end{talign}

\section{Proof of \cref{adahedged_regret}:  \AdaHedgeD Regret} \label{proof_adahedged_regret}
Fix any $\uu\in\wset$.
Since the \AdaHedgeD regularization sequence $(\lam_t)_{t\geq 1}$ is non-decreasing, \cref{oaftrl_regret} gives the regret bound
\begin{talign}
\regret_T(\uu) 
    &\leq 
        \lam_{T} \reg(\uu) + \sum_{t=1}^T \delta_t 
    = 
        \lam_{T} \reg(\uu) + \alpha \lam_{T+D+1}
    \leq 
        (\reg(\uu) +  \alpha) \lam_{T+D+1},
\end{talign}
and the proof of \cref{oaftrl_regret} gives the upper estimate \cref{delta_min_bound}:
\begin{talign}
\label{delta_a_b_bound}
\delta_t 
    \leq 
        \min\Big(\frac{\bbftrl{t}}{ \lam_t}, \abftrl{t} \Big)
    \qtext{for all} t \in [T].
\end{talign}
Hence, it remains to bound $\lam_{T+D+1}$. 
Since $\lam_1 =  \dots = \lam_{D+1} = 0$ and $\alpha(\lam_{t+1} - \lam_t) = \delta_{t-D}$ for $t \geq D+1$, 
\begin{talign}
\alpha \lam_{T+D+1}^2 
    &= 
        \sum_{t=1}^{T+D} \alpha(\lam_{t+1}^2 - \lam_{t}^2) 
    = 
        \sum_{t=D+1}^{T+D} \left(\alpha(\lam_{t+1} - \lam_{t})^2 + 2\alpha(\lam_{t+1} - \lam_{t}) \lam_{t} \right)\\
    & = 
        \sum_{t=1}^{T} \left(\delta_t^2/\alpha + 2 \delta_t \lam_{t+D} \right)  \qtext{by the definition of $\lam_{t+1}$}\\
    & = 
        \sum_{t=1}^{T} \left(\delta_t^2/\alpha + 2 \delta_t \lam_{t} + 2 \delta_t (\lam_{t+D}- \lam_{t}) \right) \\
    & \leq
        \sum_{t=1}^{T} \left(\delta_t^2/\alpha + 2 \delta_t \lam_{t} + 2 \delta_t \max_{t\in[T]}(\lam_{t+D}- \lam_{t}) \right) \\
    & =
        \sum_{t=1}^{T} \left(\delta_t^2/\alpha + 2 \delta_t \lam_{t} \right) + 2 \lam_{T+D+1} \max_{t\in[T]} \delta_{t-D:t-1} \\
    &\leq 
        \sum_{t=1}^{T} \left(\abftrl{t}^2 / \alpha + 2 \bbftrl{t} \right) + 2 \lam_{T+D+1} \max_{t \in [T]} \abftrl{t-D:t-1}
        \qtext{by \cref{delta_a_b_bound}.}
\end{talign}
Solving the above quadratic inequality for $\lam_{T+D+1}$ and applying the triangle inequality, we find
\begin{talign}
\alpha\lam_{T+D+1}
    &\leq \max_{t \in [T]} \abftrl{t-D:t-1} + \half\sqrt{4(\max_{t \in [T]} \abftrl{t-D:t-1})^2 + 4\sum_{t=1}^{T} \abftrl{t}^2 + 2 \alpha \bbftrl{t}} \\
    &\leq 2\max_{t \in [T]} \abftrl{t-D:t-1}
    + \sqrt{\sum_{t=1}^{T} \abftrl{t}^2 + 2 \alpha \bbftrl{t}}.
\end{talign}

\section{Proof of \cref{base-and-hinter-regret}: Learning to hint regret} \label{adaptive-hinting-proofs}

We begin by bounding the hinting problem regret. 
Since \DORMP is used for the hinting problem, the following result is an immediate corollary of \cref{dorm_dorm+_regret}.

\begin{corollary}[\DORMP hinting problem regret] \label{dorm+_hinting_regret}
With convex losses $l_t(\omega) = f_t(H_t \omega)$ and no meta-hints, the \DORMP hinting problem iterates $\omega_t$
satisfy, for each $v\in\simplexm$,
\begin{talign}
\hregret_T(v) 
    &\defeq \sum_{t=1}^T l_t(\omega_t) -  \sum_{t=1}^T l_t(v) 
    \leq \sqrt{\frac{m^{2/q}(q-1)}{2}\sum_{t=1}^T \beta_{t,\infty}}   \qtext{for} 
    \\
\beta_{t, \infty} 
    &= 
    \begin{cases}
    \huber(\staticnorm{\sum_{s=t-D}^t\rho_s }_{\infty}, \norm{\rho_{t-D}}_{\infty}), & \text{for } t < T \\
    \half\staticnorm{\sum_{s=t-D}^t\rho_s }_{\infty}^2, & \text{for } t = T
    \end{cases}
    \\
\qtext{where}
\rho_{t} 
    &\defeq \boldone\inner{\gamma_{t}}{\omega_{t}} - \gamma_{t}
    \qtext{for}
    \gamma_t \in \subdiff l_t(\omega_t)
    \qtext{is the \emph{instantaneous hinting problem regret}.}
\end{talign}
If, in addition, $q = \argmin_{q' \geq 2} m^{2/q'}(q'-1)$, then $\hregret_T(v) \leq \sqrt{(2\log_2(m)-1) \sum_{t=1}^T \beta_{t,\infty}}$. 
\end{corollary}

Our next lemma, proved in \cref{proof-beta-bound-general},  provides an interpretable bound for each $\beta_{t,\infty}$ term in terms of the hinting problem subgradients $(\gamma_t)_{t\geq 1}$.
\begin{lemma}[Hinting problem subgradient regret bound] \label{beta-bound-general}
Under the notation and assumptions of \cref{dorm+_hinting_regret}, 
\begin{talign}
    \beta_{t,\infty} 
        &\leq 
        \begin{cases}
            \huber( \xi_t, \zeta_t ) & \text{if } t < T \\
            \half\xi_t & \text{if } t = T
        \end{cases},
        \qtext{for} 
        \\
    \xi_t 
        &\defeq 4(D+1)\sum_{s=t-D}^{t} \norm{\gamma_{s}}_{\infty}^2 
        \qtext{and} 
        \\
    \zeta_t 
        &\defeq 4  \norm{\gamma_{t-D}}_{\infty} \sum_{s=t-D}^t \norm{\gamma_{s}}_{\infty}.
\end{talign}
\end{lemma}
Now fix any $\uu\in\wset$.
We invoke \cref{base_assumptions},
\cref{dorm+_hinting_regret}, and \cref{beta-bound-general} in turn to bound the base problem regret
\begin{talign}
\regret_T(\uu) 
    &= \sum_{t=1}^T \loss_t(\ww_t) - \loss_t(\uu) \\    
    &\leq C_0(\uu) + C_1(\uu) \sqrt{\sum_{t=1}^T f_t(\h_t(\omega_t))} \qtext{by \cref{base_assumptions}} \\
    &\leq C_0(\uu) + C_1(\uu) \sqrt{\inf_{v\in\vset} \sum_{t=1}^T f_t(\h_t(v)) + \sqrt{(2\log_2(m)-1) \sum_{t=1}^T \beta_{t,\infty}}} \qtext{by \cref{dorm+_hinting_regret}} \\
    &\leq C_0(\uu) + C_1(\uu) \sqrt{\inf_{v\in\vset} \sum_{t=1}^T f_t(\h_t(v)) + \sqrt{(2\log_2(m)-1) (\half\xi_T + \sum_{t=1}^{T-1} \huber(\xi_t, \zeta_t))}} \qtext{by \cref{beta-bound-general}.}
\end{talign}
The advertised bound now follows from the triangle inequality.
\subsection{Proof of \cref{beta-bound-general}: Hinting problem subgradient regret bound} \label{proof-beta-bound-general}
Fix any $t\in[T]$.
The triangle inequality implies that
\begin{talign}
\infnorm{\rho_t}
    = 
        \infnorm{\gamma_t - \boldone \inner{\omega_t}{\gamma_t}}
    \leq 
        \infnorm{\gamma_t} + |\inner{\omega_t}{\gamma_t}|
    \leq
        2\infnorm{\gamma_t}
\end{talign}
since $\omega_t\in\simplexm$.
We repeatedly apply this finding in conjunction with Jensen's inequality to conclude
\begin{talign}
\staticnorm{\sum_{s=t-D}^t\rho_s}_{\infty}^2 
    &\leq  (D+1) \sum_{s=t-D}^t \norm{\rho_s}_{\infty}^2 
    \leq 4(D+1)\sum_{s=t-D}^{t} \norm{\gamma_{s}}_{\infty}^2 
    \qtext{and}
    \\
\norm{\rho_{t-D}}_{\infty} \staticnorm{\sum_{s=t-D}^t\rho_s}_{\infty} 
    &\leq
        \norm{\rho_{t-D}}_{\infty} \sum_{s=t-D}^t \norm{\rho_s}_{\infty} 
    \leq 
        4 \norm{\gamma_{t-D}}_{\infty} \sum_{s=t-D}^t \norm{\gamma_{s}}_{\infty}.
\end{talign}

\section{Examples: Learning to Hint with \DORMP and \AdaHedgeD} \label{proof_adahedged-example}
By \cref{adahedged_regret}, \AdaHedgeD satisfies \cref{base_assumptions} with 
$f_t(\h_t) =  \staticnorm{\rr_{t}}_*  \staticnorm{\h_t -\sum_{s=t-D}^t\rr_s}_*
            \geq 
            \frac{\abftrl{t}^2 + 2\alpha \bbftrl{t}}{\diam{\wset}^2 + 2\alpha}$, 
$C_1(\uu) = \sqrt{\diam{\wset}^2 + 2\alpha}$, and
$C_0(\uu) = 2\diam{\wset}\max_{t \in [T]} \sum_{s=t-D}^{t-1}\dualnorm{\g_{s}}$.

By \cref{dorm_dorm+_regret}, \DORMP satisfies \cref{base_assumptions} with
$f_t(\h) = \norm{\rr_{t-D} + \h_{t+1} - \h_t}_{q}  \norm{\h -\sum_{s=t-D}^t\rr_s}_{q}$, $C_0(\uu) = 0$, and $C_1(\uu) = \sqrt{\frac{\pnorm{\uu}^2}{2(p-1)}}$.

These choices give rise to the hinting losses
\begin{talign}
l_t^{\DORMP}(\omega) 
    &= %
        \norm{\rr_{t-D} + \h_{t+1} - \h_t}_{q}  \norm{H_t \omega -\sum_{s=t-D}^t\rr_s}_{q} \label{eq:hinting-loss-doomd}
    \qtext{and}        
    \\
l_t^{\AdaHedgeD}(\omega)
    &= %
        \qnorm{\g_{t}} \qnorm{H_t \omega -\sum_{s=t-D}^t\g_s}
    \qtext{when}
        \dualnorm{\cdot}=\qnorm{\cdot}
    \qtext{for}
        q \in [1,\infty]. \label{eq:hinting-loss-odaftrl} 
\end{talign}
The following lemma, proved in \cref{proof_hinting_loss_subgradient}, identifies subgradients of these hinting losses.
\begin{lemma}[Hinting loss subgradient]
\label{hinting_loss_subgradient}
If $l_t(\omega) = \qnorm{\bar{\g}_t} \qnorm{H_t \omega - \vv_t}$ for some $\bar{\g}_t,\vv_t\in\reals^d$ and $H_t \in \reals^{d\times m}$, then
\begin{align} \label{eq:hinting-gradient}
\gamma_t
    =
    \begin{cases}
        \frac{\norm{\bar\g_{t}}_q}{\norm{H_t \omega-\vv_t}_q^{q-1}} H_t^\top |H_t \omega-\vv_t|^{q-1} \sign(H_t \omega-\vv_t) 
        & \text{if } q < \infty \\
        \norm{\bar\g_{t}}_{\infty} \sign(\mu) H_t^\top \basis_k 
        & \text{if } q = \infty
    \end{cases}
    \quad
    \in 
    \quad
    \subdiff l_t(\omega) 
\end{align}
for $k = \argmax_{j \in [d]} (H_t \omega-\vv_t)_j$ and $\mu = \max_{j \in [d]} (H_t \omega-\vv_t)_j$.
\end{lemma}

Our next lemma, proved in \cref{proof_hint-gradient-bound}, bounds the $\infty$-norm of this hinting loss subgradient in terms of the base problem subgradients.

\begin{lemma}[Hinting loss subgradient bound] \label{hint-gradient-bound}
Under the assumptions and notation of \cref{hinting_loss_subgradient}, the subgradient $\gamma_t$ satisfies 
$\infnorm{\gamma_t} \leq d^{1/q}\norm{\bar\g_{t}}_{q} \maxentrynorm{H_t}$
for $\maxentrynorm{H_t}$ the maximum absolute entry of $H_t$. 
\end{lemma}

\subsection{Proof of \cref{hinting_loss_subgradient}: Hinting loss subgradient}\label{proof_hinting_loss_subgradient}
The result follows immediately from the chain rule and the following lemma.
\begin{lemma}[Subgradients of $p$-norms] \label{p-norm-grad}
Suppose $\ww \in \R^d$ and $k \in \argmax_{j \in [d]} |\ww_j|$.
Then
\begin{talign}
    \subdiff \norm{\ww}_p \ni
    \begin{cases}
        \frac{|\ww|^{p-1}}{\norm{\ww}_p^{p-1}} \sign(\ww) 
        &\text{if $\norm{\ww}_p \neq 0, p\in[1,\infty)$} \\
        \basis_k \sign(\ww_k)  
        &\text{if $\norm{\ww}_{p} \neq 0, p = \infty$} \\
        \boldzero 
        &\text{if $\norm{\ww}_p = 0$} 
    \end{cases}.
\end{talign}
\end{lemma}
\begin{proof}
Since $\boldzero$ is a minimizer of $\pnorm{\cdot}$, we have $\pnorm{\uu} \geq \pnorm{\boldzero} + \inner{\boldzero}{\uu-\boldzero}$ for any $\uu\in\reals^d$ and hence $\boldzero \in \subdiff \pnorm{\boldzero}$.

For $p \in [1, \infty)$, by the chain rule, if $\norm{\ww}_p \neq \boldzero$,
\begin{talign}
\partial_j \norm{\ww}_p 
    &= \partial_j \big( \sum_{k=1}^n |\ww_k|^p \big)^{1/p} = \frac{1}{p} \big( \sum_{k=1}^n |\ww_k|^p \big)^{(1/p) - 1}   p |\ww_j|^{p-1} \sign(\ww_j) \\
    &=   \Big(\big( \sum_{k=1}^n |\ww_k|^p \big)^{1/p}\Big)^{-(p-1)}  |\ww_j|^{p-1} \sign(\ww_j) \\
    &=  \Big(\frac{|\ww_j|}{\norm{\ww}_p}\Big)^{p-1} \sign(\ww_j).
\end{talign}

For $p = \infty$, we have that $\infnorm{\ww} = \max_{j \in [n]} |\ww_j|$. By the Danskin-Bertsekas Theorem \citep{danskin2012theory} for subdifferentials, $\subdiff \infnorm{\ww} = \conv\{\cup \subdiff |\ww_j| \qtext{s.t.} |\ww_j| = \infnorm{\ww} \} = \conv\{\cup \sign(\ww_j) \basis_j \qtext{s.t.} |\ww_j| = \infnorm{\ww} \}$, where $\conv$ is the convex hull operation.
\end{proof}

\subsection{Proof of \cref{hint-gradient-bound}: Hinting loss subgradient bound} \label{proof_hint-gradient-bound}

If $q \in [1, \infty)$, we have
\begin{talign}
\norm{\gamma_t}_{\infty} 
    &= \left\lVert \frac{\norm{\bar\g_{t}}_q}{\norm{H_t \omega-\sum_{s=t-D}^t\g_s}_q^{q-1}} H_t^\top |H_t \omega-\sum_{s=t-D}^t\g_s|^{q-1} \sign(H_t \omega-\sum_{s=t-D}^t\g_s)\right\rVert_{\infty} \\
    &\leq \frac{\norm{\bar\g_{t}}_q \max_{j\in[d]} \qnorm{H_t\basis_j}}{\norm{H_t \omega-\sum_{s=t-D}^t\g_s}_q^{q-1}} \qnorm{H_t \omega-\sum_{s=t-D}^t\g_s}^{q-1} \qtext{by \Holder's inequality for $(q,p)$}\\
    &\leq d^{1/q}\norm{\bar\g_{t}}_q \maxentrynorm{H_t} \qtext{ by \cref{norm-equality}.}
\end{talign}

If $q=\infty$, we have
\begin{talign}
\infnorm{\gamma_t} = \left\lVert \norm{\bar\g_{t}}_{\infty} \sign(\mu) H_t^\top \basis_{k} \right\rVert_{\infty} 
   =\indic{\mu \neq 0}\norm{\bar\g_{t}}_{\infty} \maxentrynorm{H_t}
   \leq d^{1/q}\norm{\bar\g_{t}}_{\infty} \maxentrynorm{H_t}.
\end{talign}

\section{Experiment Details} \label{sec:experiment_details}
\subsection{Subseasonal Forecasting Application}
We apply the online learning techniques developed in this paper to the problem of adaptive ensembling for subseasonal weather forecasting.  Subseasonal forecasting is the problem predicting meteorological variables, often temperature and precipitation, 2-6 weeks in advance. These mid-range forecasts are critical for managing water resources and mitigating  wildfires, droughts, floods, and other extreme weather events \citep{hwang2019improving}. However, the subseasonal forecasting task is notoriously difficult due to the joint influences of short-term initial conditions and long-term boundary conditions \citep{white2017potential}. 

To improve subseasonal weather forecasting capabilities, the US Department of Reclamation launched the Sub-Seasonal Climate Forecast Rodeo competition \cite{nowak2020subseasonal}, a yearlong real-time forecasting competition for the Western United States. Our experiments are based on \citet{frii}, a snapshot of public subseasonal model forecasts including both physics-based and machine learning models. These models were developed for the subseasonal forecasting challenge and make semimonthly forecasts for the contest period (19 October 2019 -- 29 September 2020).  

To expand our evaluation beyond the subseasonal forecasting competition, we used the forecasts in \citet{frii} for analogous yearlong periods (26 semi-monthly dates starting from the last Wednesday in October) beginning in Oct. 2010 and ending in Sep. 2020. 
Throughout, we refer to the yearlong period beginning in Oct. 2010 -- Sep. 2011 as the 2011 year and so on for each subsequent year.
For each forecast date $t$, the models in \citet{frii} were trained only on data available at time $t$ and model hyper-parameters were tuned to optimize average RMSE loss on the 3-year period preceding the forecast date $t$. For a few of the forecast dates, one or more models had missing forecasts; only dates for which all models have forecasts were used in evaluation.

\subsection{Problem Definition}
Denote the set of $d=6$ input models $\{\M_1, \dots \M_d\}$ with labels: \texttt{llr} (Model1), \texttt{multillr} (Model2), \texttt{tuned\_catboost} (Model3), \texttt{tuned\_cfsv2} (Model4), \texttt{tuned\_doy} (Model5) and \texttt{tuned\_salient\_fri} (Model6). On each semimonthly forecast date, each model $\M_i$ makes a prediction for each of two meteorological variables (cumulative precipitation and average temperature over 14 days) and two forecasting horizons (3-4 weeks and 5-6 weeks).  For the 3-4 week and 5-6 horizons respectively, the forecaster experiences a delay of $D=2$ and $D=3$ forecasts. Each model makes a total of $T=26$ semimonthly forecasts for these four tasks. 

At each time $t$, each input model $\M_i$ produces a prediction at $G=514$ gridpoints in the Western United States: $\x^c_{t, i} \in \mathbb{R}^{G} = \M_i(t)$ for task $c$ at time $t$. Let $\X^c_{t} \in \R^{G \times d}$ be the matrix containing each input model's predictions as columns. The true meterological outcome for task $c$ is $\y_t^c \in \mathbb{R}^G$.  As online learning is performed for each task separately, we drop the task superscript $c$ in the following. 

At each timestep, the online learner makes a forecast prediction $\hat{\y}_t$ by playing $\ww_t \in \wset = \simplex_{d-1}$, corresponding to a convex combination of the individual models: $\hat{\y}_t = \X_{t} \ww_t$. The learner then incurs a loss for the play $\ww_t$ according to the root mean squared (RMSE) error over the geography of interest:
\begin{align} \label{eq:rodeo_loss}
    \loss_t(\ww_t) &= \frac{1}{\sqrt{G}} \normt{\y_t - \X_t \ww_t}, \\
    \subdiff \loss_t(\ww_t) &\ni \g_t =
    \begin{cases}
    \frac{\X_t^\top(\X_t \ww_t - \y_t)}{\sqrt{G}\normt{\X_t \ww_t - \y_t}} & \qtext{if} \X_t \ww_t - \y_t \neq \boldzero \\
    \boldzero & \qtext{if} \X_t \ww_t - \y_t = \boldzero
    \end{cases}
\end{align}    

Our objective for the subseasonal forecasting application is to produce an adaptive ensemble forecast that competes with the best input model over the yearlong period. Hence, in our evaluation, we take the competitor set to be the set of individual models $\uset = \{ \basis_i : i \in [d] \}$.

\section{Extended Experimental Results} \label{sec:experiment_extended_results}
We present complete experimental results for the four experiments presented in the main paper (see \cref{sec:experiments}).

\subsection{Competing with the Best Input Model}
Results for our three delayed online learning algorithms --- \DORM, \DORMP, and \AdaHedgeD --- on the four subseasonal prediction tasks for the four optimism strategies described in \cref{sec:experiments} (\texttt{recent\_g}, \texttt{prev\_g}, \texttt{mean\_g}, \texttt{none}) are presented below. Each table and figure shows the average RMSE loss and the annual regret versus the best input model in any given year respectively for each algorithm and task. 

\DORMP is a competitive model for all three hinting strategies and under the \texttt{recent\_g} hinting strategy achieves negative regret on all tasks except Temp. 5-6w. For the Temp. 5-6w task, no online learning model outperforms the best input model for any hinting strategy. For the precipitation tasks, the online learning algorithms presented achieve negative regret using all three hinting strategies for all four tasks. Within the subseasonal forecasting domain, precipitation is often considered a more challenging forecasting task than temperature \citep{white2017potential}. The gap between the best model and the worst model tends to be larger for precipitation than for temperature, and this could in part explain the strength of the online learning algorithms for these tasks.

\begin{table}[H]
\caption{\textbf{Hint \texttt{recent\_g}:} Average RMSE of the 2011-2020 semimonthly forecasts for online learning algorithms (left) and input models (right) over a $10$-year evaluation period with the top-performing learners and input models bolded and blue. %
In each task, the online learners compare favorably with the best input model and learn to downweight the lower-performing candidates, like the worst models italicized in red.}
\vskip 0.15in
\centering
\begin{small}
\begin{sc}
\begin{tabular}{lrrr|rrrrrr}
\toprule
\multicolumn{1}{c}{\textbf{recent\_g}} &  \multicolumn{1}{c}{AdaHedgeD} &   \multicolumn{1}{c}{DORM} &  \multicolumn{1}{c}{DORM+} &   \multicolumn{1}{c}{Model1}  &  \multicolumn{1}{c}{Model2} &  \multicolumn{1}{c}{Model3} &  \multicolumn{1}{c}{Model4}  &  \multicolumn{1}{c}{Model5}  &  \multicolumn{1}{c}{Model6}  \\
\midrule
Precip. 3-4w &     21.726 & 21.731 & \textcolor{RoyalBlue}{\textbf{21.675}} &  \textcolor{RoyalBlue}{\textbf{21.973}} &  22.431 &  22.357 &  21.978 &  21.986 &  \textcolor{Maroon}{\textit{23.344}} \\
Precip. 5-6w &     21.868 & 21.957 & \textcolor{RoyalBlue}{\textbf{21.838}} &  22.030 &  22.570 &  22.383 &  22.004 &  \textcolor{RoyalBlue}{\textbf{21.993}} &  \textcolor{Maroon}{\textit{23.257}} \\
Temp. 3-4w   &      2.273 &  2.259 &  \textcolor{RoyalBlue}{\textbf{2.247}} &   \textcolor{RoyalBlue}{\textbf{2.253}} &   2.352 &   2.394 &   2.277 &   2.319 &   \textcolor{Maroon}{\textit{2.508}} \\
Temp. 5-6w   &      2.316 &  2.316 &  \textcolor{RoyalBlue}{\textbf{2.303}} &   \textcolor{RoyalBlue}{\textbf{2.270}} &   2.368 &   2.459 &   2.278 &   2.317 &   \textcolor{Maroon}{\textit{2.569}} \\
\bottomrule
\end{tabular}
\end{sc}
\end{small}
\vskip -0.2in
\end{table}

\begin{figure}[H] 
  \begin{minipage}[b]{0.5\linewidth}
    \includegraphics[width=\linewidth]{figures/regret_zoo_contest_precip_34w.pdf} 
    \caption*{Precipitation Weeks 3-4} 
  \end{minipage} 
  \begin{minipage}[b]{0.5\linewidth}
    \includegraphics[width=\linewidth]{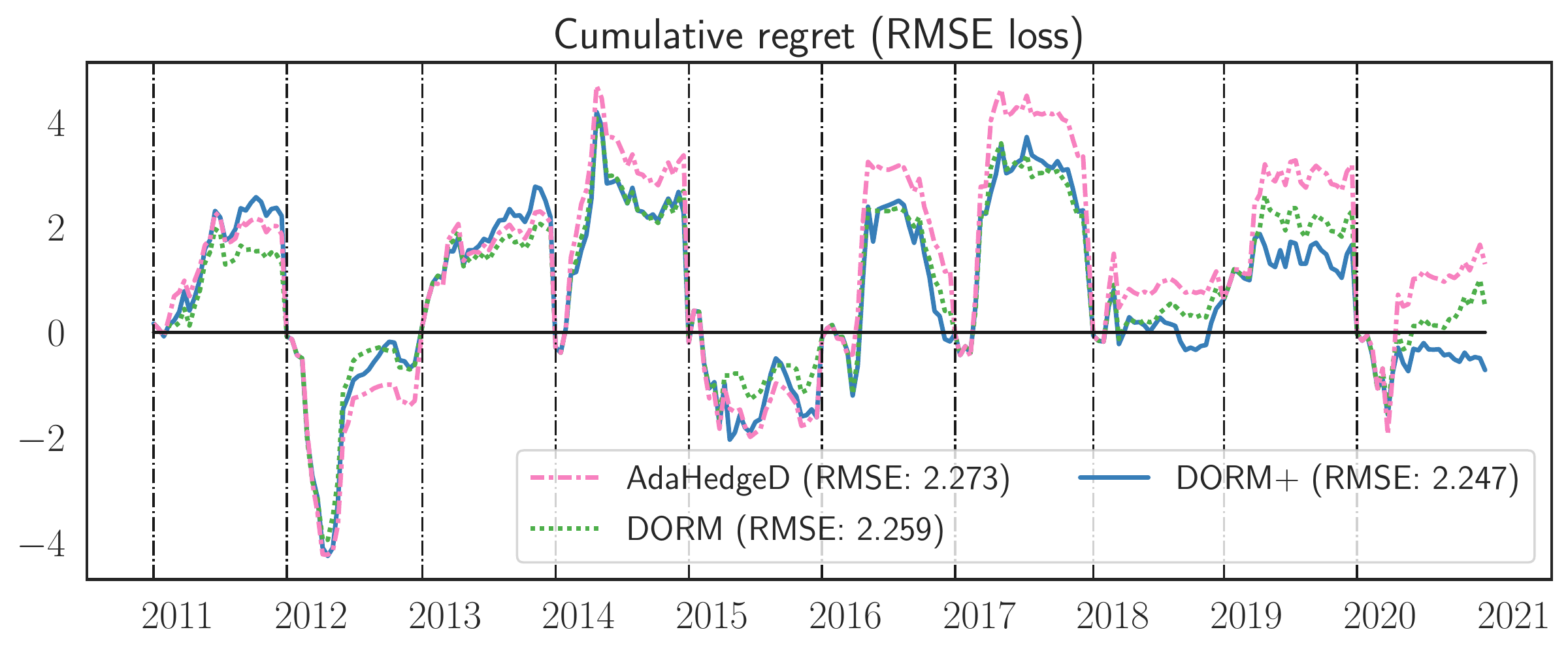} 
    \caption*{Temperature Weeks 3-4} 
  \end{minipage}  
  \begin{minipage}[b]{0.5\linewidth}
    \includegraphics[width=\linewidth]{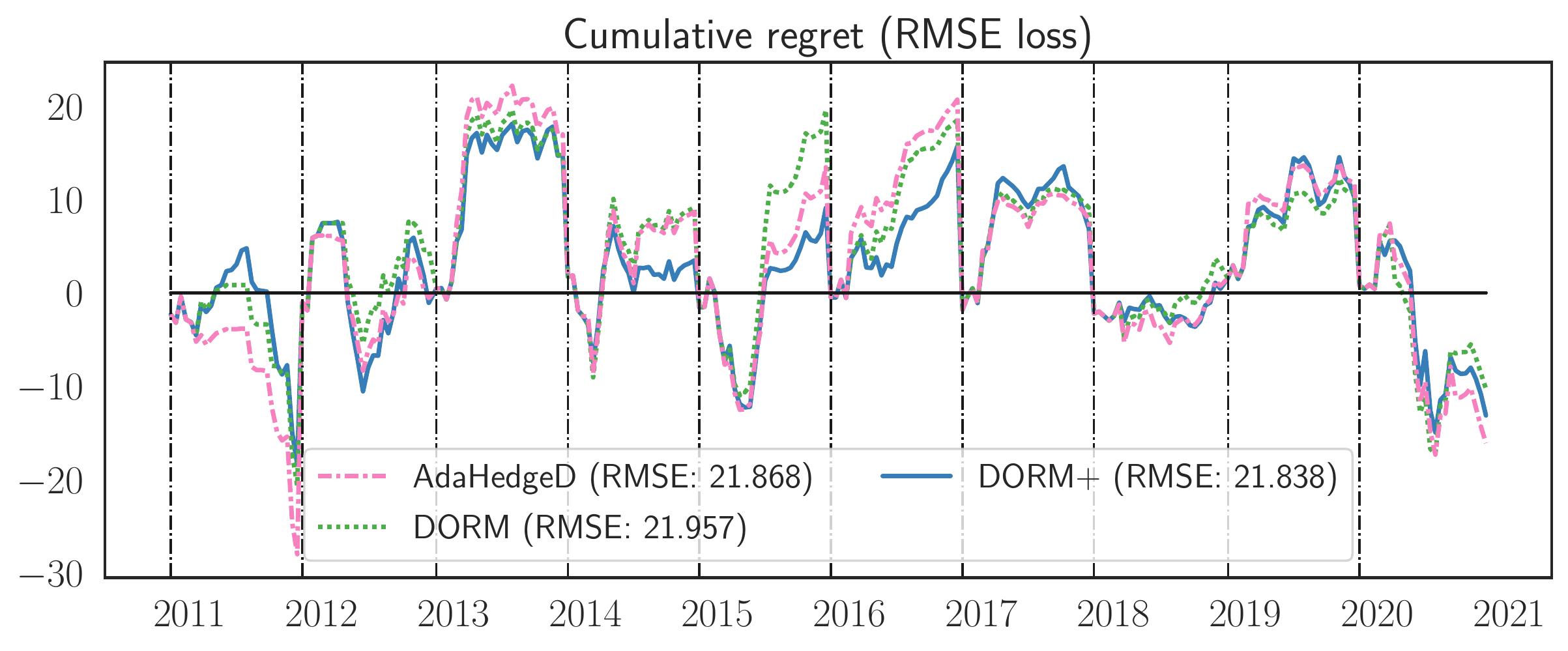} 
    \caption*{Precipitation Weeks 5-6} 
  \end{minipage} 
  \begin{minipage}[b]{0.5\linewidth}
    \includegraphics[width=\linewidth]{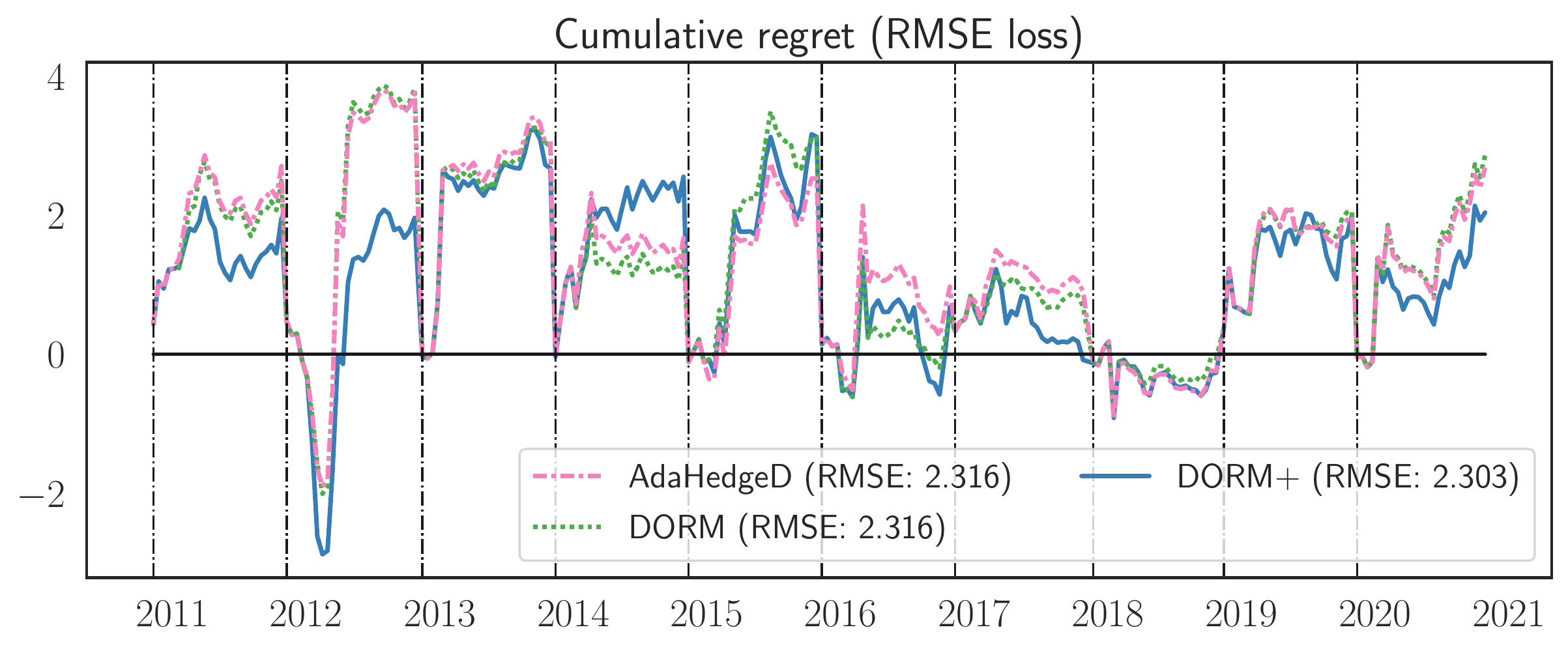} 
    \caption*{Temperature Weeks 5-6} 
  \end{minipage} 
  \caption{\textbf{Hint \texttt{recent\_g}:} Yearly cumulative regret under RMSE loss for the three delayed online learning algorithms presented, over the $10$-year evaluation period. The zero line corresponds to the performance of the best input model in a given year.}
  \label{extended_hint_prev}
\end{figure}

\begin{table}[H]
\caption{\textbf{Hint \texttt{prev\_g}:} Average RMSE of the 2010-2020 semimonthly forecasts for all four tasks over over a $10$-year evaluation period.}
\vskip 0.15in
\centering
\begin{small}
\begin{sc}
\begin{tabular}{lrrr|rrrrrr}
\toprule
\multicolumn{1}{c}{\textbf{\texttt{prev\_g}}} &  \multicolumn{1}{c}{AdaHedgeD} &   \multicolumn{1}{c}{DORM} &  \multicolumn{1}{c}{DORM+} &   \multicolumn{1}{c}{Model1}  &  \multicolumn{1}{c}{Model2} &  \multicolumn{1}{c}{Model3} &  \multicolumn{1}{c}{Model4}  &  \multicolumn{1}{c}{Model5}  &  \multicolumn{1}{c}{Model6}  \\
\midrule
Precip. 3-4w &     21.760 & 21.777 & \textcolor{RoyalBlue}{\textbf{21.729}} &  \textcolor{RoyalBlue}{\textbf{21.973}} &  22.431 &  22.357 &  21.978 &  21.986 &  \textcolor{Maroon}{\textit{23.344}} \\
Precip. 5-6w &     21.943 & 21.964 & \textcolor{RoyalBlue}{\textbf{21.911}} &  22.030 &  22.570 &  22.383 &  22.004 &  \textcolor{RoyalBlue}{\textbf{21.993}} &  \textcolor{Maroon}{\textit{23.257}} \\
Temp. 3-4w   &      2.266 &  2.269 & \textcolor{RoyalBlue}{\textbf{2.250}} &   \textcolor{RoyalBlue}{\textbf{2.253}} &   2.352 &   2.394 &   2.277 &   2.319 &   \textcolor{Maroon}{\textit{2.508}} \\
Temp. 5-6w   &      2.306 &  2.307 &  \textcolor{RoyalBlue}{\textbf{2.305}} &  \textcolor{RoyalBlue}{\textbf{2.270}} &   2.368 &   2.459 &   2.278 &   2.317 &   \textcolor{Maroon}{\textit{2.569}} \\
\bottomrule
\end{tabular}
\end{sc}
\end{small}
\vskip -0.2in
\end{table}

\begin{figure}[H]
  \begin{minipage}[b]{0.5\linewidth}
    \includegraphics[width=\linewidth]{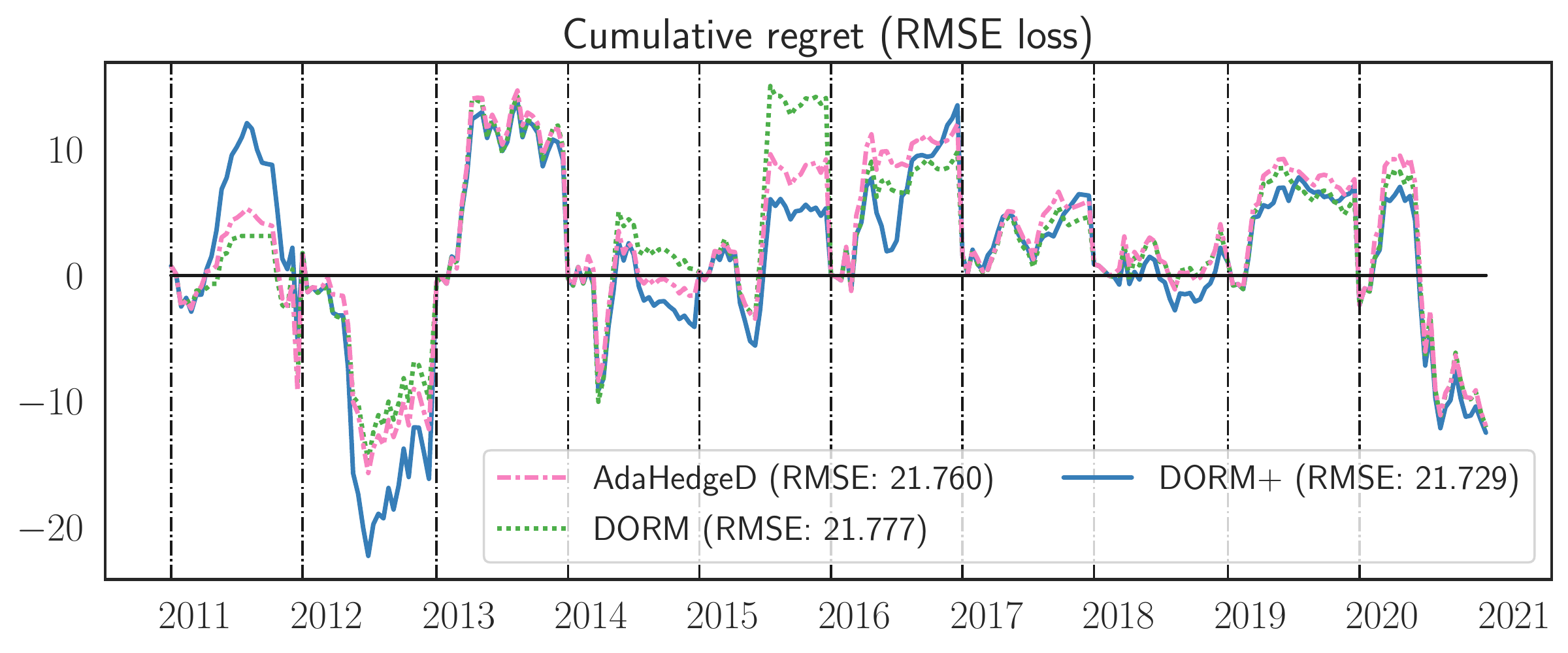} 
    \caption*{Precipitation Weeks 3-4} 
  \end{minipage} 
  \begin{minipage}[b]{0.5\linewidth}
    \includegraphics[width=\linewidth]{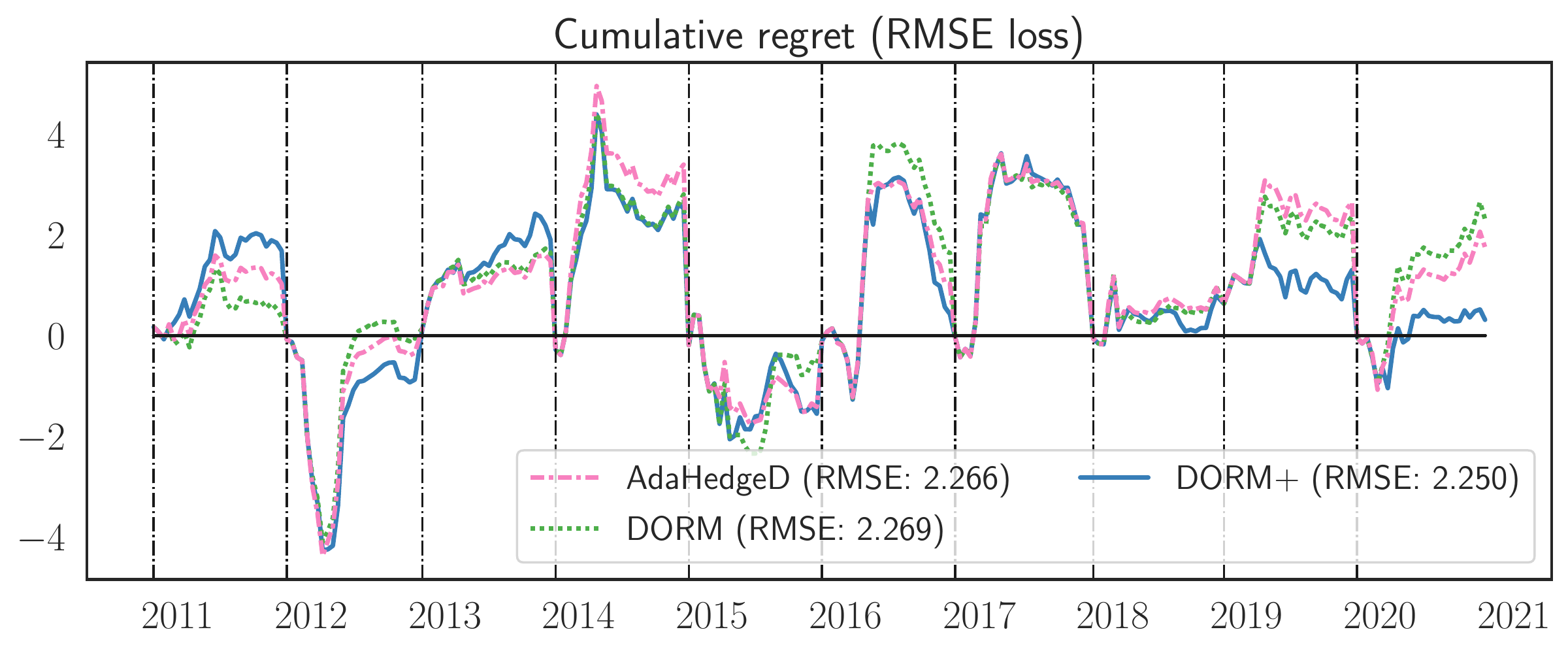} 
    \caption*{Temperature Weeks 3-4} 
  \end{minipage}  
  \begin{minipage}[b]{0.5\linewidth}
    \includegraphics[width=\linewidth]{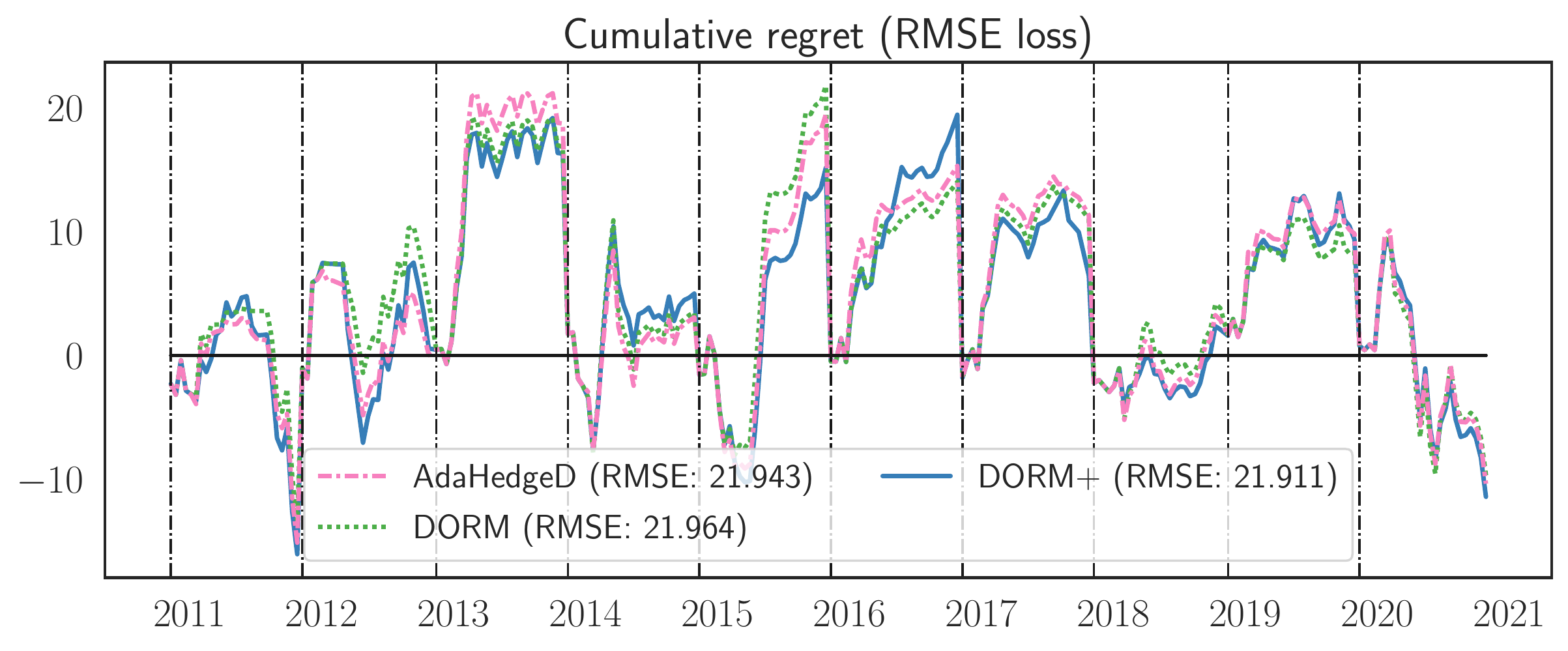} 
    \caption*{Precipitation Weeks 5-6} 
  \end{minipage} 
  \begin{minipage}[b]{0.5\linewidth}
    \includegraphics[width=\linewidth]{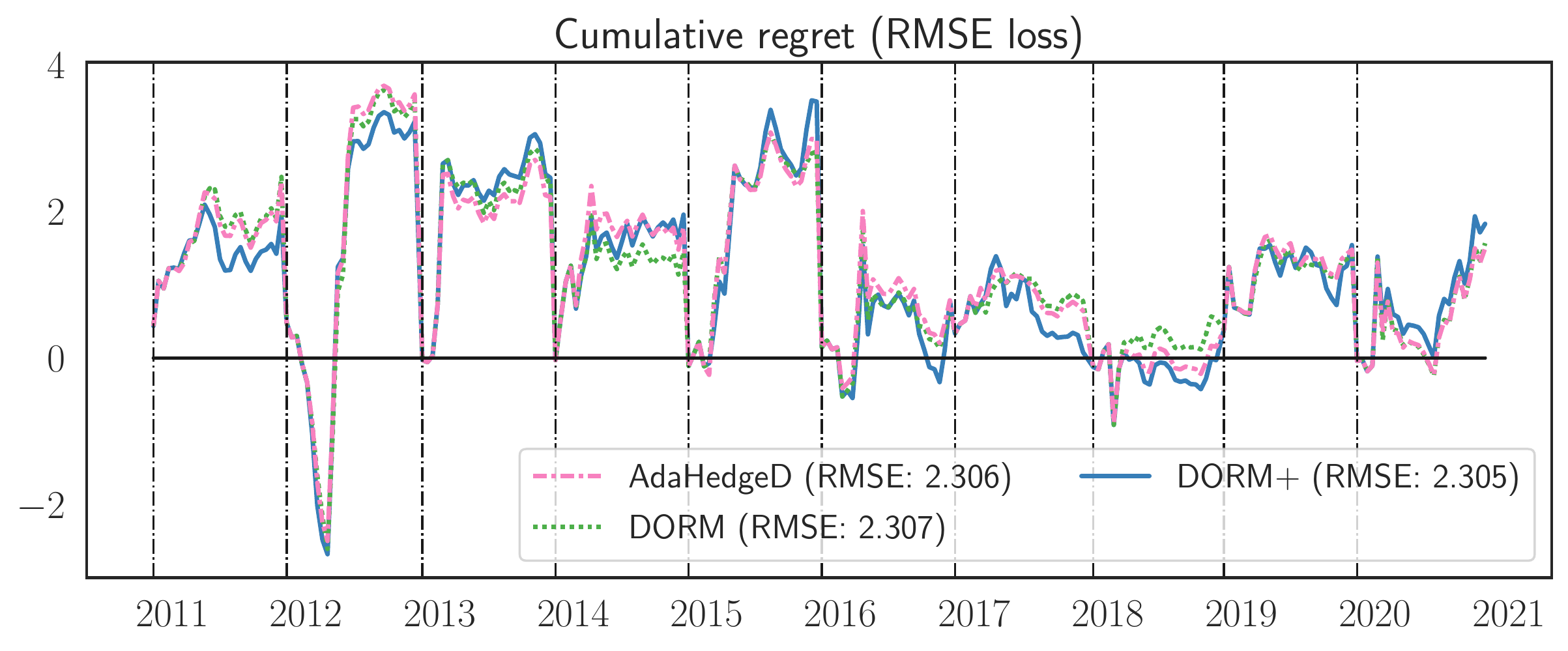} 
    \caption*{Temperature Weeks 5-6} 
  \end{minipage} 
  \caption{\textbf{Hint \texttt{prev\_g}:} Yearly cumulative regret under RMSE loss for the three delayed online learning algorithms presented.}
  \label{extended_hint_avg_prev_g}
\end{figure}

\begin{table}[H]
\caption{\textbf{Hint \texttt{mean\_g}:} Average RMSE of the 2010-2020 semimonthly forecasts for all four tasks over over a $10$-year evaluation period.}
\vskip 0.15in
\centering
\begin{small}
\begin{sc}
\begin{tabular}{lrrr|rrrrrr}
\toprule
\multicolumn{1}{c}{\textbf{mean\_g}} &  \multicolumn{1}{c}{AdaHedgeD} &   \multicolumn{1}{c}{DORM} &  \multicolumn{1}{c}{DORM+} &   \multicolumn{1}{c}{Model1}  &  \multicolumn{1}{c}{Model2} &  \multicolumn{1}{c}{Model3} &  \multicolumn{1}{c}{Model4}  &  \multicolumn{1}{c}{Model5}  &  \multicolumn{1}{c}{Model6}  \\
\midrule
Precip. 3-4w &     21.864 & 21.945 & \textcolor{RoyalBlue}{\textbf{21.830}} &  \textcolor{RoyalBlue}{\textbf{21.973}} &  22.431 &  22.357 &  21.978 &  21.986 &  \textcolor{Maroon}{\textit{23.344}} \\
Precip. 5-6w &     21.993 & 22.054 & \textcolor{RoyalBlue}{\textbf{21.946}} &  22.030 &  22.570 &  22.383 &  22.004 &  \textcolor{RoyalBlue}{\textbf{21.993}} &  \textcolor{Maroon}{\textit{23.257}} \\
Temp. 3-4w   &      2.273 &  2.277 &  \textcolor{RoyalBlue}{\textbf{2.257}} &   \textcolor{RoyalBlue}{\textbf{2.253}} &   2.352 &   2.394 &   2.277 &   2.319 &   \textcolor{Maroon}{\textit{2.508}} \\
Temp. 5-6w   &      \textcolor{RoyalBlue}{\textbf{2.311}} &  2.320 &  2.314 &   \textcolor{RoyalBlue}{\textbf{2.270}} &   2.368 &   2.459 &   2.278 &   2.317 &   \textcolor{Maroon}{\textit{2.569}} \\
\bottomrule
\end{tabular}
\end{sc}
\end{small}
\vskip -0.2in
\end{table}

\begin{figure}[H]
  \begin{minipage}[b]{0.5\linewidth}
    \includegraphics[width=\linewidth]{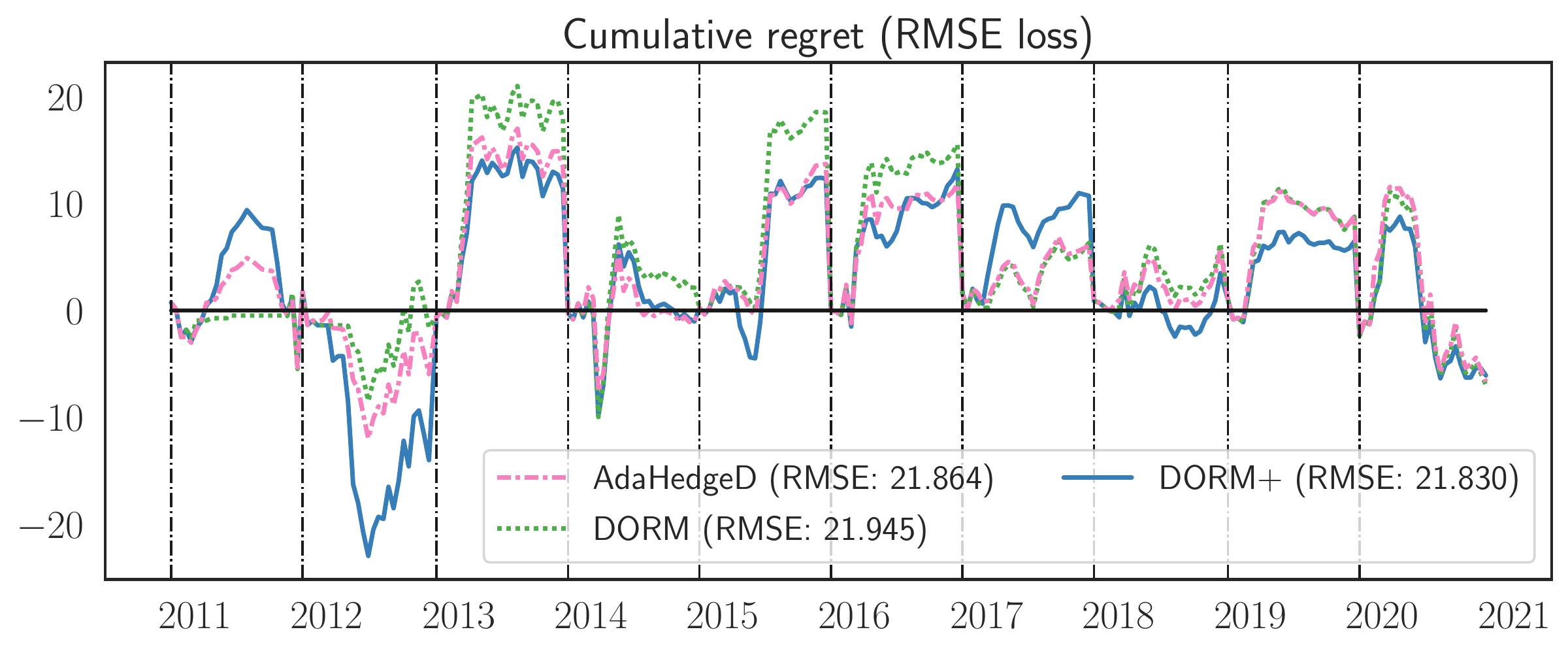} 
    \caption*{Precipitation Weeks 3-4} 
  \end{minipage} 
  \begin{minipage}[b]{0.5\linewidth}
    \includegraphics[width=\linewidth]{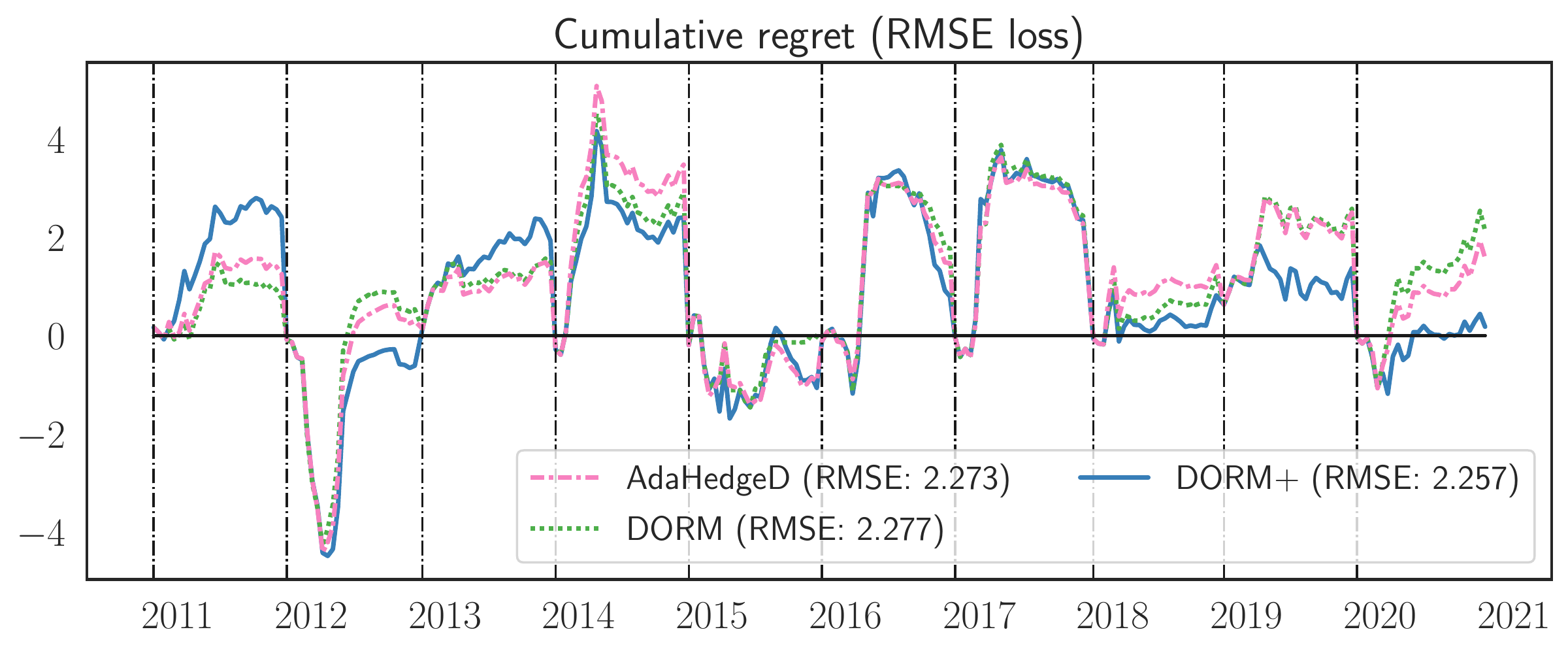} 
    \caption*{Temperature Weeks 3-4} 
  \end{minipage}  
  \begin{minipage}[b]{0.5\linewidth}
    \includegraphics[width=\linewidth]{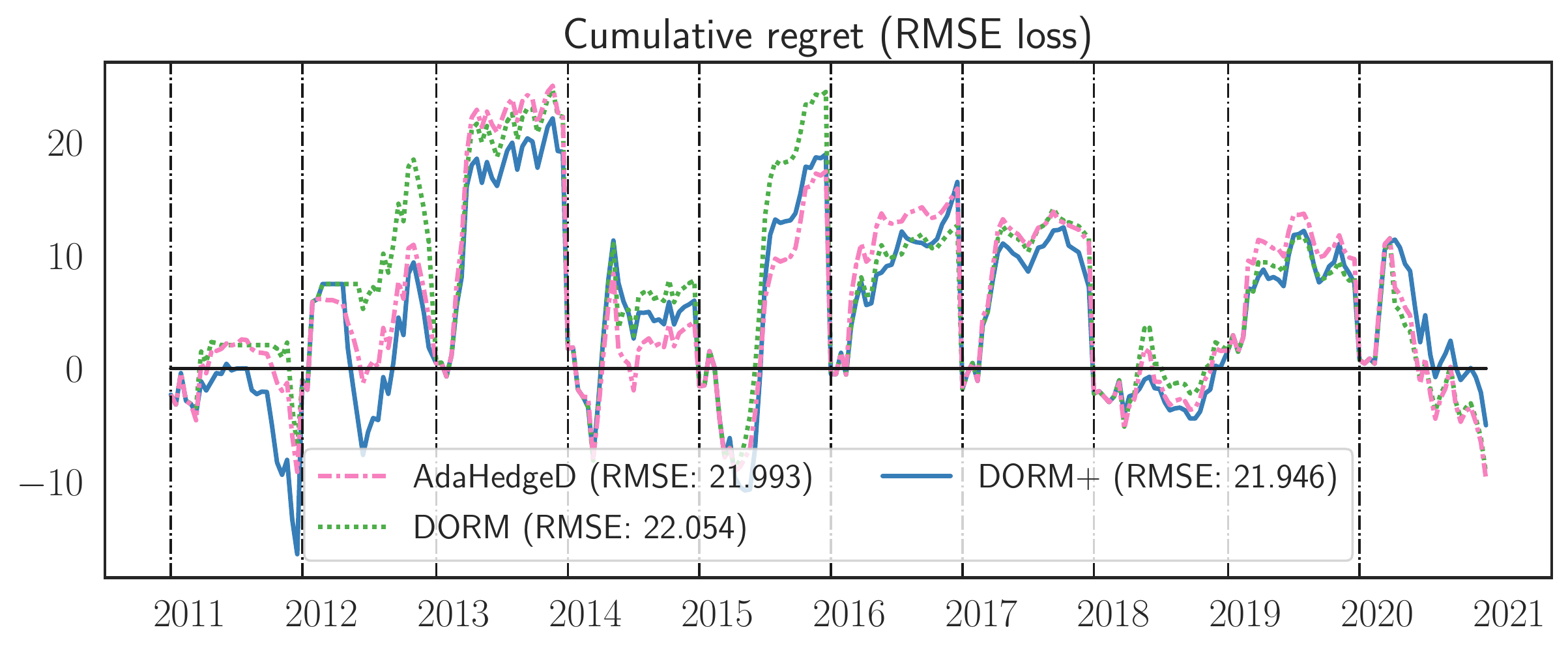} 
    \caption*{Precipitation Weeks 5-6} 
  \end{minipage} 
  \begin{minipage}[b]{0.5\linewidth}
    \includegraphics[width=\linewidth]{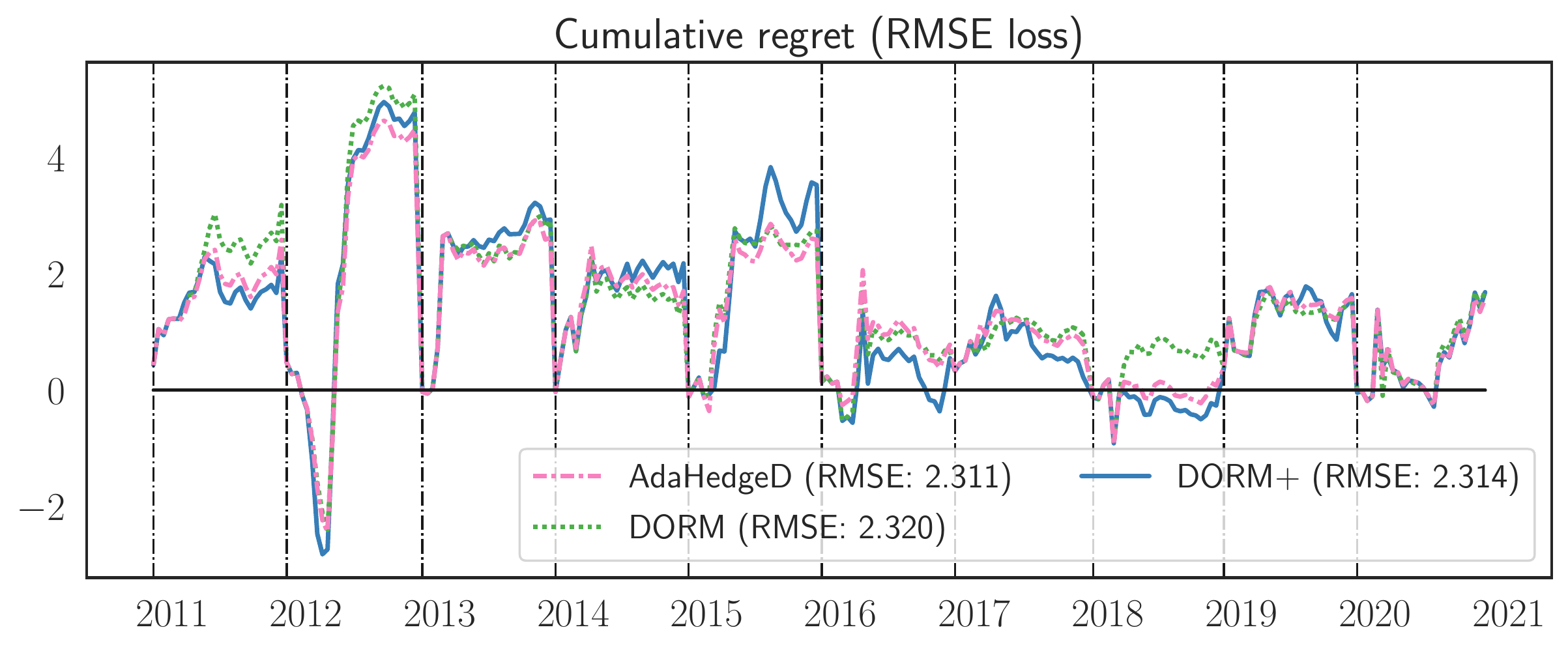} 
    \caption*{Temperature Weeks 5-6} 
  \end{minipage} 
  \caption{\textbf{Hint \texttt{mean\_g}:} Yearly cumulative regret under RMSE loss for the three delayed online learning algorithms presented.}
  \label{extended_hint_mean}
\end{figure}

\begin{table}[H]
\caption{\textbf{Hint \texttt{none}:} Average RMSE of the 2010-2020 semimonthly forecasts for all four tasks over over a $10$-year evaluation period.}
\vskip 0.15in
\centering
\begin{small}
\begin{sc}
\begin{tabular}{lrrr|rrrrrr}
\toprule
\multicolumn{1}{c}{\textbf{None}} &  \multicolumn{1}{c}{AdaHedgeD} &   \multicolumn{1}{c}{DORM} &  \multicolumn{1}{c}{DORM+} &   \multicolumn{1}{c}{Model1}  &  \multicolumn{1}{c}{Model2} &  \multicolumn{1}{c}{Model3} &  \multicolumn{1}{c}{Model4}  &  \multicolumn{1}{c}{Model5}  &  \multicolumn{1}{c}{Model6}  \\
\midrule
Precip. 3-4w &     \textcolor{RoyalBlue}{\textbf{21.760}} & 21.835 & 21.796 &  \textcolor{RoyalBlue}{\textbf{21.973}} &  22.431 &  22.357 &  21.978 &  21.986 &  \textcolor{Maroon}{\textit{23.344}} \\
Precip. 5-6w &     \textcolor{RoyalBlue}{\textbf{21.860}} & 21.967 & 21.916 &  22.030 &  22.570 &  22.383 &  22.004 &  \textcolor{RoyalBlue}{\textbf{21.993}} &  \textcolor{Maroon}{\textit{23.257}} \\
Temp. 3-4w   &      2.266 &  2.272 &  \textcolor{RoyalBlue}{\textbf{2.258}} &   \textcolor{RoyalBlue}{\textbf{2.253}} &   2.352 &   2.394 &   2.277 &   2.319 &   \textcolor{Maroon}{\textit{2.508}} \\
Temp. 5-6w   &      \textcolor{RoyalBlue}{\textbf{2.296}} &  2.311 &  2.308 &   \textcolor{RoyalBlue}{\textbf{2.270}} &   2.368 &   2.459 &   2.278 &   2.317 &   \textcolor{Maroon}{\textit{2.569}} \\
\bottomrule
\end{tabular}
\end{sc}
\end{small}
\vskip -0.2in
\end{table}

\begin{figure}[H]
  \begin{minipage}[b]{0.5\linewidth}
    \includegraphics[width=\linewidth]{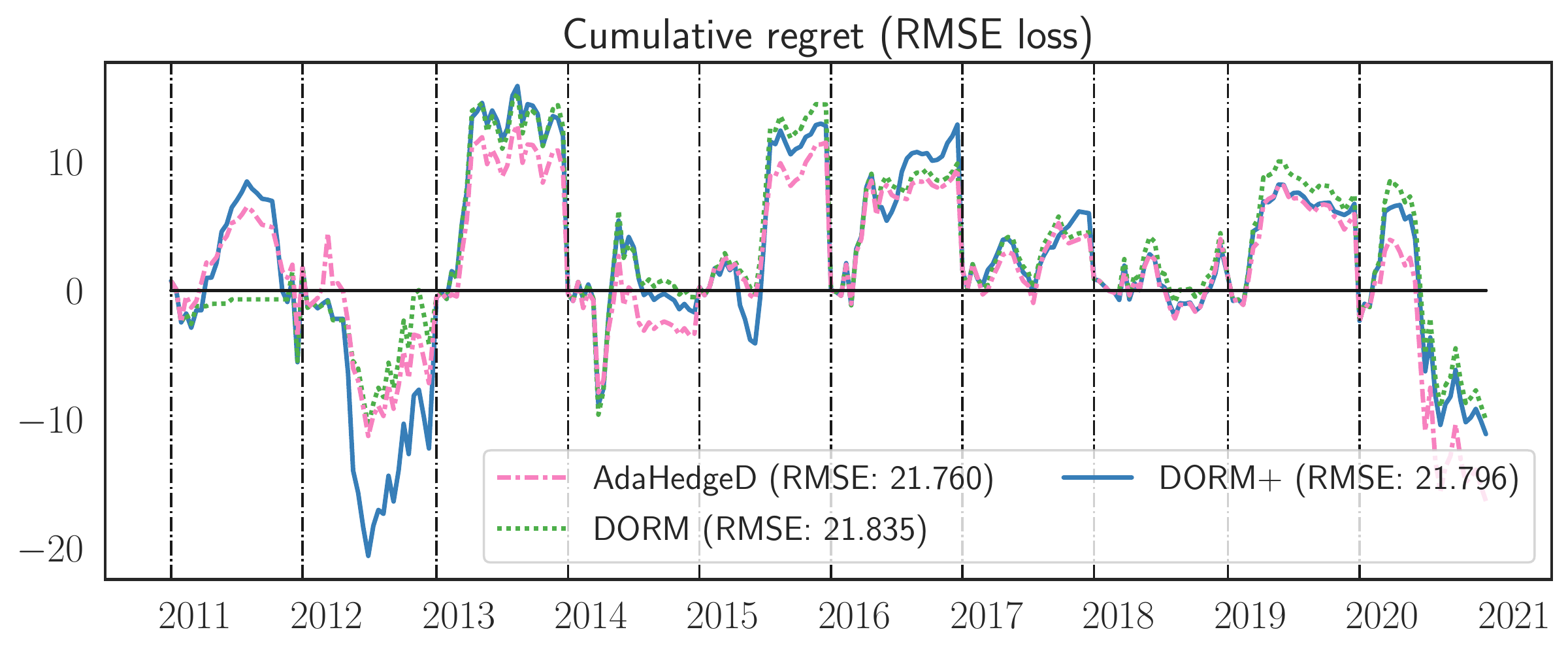} 
    \caption*{Precipitation Weeks 3-4} 
  \end{minipage} 
  \begin{minipage}[b]{0.5\linewidth}
    \includegraphics[width=\linewidth]{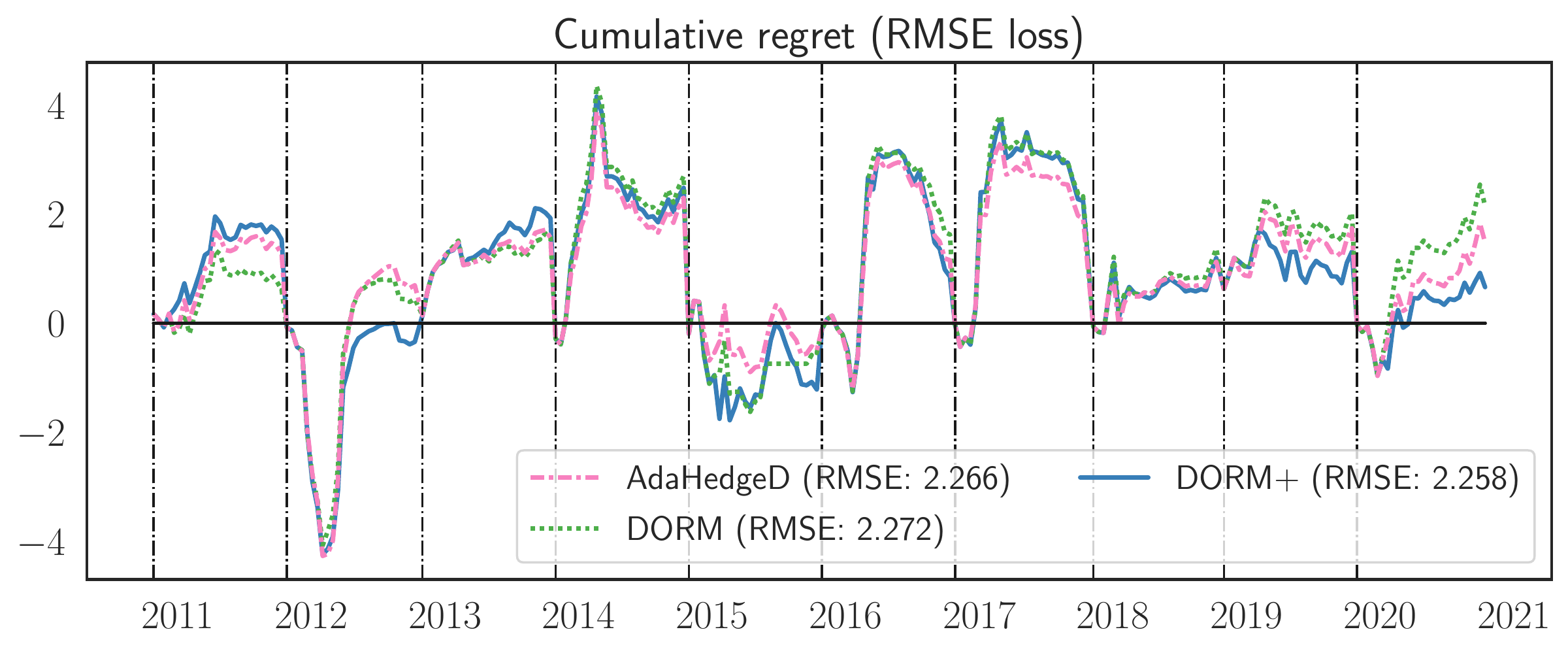} 
    \caption*{Temperature Weeks 3-4} 
  \end{minipage}  
  \begin{minipage}[b]{0.5\linewidth}
    \includegraphics[width=\linewidth]{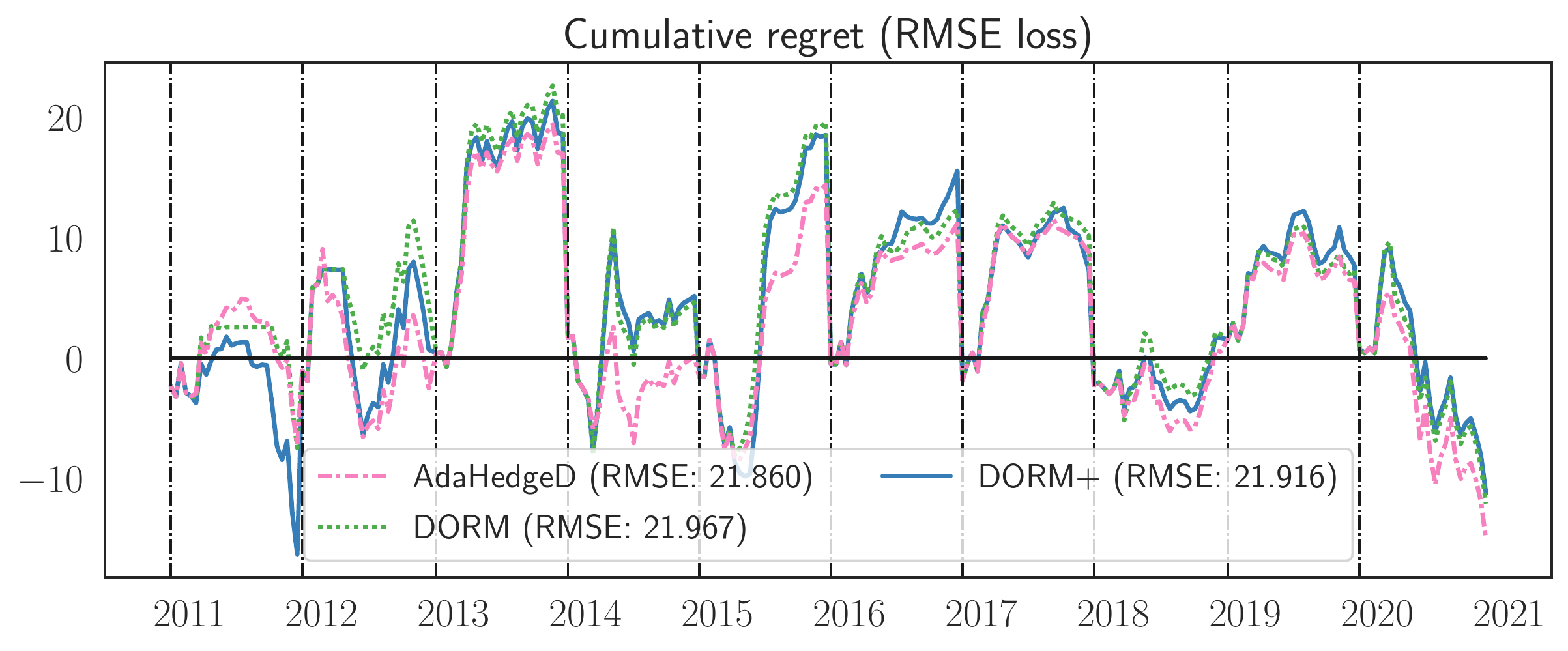} 
    \caption*{Precipitation Weeks 5-6} 
  \end{minipage} 
  \begin{minipage}[b]{0.5\linewidth}
    \includegraphics[width=\linewidth]{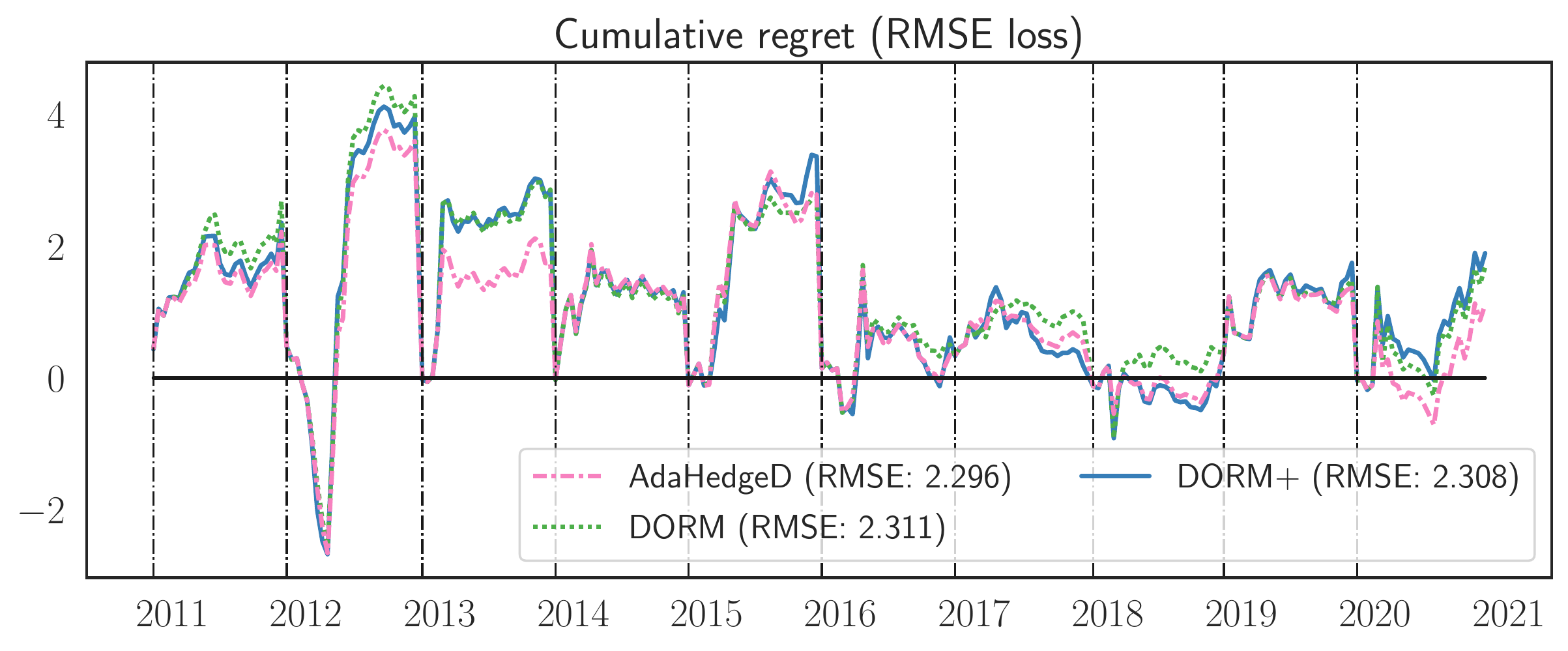} 
    \caption*{Temperature Weeks 5-6} 
  \end{minipage} 
  \caption{\textbf{Hint \texttt{none}:} Yearly cumulative regret under RMSE loss for the three delayed online learning algorithms presented.}
  \label{extended_hint_none}
\end{figure}

\newpage
\subsection{Impact of Regularization}
Results for three regularization strategies---\AdaHedgeD, \DORMP, and \DUB---on all four subseasonal prediction as described in \cref{sec:experiments}. \cref{extended_regularization} shows the annual regret versus the best input model in any given year for each algorithm and task, and \cref{extended_reg_weights} presents an example of the weights played by each algorithm in the final evaluation year, as well as the regularization weight used by each algorithm. 

The under- and over-regularization of \AdaHedgeD and \DUB respectively compared with \DORMP is evident in all four tasks, both in the regret and weight plots. Due to the looseness of the regularization settings used in \DUB, its plays can be seen to be very close to the uniform ensemble in all four tasks. For this subseasonal prediction problem, the uniform ensemble is competitive, especially for the 5-6 week horizons. However, in problems where the uniform ensemble has higher regret, this over-regularization property of \DUB would be undesirable. The more adaptive plays of \DORMP and \AdaHedgeD have the potential to better exploit heterogeneous performance among different input models. 
\begin{figure}[H] 
  \begin{minipage}[b]{0.5\linewidth}
    \includegraphics[width=\linewidth]{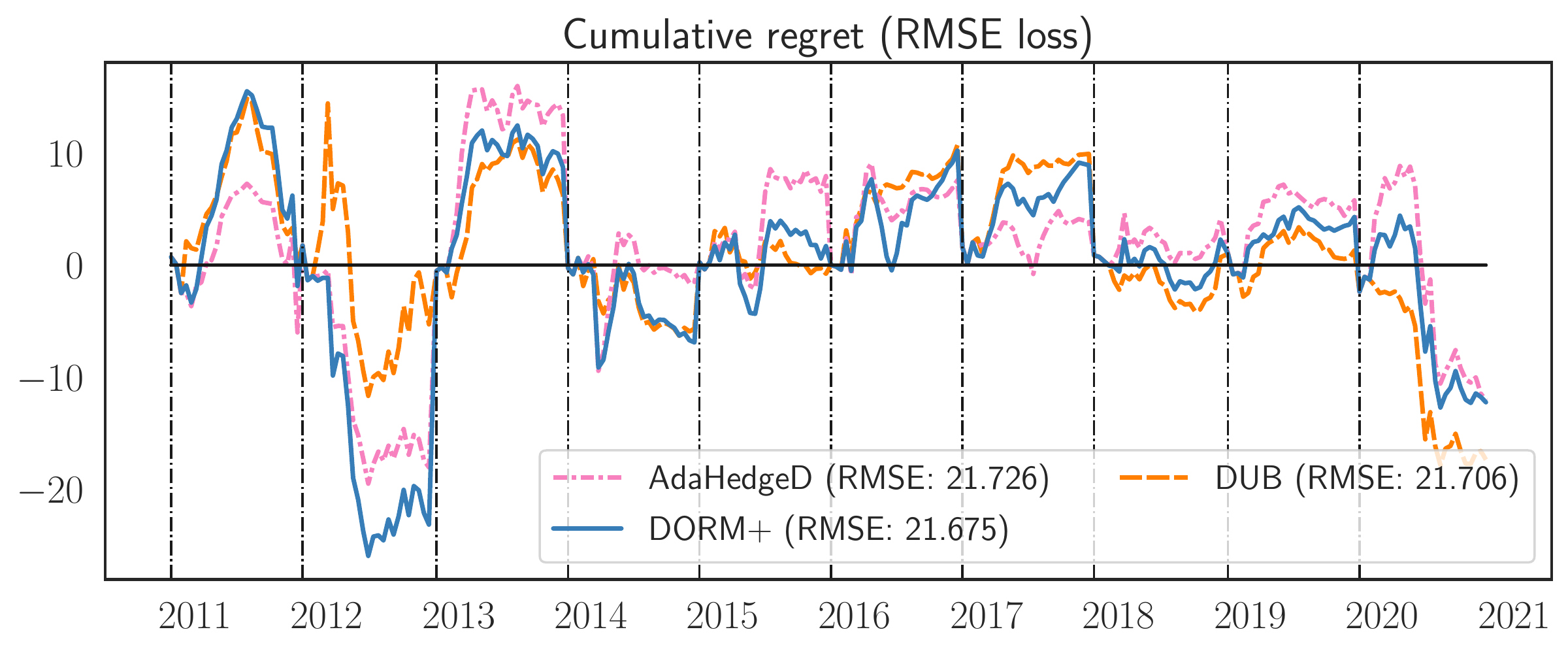} 
    \caption*{Precipitation Weeks 3-4} 
  \end{minipage} 
  \begin{minipage}[b]{0.5\linewidth}
    \includegraphics[width=\linewidth]{figures/regret_regularization_contest_tmp2m_34w.pdf} 
    \caption*{Temperature Weeks 3-4} 
  \end{minipage}  
  \begin{minipage}[b]{0.5\linewidth}
    \includegraphics[width=\linewidth]{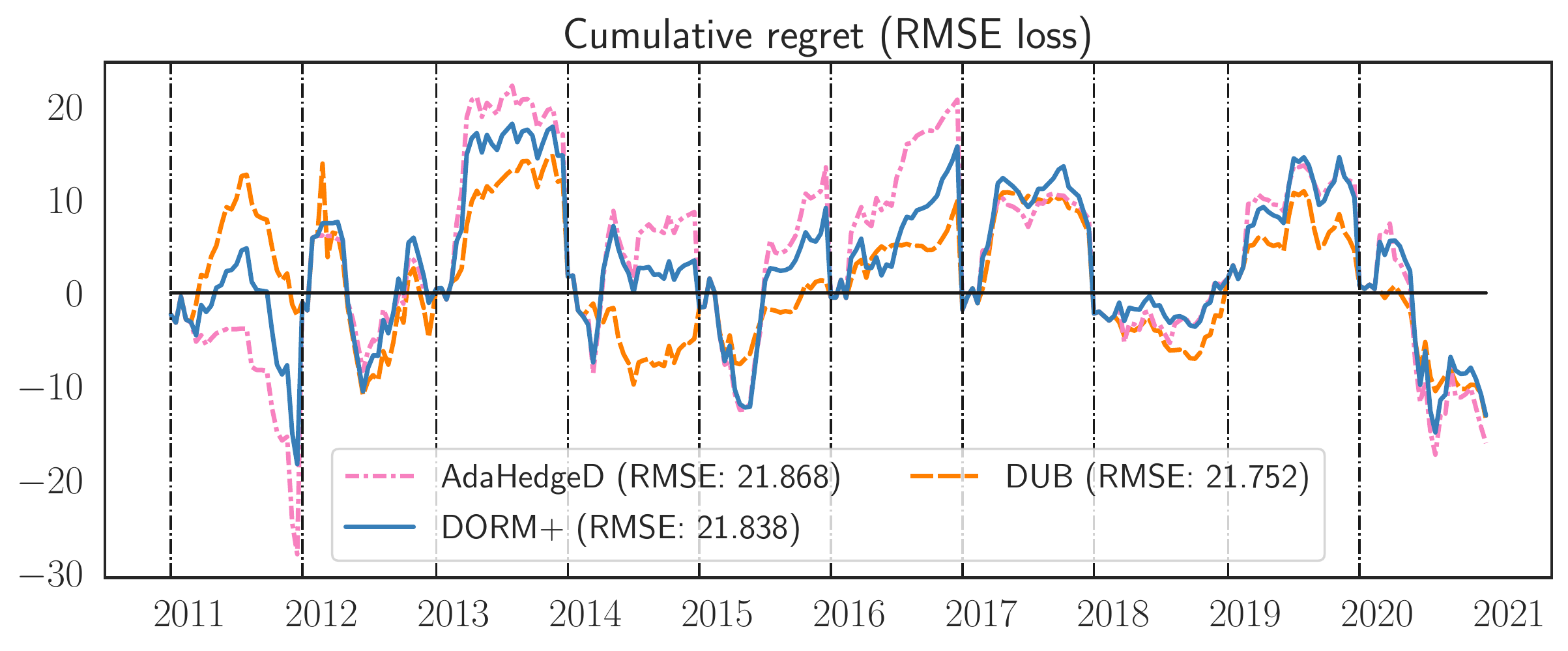} 
    \caption*{Precipitation Weeks 5-6} 
  \end{minipage} 
  \begin{minipage}[b]{0.5\linewidth}
    \includegraphics[width=\linewidth]{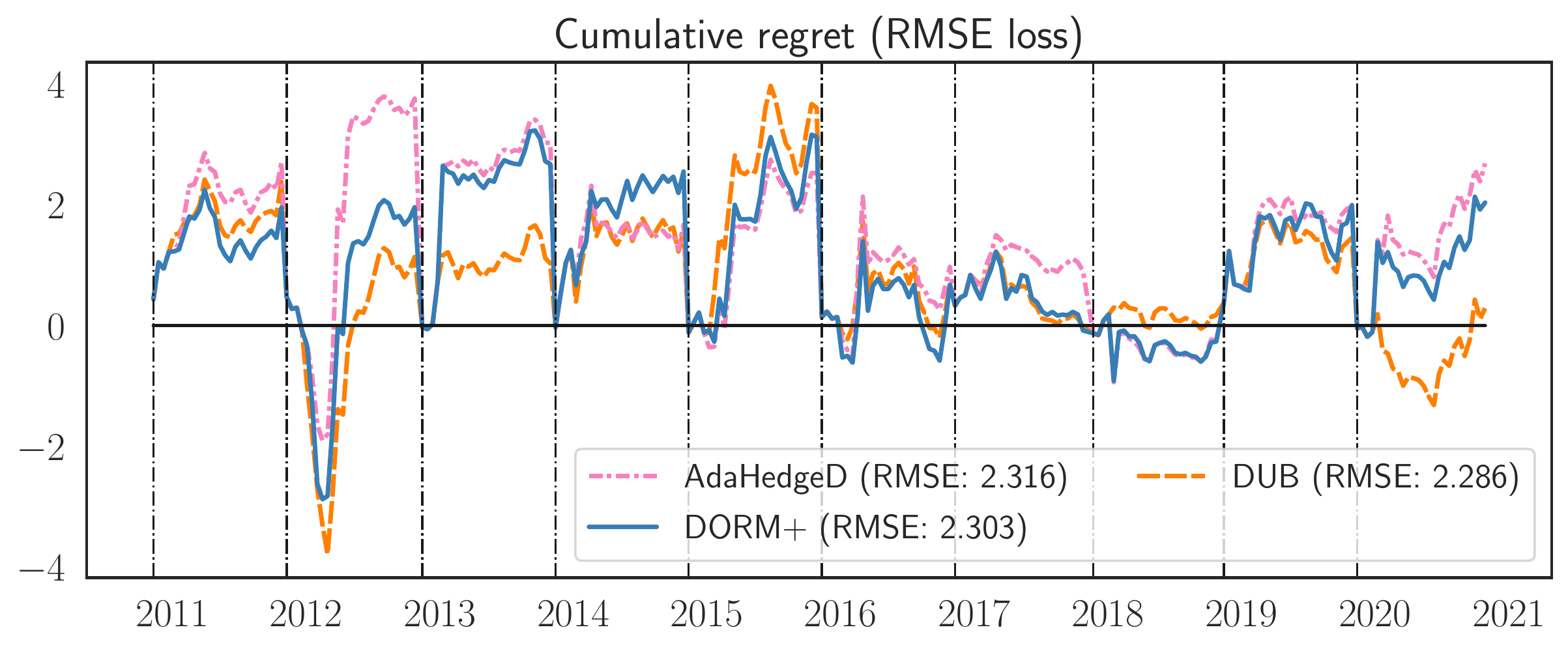} 
    \caption*{Temperature Weeks 5-6} 
  \end{minipage} 
  \caption{\textbf{Overall regret:} Yearly cumulative regret under the RMSE loss for the three regularization algorithms presented.}
 \label{extended_regularization}   
\end{figure}

\newpage
\begin{figure}[H]
  \centering
  \subfigure[Precipitation Weeks 3-4]{
    \includegraphics[height=1.5in]{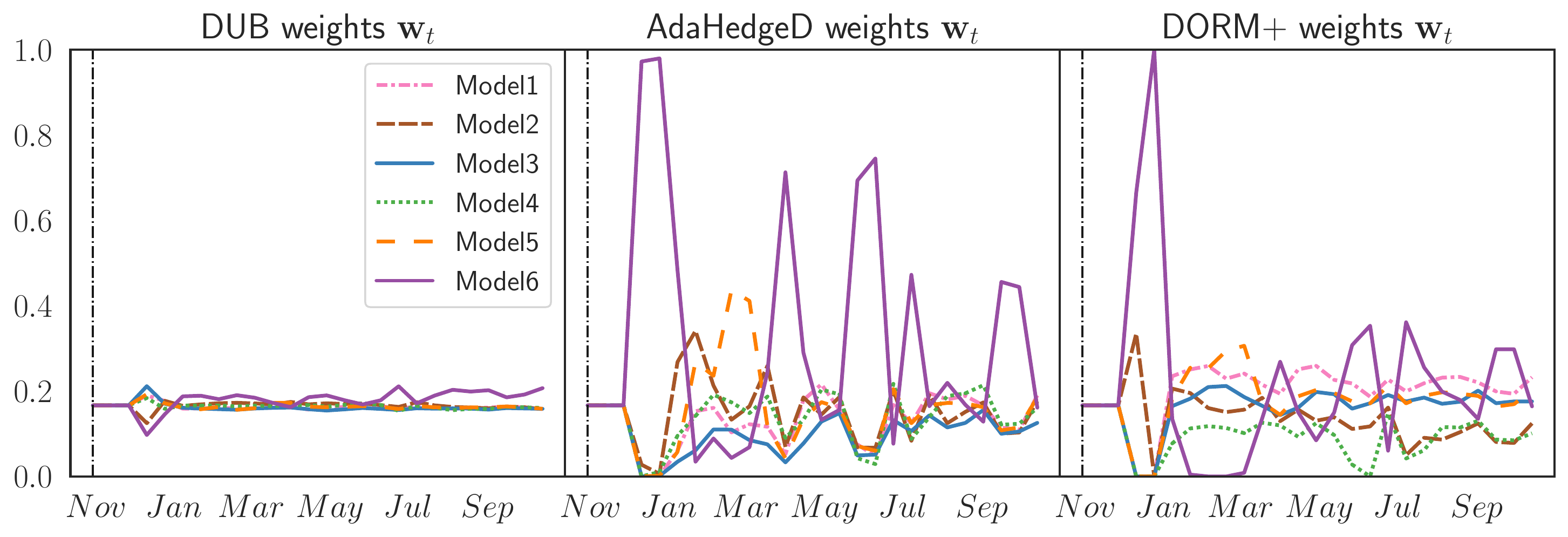} 
    \includegraphics[height=1.52in]{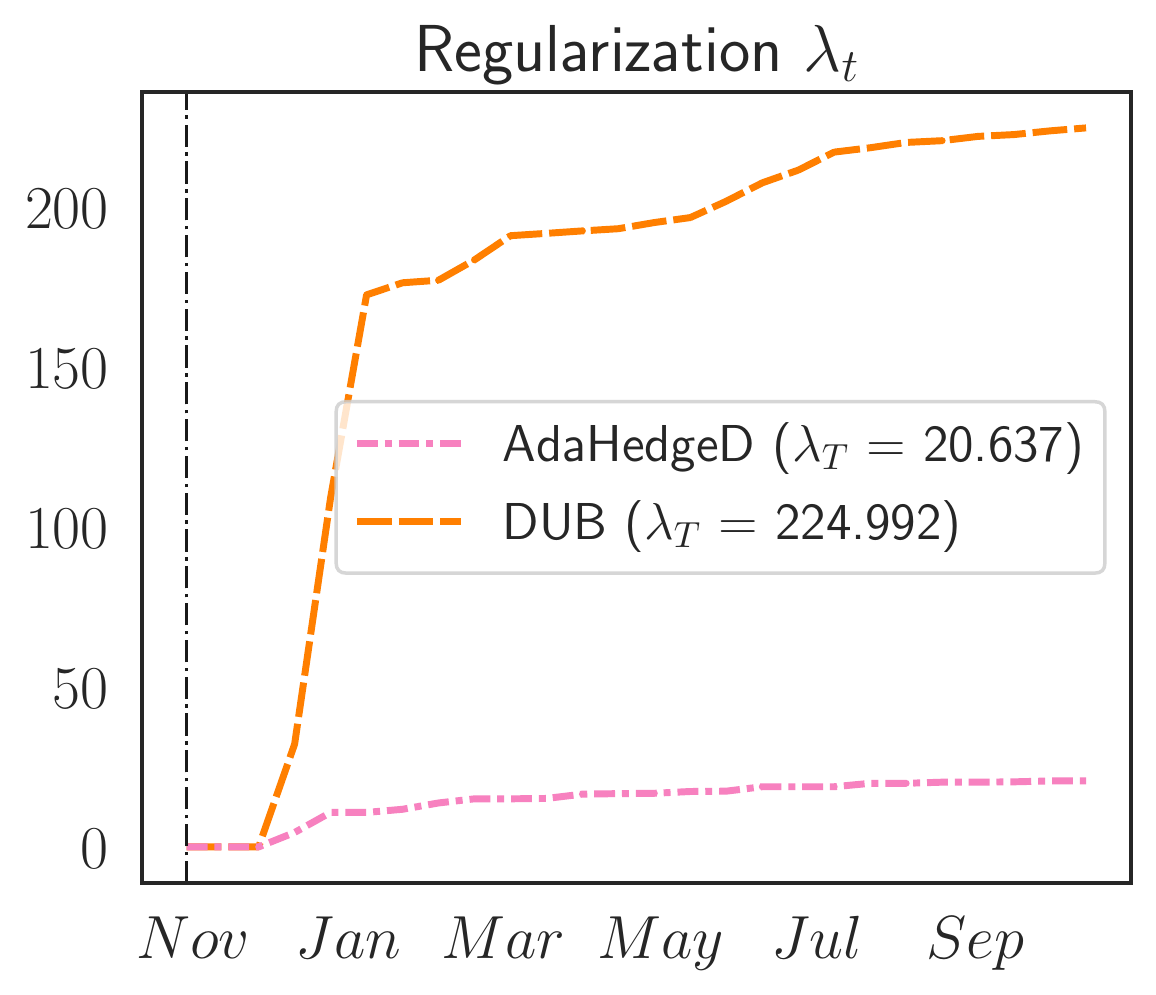}        
  }
  \subfigure[Precipitation Weeks 5-6]{
    \includegraphics[height=1.5in]{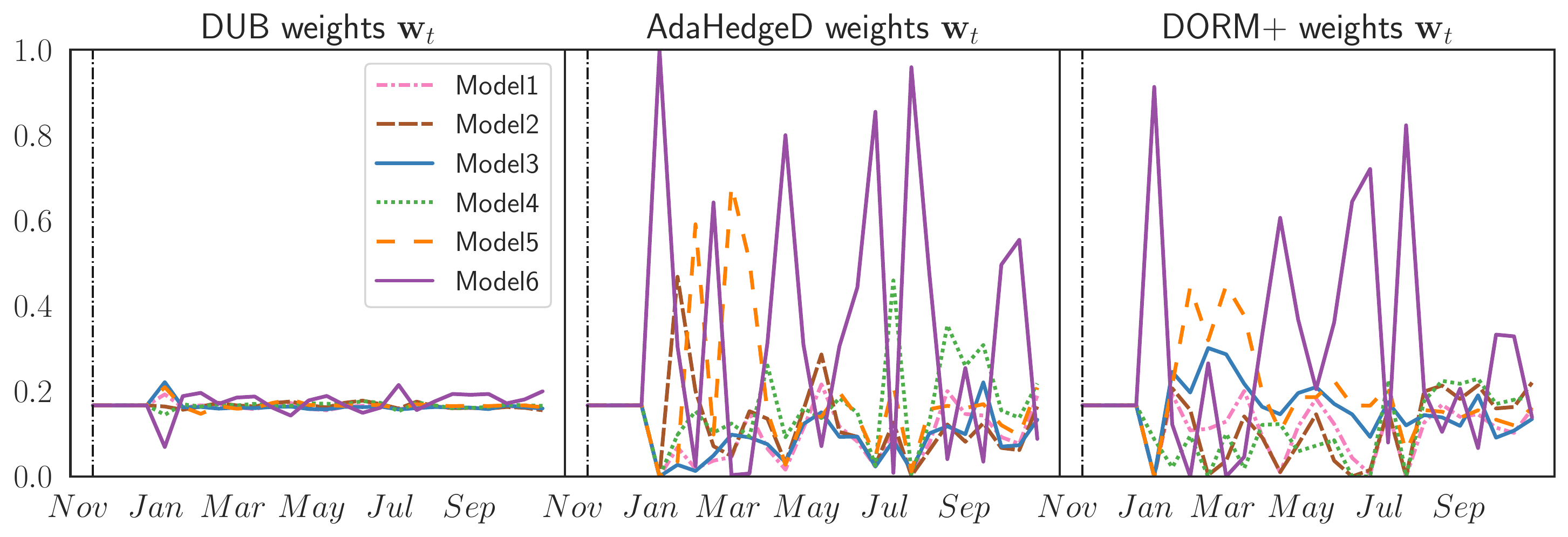} 
    \includegraphics[height=1.52in]{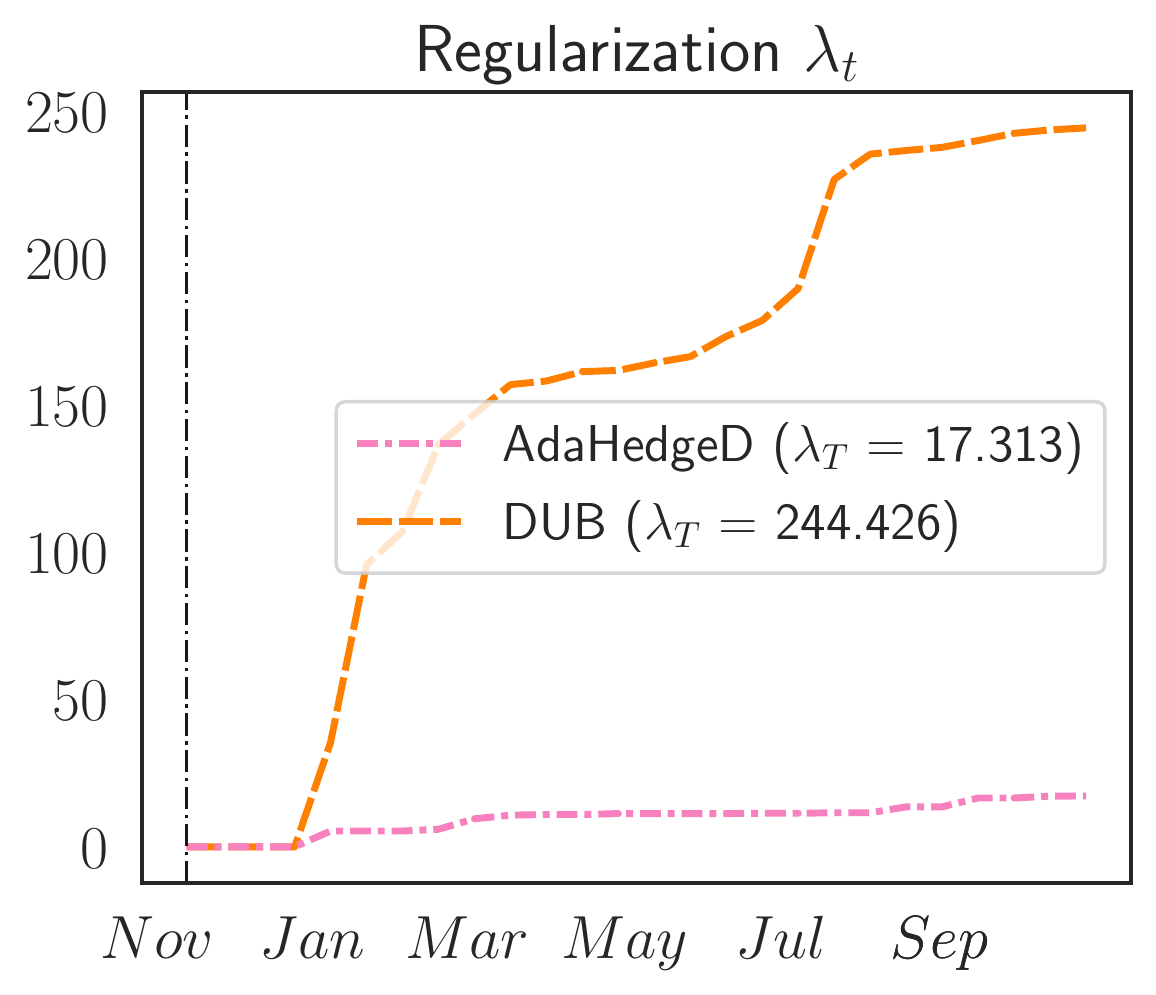}        
  }
  \subfigure[Temperature Weeks 3-4]{
    \includegraphics[height=1.5in]{figures/weights_regularization_contest_tmp2m_34w.pdf} 
    \includegraphics[height=1.52in]{figures/params_regularization_contest_tmp2m_34w.pdf}    
  }
  \subfigure[Temperature Weeks 5-6]{
    \includegraphics[height=1.5in]{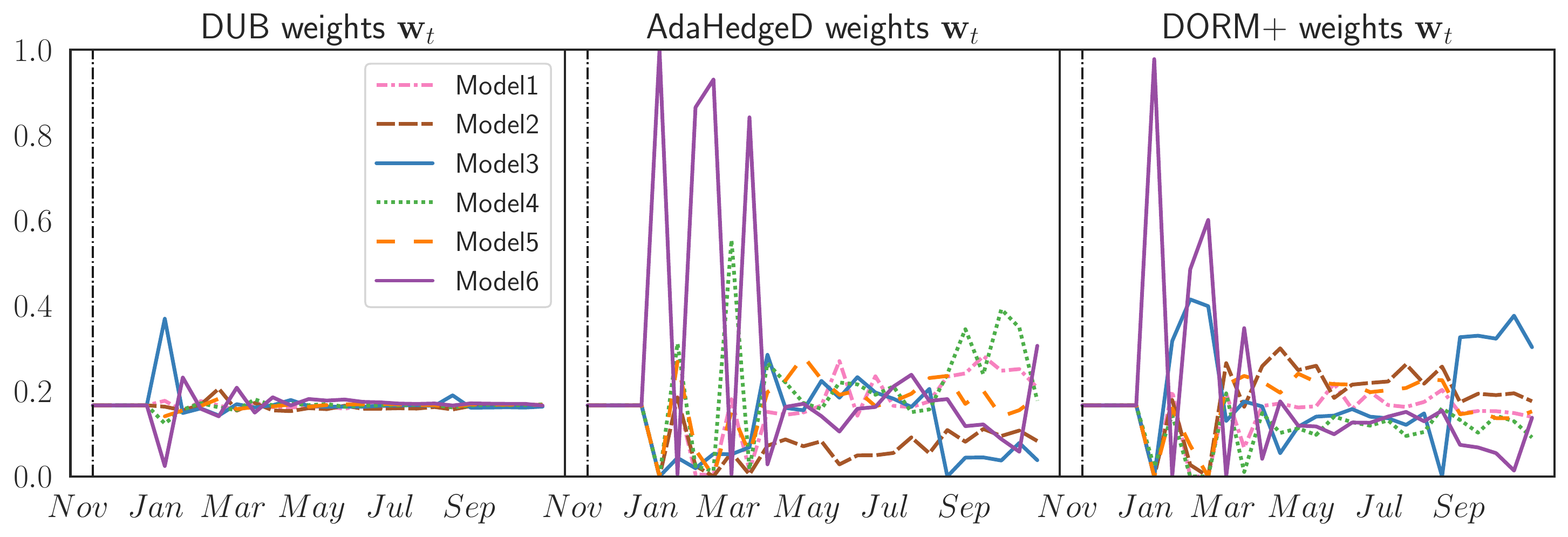}
    \includegraphics[height=1.52in]{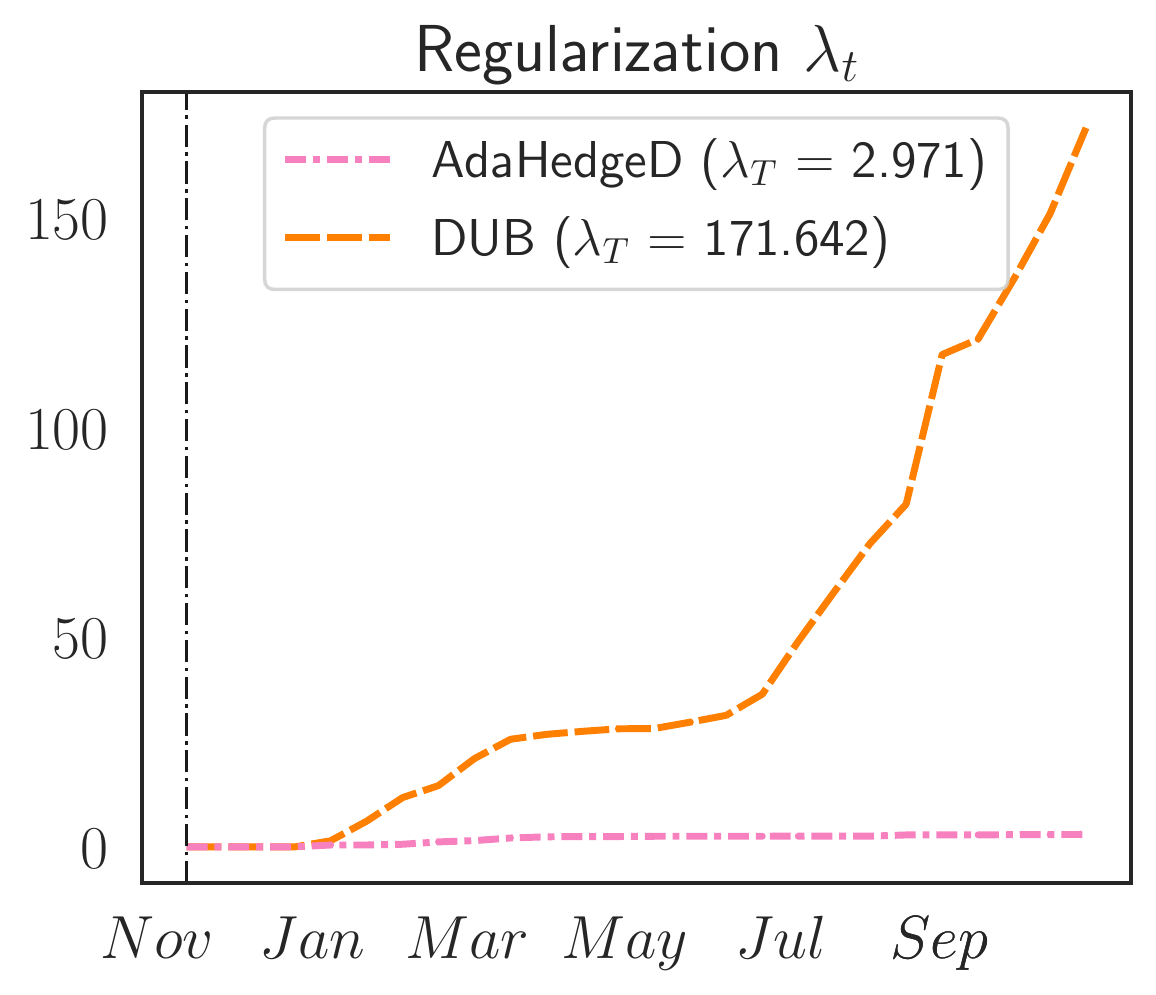} 
    }
\caption{\textbf{Impact of regularization:} The plays $\ww_t$ of online learning algorithms used to combine the input models for all four tasks in the 2020 evaluation year. The weights of \DUB and \AdaHedgeD appear respectively over and under regularized compared to \DORMP due to their selection of regularization strength $\lam_t$ (right).}
 \label{extended_reg_weights} 
\end{figure}

\newpage
\subsection{To Replicate or Not to Replicate}
\label{sec:replicate}
We compare the performance of replicated and non-replicated variants of our \DORMP algorithm as in \cref{sec:experiments}. Both algorithms perform well, but in all tasks, \DORMP outperforms replicated \DORMP (in which $D+1$ independent copies of \DORMP make staggered predictions). \cref{extended_replication} provides an example of the  weight plots produced by the replication strategy for all for tasks.

 The replicated algorithms only have the opportunity to learn from $T/(D+1)$ plays. For the 3-4 week horizons tasks $D=2$ and for the 5-6 week horizons tasks $D=3$. Because our forecasting horizons are short ($T=26$), further limiting the feedback available to each online learner via replication could be detrimental to practical model performance. 

\begin{table}[H]
\caption{\textbf{Replication RMSE:} Average RMSE of the 2010-2020 semimonthly forecasts for four tasks over over a $10$-year evaluation period for replicated versus standard \DORMP.}
\vskip 0.15in
\centering
\begin{small}
\begin{sc}
\begin{tabular}{lrr|rrrrrr}
\toprule
\multicolumn{1}{c}{} &  \multicolumn{1}{c}{DORM+} &   \multicolumn{1}{c}{Replicated DORM+} &  \multicolumn{1}{c}{Model1}  &  \multicolumn{1}{c}{Model2} &  \multicolumn{1}{c}{Model3} &  \multicolumn{1}{c}{Model4}  &  \multicolumn{1}{c}{Model5}  &  \multicolumn{1}{c}{Model6}  \\
\midrule
Precip. 3-4w & \textcolor{RoyalBlue}{\textbf{21.675}} &            21.720 &  \textcolor{RoyalBlue}{\textbf{21.973}} &  22.431 &  22.357 &  21.978 &  21.986 &  \textcolor{Maroon}{\textit{23.344}} \\
Precip. 5-6w & \textcolor{RoyalBlue}{\textbf{21.838}} &            21.851 &  22.030 &  22.570 &  22.383 &  22.004 &  \textcolor{RoyalBlue}{\textbf{21.993}} &  \textcolor{Maroon}{\textit{23.257}} \\
Temp. 3-4w   &  \textcolor{RoyalBlue}{\textbf{2.247}} &             2.249 &   \textcolor{RoyalBlue}{\textbf{2.253}} &   2.352 &   2.394 &   2.277 &   2.319 &   \textcolor{Maroon}{\textit{2.508}} \\
Temp. 5-6w   &  \textcolor{RoyalBlue}{\textbf{2.303}} &             2.315 &   \textcolor{RoyalBlue}{\textbf{2.270}} &   2.368 &   2.459 &   2.278 &   2.317 &   \textcolor{Maroon}{\textit{2.569}} \\
\bottomrule
\end{tabular}
\end{sc}
\end{small}
\vskip -0.2in
\end{table}

\begin{figure}[H]
  \begin{minipage}[b]{0.5\linewidth}
    \includegraphics[width=\linewidth]{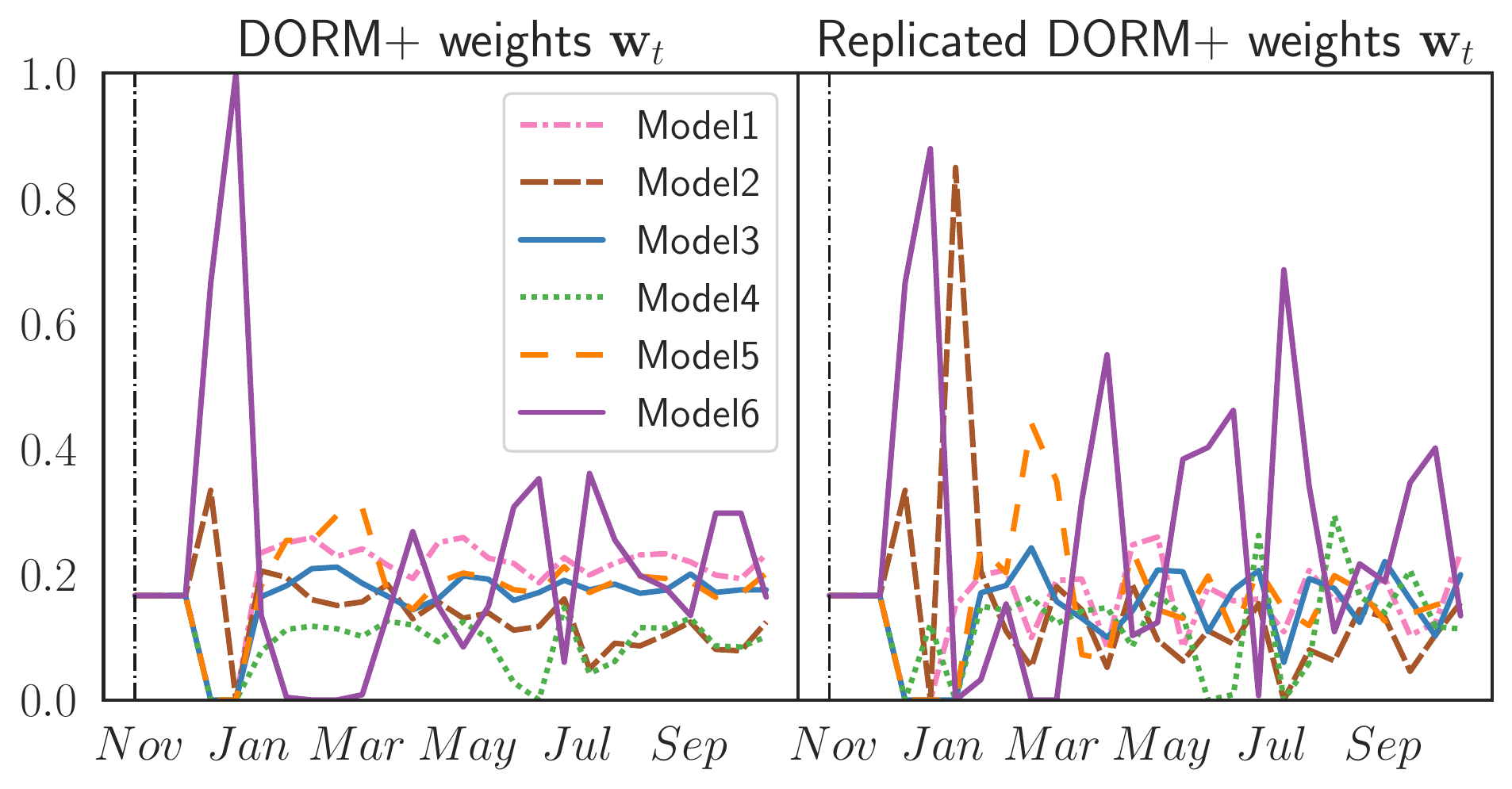} 
    \caption*{Precipitation Weeks 3-4} 
  \end{minipage} 
   \begin{minipage}[b]{0.5\linewidth}
    \includegraphics[width=\linewidth]{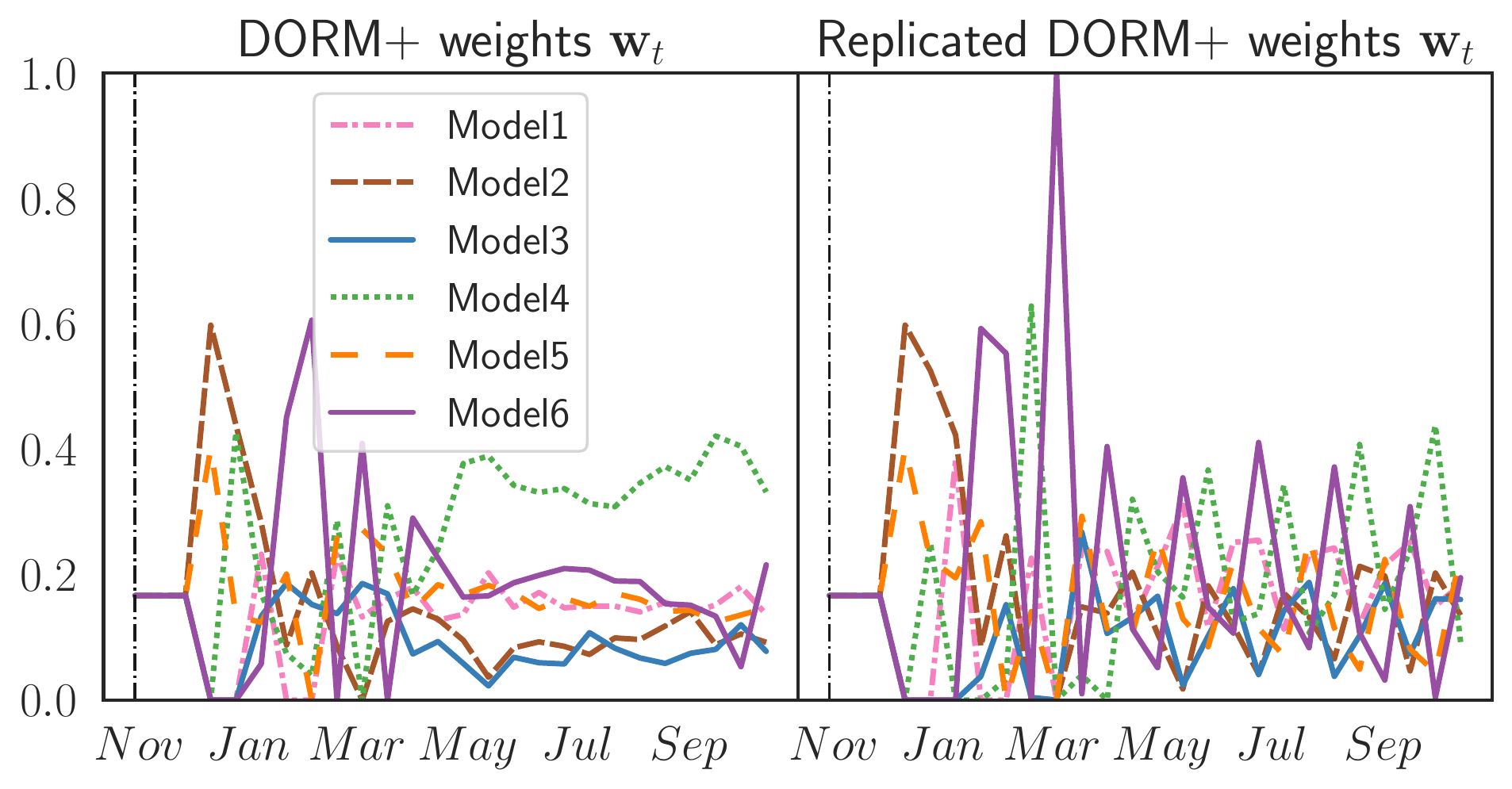} 
    \caption*{Temperature Weeks 3-4} 
  \end{minipage} 
  \begin{minipage}[b]{0.5\linewidth}
    \includegraphics[width=\linewidth]{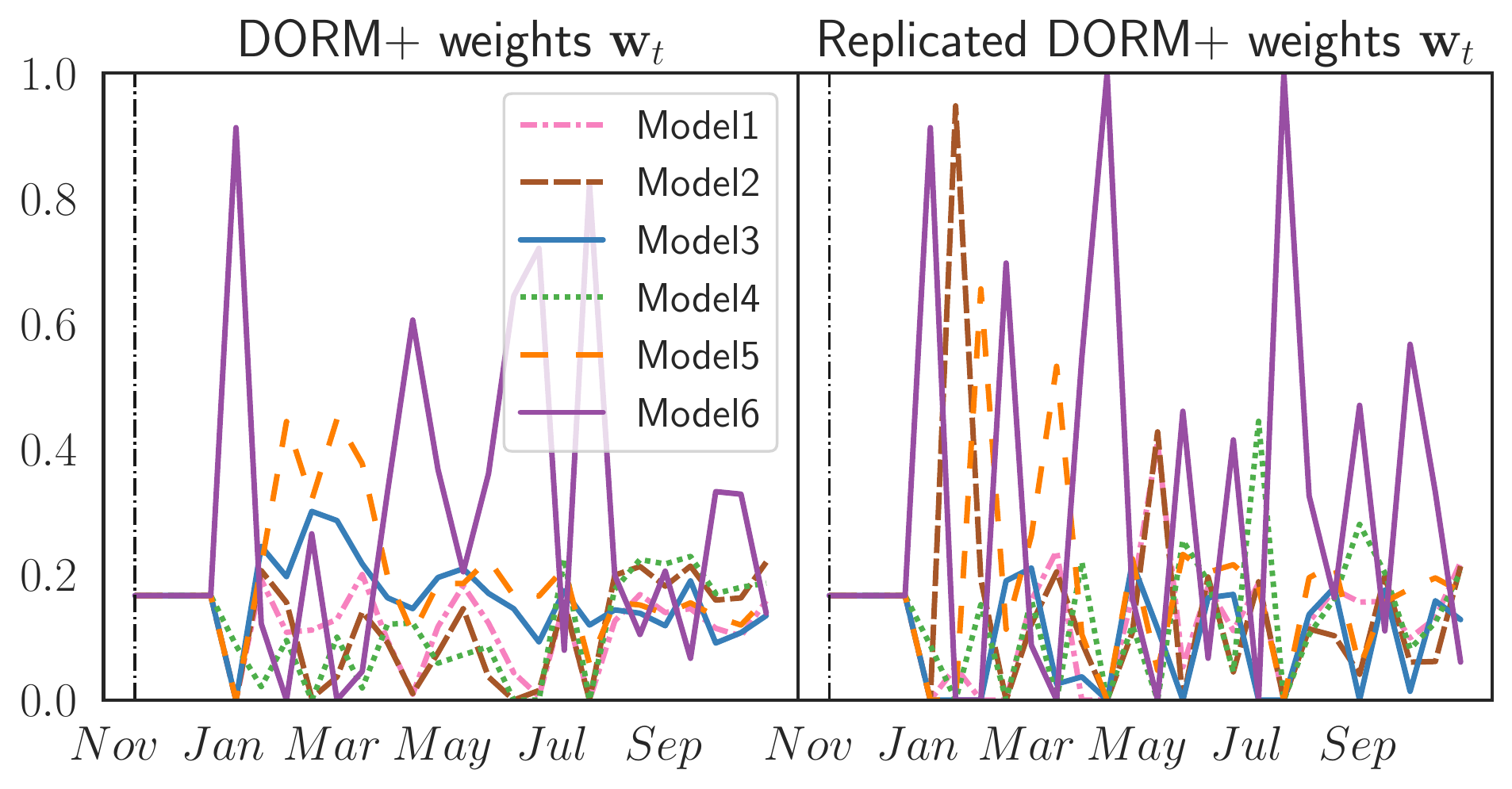} 
    \caption*{Precipitation Weeks 5-6} 
  \end{minipage} 
  \begin{minipage}[b]{0.5\linewidth}
    \includegraphics[width=\linewidth]{figures/weights_rep_contest_tmp2m_56w.pdf}
    \caption*{Temperature Weeks 5-6} 
  \end{minipage} 
  \caption{\textbf{Replication weights}: The plays $\ww_t$ of \DORMP and replicated \DORMP for all four tasks in the final evaluation year.}
  \label{extended_replication}
\end{figure}

\newpage
\subsection{Learning to Hint}
We examine the effect of optimism on the \DORMP algorithms and the ability of our ``learning to hint'' strategy to recover the performance of the best optimism strategy in retrospect as described in \cref{sec:experiments}. We use \DORMP as the meta-algorithm for hint learning to produce the \texttt{learned} optimism strategy that plays a convex combination of the three constant hinters.

As reported in the main text, the regret of the base algorithm using the learned hinting strategy generally falls between the worst and the best hinting strategy for any given year. Because the best hinting strategy for any given year is unknown \textit{a priori}, the adaptivity of the hint learner is useful practically. Currently, the hint learner is only optimizing a loose upper bound on base problem regret. Deriving loss functions for hint learning that more accurately quantify the effect of the hinter on base model regret is an important next step in achieving negative regret for online hinting algorithms.

\begin{figure}[H] 
  \begin{minipage}[b]{0.5\linewidth}
    \includegraphics[width=\linewidth]{figures/regret_hinting_dormplus_contest_precip_34w.pdf} 
    \caption*{Precipitation Weeks 3-4} 
  \end{minipage} 
  \begin{minipage}[b]{0.5\linewidth}
    \includegraphics[width=\linewidth]{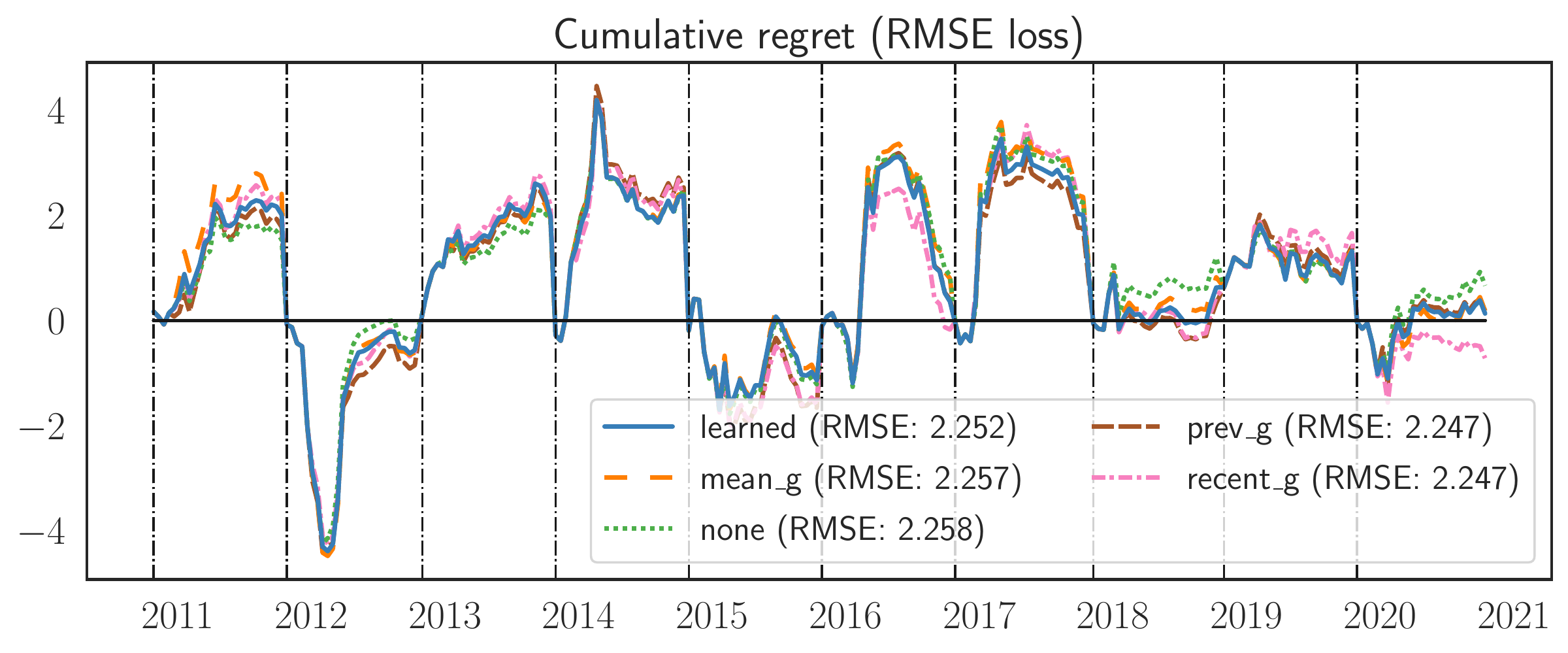}
    \caption*{Temperature Weeks 3-4} 
  \end{minipage}  
  \begin{minipage}[b]{0.5\linewidth}
    \includegraphics[width=\linewidth]{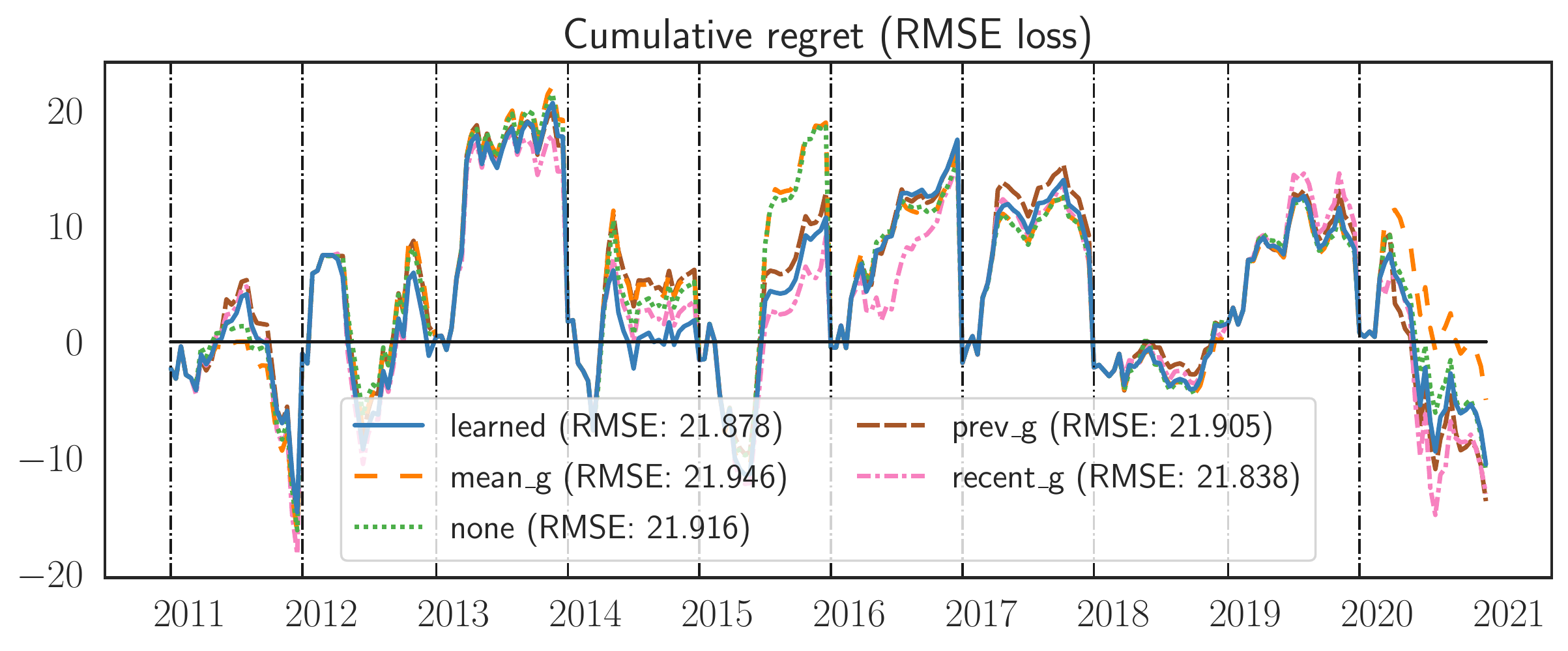}
    \caption*{Precipitation Weeks 5-6} 
  \end{minipage} 
  \begin{minipage}[b]{0.5\linewidth}
    \includegraphics[width=\linewidth]{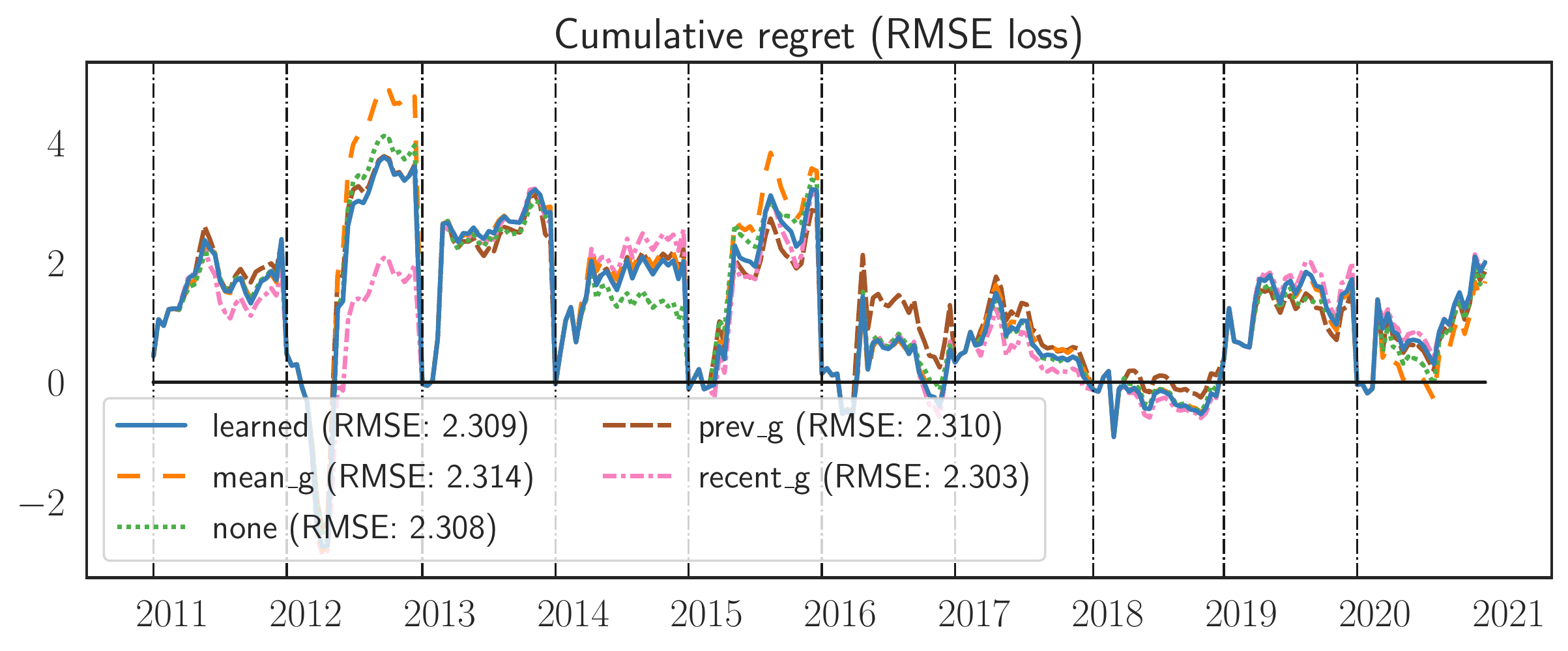}
    \caption*{Temperature Weeks 5-6} 
  \end{minipage} 
  \caption{\textbf{Overall regret:} Yearly cumulative regret under the RMSE loss for \DORMP using the three constant hinting strategies presented and the learned hinter, over the $10$-year evaluation period. }
 \label{extended_hinting}   
\end{figure}

\subsection{Impact of Different Forms of Optimism} \label{sec:future_past}
The regret analysis presented in this work suggest that optimistic strategies under delay can benefit from  hinting at both the ``past'' $\g_{t-D:t-1}$ missing losses and the ``future'' unobserved loss $\g_t$. 
To study the impact of different forms of optimism on \DORMP, we provide a \texttt{recent\_g} hint for either only the missing \textbf{future} loss $\g_t$, only the missing \textbf{past} losses $\g_{t-D:t-1}$, or both \textbf{past and future} losses (the strategy used in this paper) $\g_{t-D:t}$. Inspired by the recommendation of an anonymous reviewer, we also test two hint settings that only hint at the future unobserved loss but multiply the weight of that hint by \textbf{2D+1} or \textbf{3D+1}, effectively increasing the importance of the future hint in the online learning optimization. \cref{fig:hint_past_future} presents the experimental results.

\begin{figure}[h!]
    \centering
    \includegraphics[width=1\columnwidth]{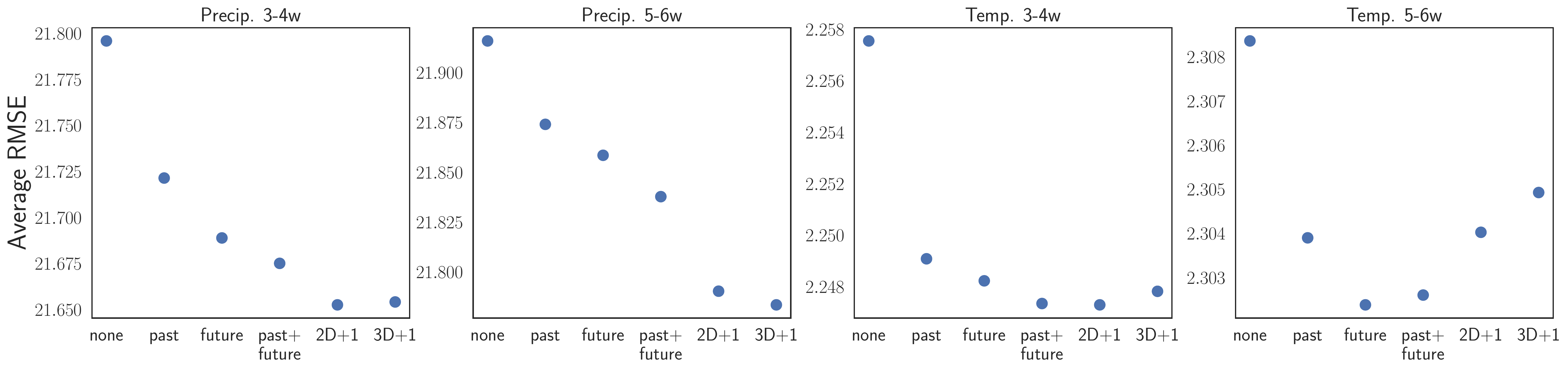}
    \vspace{0.1cm}
    \caption{DORM+ average RMSE as in \cref{exp-zoo-table} as a function of optimism strategy; see \cref{sec:future_past} for details.%
    }
    \label{fig:hint_past_future}
\end{figure}

In this experiment, all settings of optimism improve upon the non-optimistic algorithm, and, for all tasks,  providing hints for missing future losses outperforms hinting at missing past losses.
For all tasks save Temp.~5-6w, hinting at both missing past and future losses yields a further improvement. The 2D+1 and 3D+1 settings demonstrate that, for some tasks, increasing the magnitude of the optimistic hint can further improve performance in line with the online gradient descent predictions of \citet[Thm.~13]{hsieh2020multi}.
\section{Algorithmic Details} \label{sec:algorithm_details}

\subsection{\ODAFTRL with \AdaHedgeD and \DUB tuning}
The \AdaHedgeD and \DUB algorithms presented in the experiments are implementations of \ODAFTRL with a negative entropy regularizer $\reg(\ww) = \sum_{j=1}^d \ww_j \ln \ww_j + \ln d$, which is $1$-strongly convex with respect to the norm $\norm{\cdot}_1$ \citep[Lemma 16]{shalev2007online} with dual norm $\norm{\cdot}_{\infty}$. Each algorithm optimizes over the simplex and competes with the simplex: $\wset = \uset = \simplex_{d-1}$. We choose $\alpha = \sup_{\uu\in\uset} \reg(\uu) = \ln(d)$. In the following, define $\reg_t \defeq \lam_t \reg$ for $\lam_t \geq 0$. Our derivations of the update equations for \AdaHedgeD and \DUB make use of the following properties of the negative entropy regularizer, proved in \cref{proof_entropy_properties}.

\begin{lemma}[Negative entropy properties] \label{entropy_properties}
The negative entropy regularizer $\reg(\ww) = \sum_{j=1}^d \ww_j \ln \ww_j + \ln d$ with $\reg_t = \lam_t \reg$ for $\lam_t \geq 0$ satisfies the following properties on the simplex $\wset = \simplexd$. 
\begin{talign}
\reg_{\wset}^*(\theta) 
    &\defeq \sup_{\ww \in \wset} \inner{\ww}{\theta} - \reg(\ww) 
    = \ln \Big(\sum_{j=1}^d \exp(\theta_j) \Big) - \ln d, \\
(\lam \reg)_{\wset}^*(\theta) 
    &\defeq 
        \sup_{\ww \in \wset} \inner{\ww}{\theta} - \lam\reg(\ww)
    = 
    \begin{cases}
        \lam \reg_{\wset}^*(\theta / \lam) 
        =
        \lam \ln (\sum_{j=1}^d \exp(\theta_j / \lam) ) - \lam \ln d, & \text{if } \lam > 0 \\
    \max_{j \in [d]} \theta_j &  \text{if } \lam = 0
    \end{cases},
         \\
\ww^*(\theta,\lam) 
    &\defeq
    \begin{cases}
    \frac{\exp(\theta/\lam)}{\sum_{j=1}^d\exp(\theta_j/\lam)} 
        & \text{if } \lam > 0 \\
    \frac{\indic{\theta = \max_j \theta_j}}{\sum_{k \in [d]} \indic{\theta_k = \max_j \theta_j}}
        & \text{if } \lam = 0
    \end{cases}
    \in \argmin_{\ww \in \wset} \lam\reg(\ww) - \inner{\ww}{\theta} 
    \subseteq
    \subdiff (\lam\reg)_{\wset}^*(\theta).
\end{talign}
\end{lemma}

Our next corollary concerning optimal \ODAFTRL objectives follows directly from \cref{entropy_properties}.
\begin{corollary}[Optimal \ODAFTRL objectives] \label{ftrl-obj-def}
Instantiate the notation of \cref{entropy_properties}, and define the functions $\obj_{t}(\ww,\lam) \defeq \lam \reg(\ww) + \inner{\g_{1:t-1}}{\ww}$ 
for $\ww\in\wset$. Then
\begin{talign}
 - (\lam \reg)_{\wset}^*(-(\g_{1:t-1} + \h))
    &=
        \inf_{\ww\in\wset} \obj_{t}(\ww,\lam) + \inner{\h}{\ww}
    \qtext{and}
        \\
\ww^*(-(\g_{1:t-1} + \h),\lam)
    &=
        \argmin_{\ww\in\wset} \obj_{t}(\ww,\lam)  + \inner{\h}{\ww}.
\end{talign}
\end{corollary}

Using \cref{entropy_properties,ftrl-obj-def}, we can derive an expression, proved in \cref{proof_adahedge_del_setting}, for the \AdaHedgeD $\delta_t$ updates.
\begin{proposition}[\AdaHedgeD $\delta_t$] \label{adahedge_del_setting}
Instantiate the notation of \cref{adahedged_regret}, and
define the auxiliary hint vector
\begin{talign}
\label{eq:adahedged_aux_hint}
\hat\h_t 
    \defeq \g_{t-D:t} + \sigma_t(\h_{t} - \g_{t-D:t})
\qtext{for}
\sigma_t
    \defeq \min(\frac{\dualnorm{\g_t}}{\dualnorm{\h_{t} - \g_{t-D:t}}},1)
\end{talign}
along with the scalars
\begin{talign}
c_*= \max_{j : \ww_{t,j} \neq 0} \h_{t,j} - \g_{t-D:t,j}
\qtext{and}
\hat c_*= \max_{j : \hat\ww_{t,j} \neq 0} \hat\h_{t,j} - \g_{t-D:t,j}
\end{talign}
for 
\begin{talign}
\bar\ww_t 
    &= \argmin_{\ww\in\wset} \obj_{t+1}(\ww,\lam_t) = \frac{\exp(-\g_{1:t}/\lam_t)}{\sum_{j=1}^d\exp(-\g_{1:t,j}/\lam_t)} 
\qtext{and} \\
\hat\ww_t 
    &= \argmin_{\ww\in\wset} \obj_{t+1}(\ww,\lam_t) + \inner{\hat\h_{t} - \g_{t-D:t}}{\ww}
    = \frac{\exp(-(\g_{1:t-D-1}+\hat\h_{t})/\lam_t)}{\sum_{j=1}^d\exp(-(\g_{1:t-D-1,j}+\hat\h_{t,j})/\lam_t)} 
\end{talign} 
by \cref{ftrl-obj-def}.
If $\lam_t > 0$,
\begin{talign}
\delta_t
    &= 
        \min(\delta_t^{(1)}, \delta_t^{(2)}, \delta_t^{(3)})_+ \qtext{for} \\
\delta_t^{(1)}
    &= 
        \obj_{t+1}(\ww_t,\lam_t) - \obj_{t+1}(\bar\ww_t,\lam_t) \\
    &= 
        \lam_t \ln(\sum_{j\in[d]} \ww_{t,j} \exp((\h_{t,j} - \g_{t-D:t,j}) / \lam_t)) + \inner{\g_{t-D:t} - \h_{t}}{\ww_t} \\
    &= 
        \lam_t \ln(\sum_{j\in[d]} \ww_{t,j} \exp((\h_{t,j} - \g_{t-D:t,j} - c_*) / \lam_t)) + \inner{\g_{t-D:t} - \h_{t}}{\ww_t} + c_*, \\
\delta_t^{(2)}
    &= \inner{\g_t}{\ww_t - \bar\ww_t}, \qtext{and} \\
\delta_t^{(3)} 
    &= 
        \obj_{t+1}(\hat\ww_t,\lam_t) - \obj_{t+1}(\bar\ww_t,\lam_t) +  \inner{\g_t}{\ww_t - \hat\ww_t} \\
    &=
        \lam_t \ln(\sum_{j\in[d]} \hat\ww_{t,j} \exp((\hat\h_{t,j} - \g_{t-D:t,j}) / \lam_t)) + \inner{\g_{t-D:t} - \hat\h_{t}}{\hat\ww_t} +  \inner{\g_t}{\ww_t - \hat\ww_t} \\
    &=
        \lam_t \ln(\sum_{j\in[d]} \hat\ww_{t,j} \exp((\hat\h_{t,j} - \g_{t-D:t,j} - \hat c_*) / \lam_t)) + \inner{\g_{t-D:t} - \hat\h_{t}}{\hat\ww_t} + \hat c_* +  \inner{\g_t}{\ww_t - \hat\ww_t}.
\end{talign} 
If $\lam_t = 0$,
\begin{talign}
\delta_t
    &= 
        \min(\delta_t^{(1)}, \delta_t^{(2)}, \delta_t^{(3)})_+ \qtext{for} \\
\delta_t^{(1)}
    &= 
    \inner{\g_{1:t}}{\ww_t} - \min_{j\in[d]} \g_{1:t,j}, \\
\delta_t^{(2)}
    &= 
    \inner{\g_t}{\ww_t - \bar\ww_t}, \qtext{and} \\
\delta_t^{(3)}
    &=
    \inner{\g_{1:t}}{\hat\ww_t} - \min_{j\in[d]} \g_{1:t,j} 
    + \inner{\g_t}{\ww_t - \hat\ww_t}.
\end{talign}
\end{proposition}

Leveraging these results, we present the pseudocode for the \AdaHedgeD and \DUB instantiations of \ODAFTRL in \cref{alg:delay_opt_adahedge}.
\begin{algorithm}[tbh] 
\caption{\ODAFTRL with $\wset = \simplexd$, $\reg(\ww) = \sum_{j=1}^d \ww_j \ln\ww_j + \ln(d)$, delay $D\geq 0$, and tuning strategy \texttt{tuning}}
\label{alg:delay_opt_adahedge}
\begin{algorithmic}[1]
\STATE Parameter $\alpha = \sup_{\uu\in\simplexd} \reg(\uu) =  \ln(d)$
\STATE Initial regularization weight: $\lambda_0 = 0$
\IF{\texttt{tuning} is \DUB}
\STATE Initial regularization sum: $\Delta_0 = 0$
\STATE Initial maximum: $\ab^{\max} = 0$
\ENDIF
\STATE Initial subgradient sum: $\g_{1:1} = \mathbf{0} \in \R^d$
\STATE Dummy losses and iterates: $\g_{-D} = \cdots = \g_{0} = \mathbf{0} \in \R^d$, $\ww_{-D} = \cdots = \ww_{0} = \mathbf{0} \in \R^d$
\FOR{$t = 1, \dots, T$}
\STATE Receive hint $\h_t \in \R^d$
\STATE Output $\ww_{t} = \argmin_{\ww\in\wset} \obj_{t-D}(\ww, \lam_t) + \inner{\h_t}{\ww}$ as in \cref{ftrl-obj-def}
\STATE Receive $\g_{t-D} \in \R^d$ and pay $\langle \g_{t-D}, \ww_{t-D} \rangle$ 
\STATE Update subgradient sum $\g_{1:t-D} = \g_{1:t-D-1} + \g_{t-D}$ 
\IF{\texttt{tuning} is \AdaHedgeD}
\STATE Compute the auxiliary play $\bar\ww_{t-D} = \argmin_{\ww\in\wset} \obj_{t-D+1}(\ww, \lam_{t-D})$ as in \cref{ftrl-obj-def}
\STATE  Compute the auxiliary regret term $\delta_{t-D}^{(1)} = \obj_{t-D+1}(\ww_{t-D}, \lam_{t-D}) - \obj_{t-D+1}(\bar\ww_{t-D},\lam_{t-D})$ as in \cref{adahedge_del_setting} 
\STATE Compute the drift term $\delta_{t-D}^{(2)} = \inner{\g_{t-D}}{\ww_{t-D} - \bar\ww_{t-D}}$ 
\STATE Compute the auxiliary hint \cref{eq:adahedged_aux_hint}
$\hat\h_{t-D} 
    \defeq \g_{t-2D:t-D} + \min(\frac{\dualnorm{\g_{t-D}}}{\dualnorm{\h_{t-D} - \g_{t-2D:t-D}}},1)(\h_{t-D} - \g_{t-2D:t-D})
$
\STATE Compute the auxiliary play $\hat\ww_{t-D} = \argmin_{\ww\in\wset} \obj_{t-D+1}(\ww, \lam_{t-D}) +  \inner{\hat\h_{t-D} - \g_{t-2D:t-D}}{\ww}$ as in \cref{ftrl-obj-def}
\STATE Compute the regret term $\delta_{t-D}^{(3)} = \obj_{t-D+1}(\hat\ww_{t-D}, \lam_{t-D}) - \obj_{t-D+1}(\bar\ww_{t-D},\lam_{t-D}) + \inner{\g_{t-D}}{\ww_{t-D} - \hat\ww_{t-D}}$ as in \cref{adahedge_del_setting} 
\STATE Update $\lam_{t+1} = \lam_{t} + \frac{1}{\alpha} \min(\delta_{t-D}^{(1)}, \delta_{t-D}^{(2)}, \delta_{t-D}^{(3)})_+$ as in \cref{def-deltat} 
\ELSIF{\texttt{tuning} is \DUB}
\STATE Compute $\abftrl{t-D} = 2 \min\big(\infnorm{\g_{t-D}}, \infnorm{\h_{t-D} - \sum_{s=t-2D}^{t-D} \g_s }\big)$ as in \cref{abf_def}
\STATE Compute $\bbftrl{t-D} = \half \infnorm{\h_{t-D} -\sum_{s=t-2D}^{t-D} \g_s}^2 - \half(\infnorm{\h_{t-D} -\sum_{s=t-2D}^{t-D} \g_s} -  \infnorm{\g_{t-D}})_+^2$ as in \cref{abf_def}
\STATE Update $\Delta_{t+1} = \Delta_t + \abftrl{t-D}^2 + 2\alpha\bbftrl{t-D}$  
\STATE Update maximum $\ab^{\max}  = \max(\ab^{\max}, \abftrl{t-2D:t-D-1})$
\STATE Update $\lam_{t+1} = \frac{1}{\alpha}(2\ab^{\max} + \sqrt{\Delta_{t+1}})$ as in \DUB
\ENDIF
\ENDFOR
\end{algorithmic}
\end{algorithm}

\subsection{\DORM and \DORMP}
\label{sec:algorithm_details_dorm_dormp}
The \DORM and \DORMP algorithms presented in the experiments are implementations of \ODAFTRL and \DOOMD respectively that play iterates in 
$\wset \defeq \simplexd$ using the default value $\lam=1$. Both algorithms use a $p$-norm regularizer $\reg = \half \norm{\cdot}_p^2$, which is $1$-strongly convex with respect to $\norm{\cdot} = \sqrt{p-1}\pnorm{\cdot}$ \citep[see][Lemma 17]{shalev2007online} with $\norm{\cdot}_* = \frac{1}{\sqrt{p-1}}\qnorm{\cdot}$. For the paper experiments, we choose the optimal value $q = \inf_{q' \geq 2} d^{2/q'}(q'-1)$ to obtain $\ln(d)$ scaling in the algorithm regret; for $d=6$, $p=q=2$. The update equations for each algorithm are given in the main text by \DORM and \DORMP respectively. The optimistic hinters provide delayed gradient hints $\tildeg_t$, which are then used to compute regret gradient hints $\tilde{\rr}_t$, where $\tilde{\rr}_t = \inner{\tildeg_t}{\ww_t} - \tildeg_t$ and $\h_t = \sum_{s=t-D}^{t-1} \tilde{\rr}_s + \inner{\tildeg_t}{\ww_{t-1}} - \tildeg_t$.

\subsection{Adaptive Hinting}
For the adaptive hinting experiments, we use the \DORMP as both the base and hint learner. For the hint learner with \DORM base algorithm, the hint loss function is given by
\cref{eq:hinting-loss-doomd} with $q=2$. The plays of the online hinter $\omega_t$ are used to generate the hints $\h_t$ for the base algorithm using the hint matrix $H_t \in \R^{d \times m}$. The $j$-th column of $H_t$ contains hinter $j$'s predictions for the cumulative missing regret subgradients $\rr_{t-D:t}$. The final hint for the base learner is $\h_t = H_t \omega_t$. Psuedo-code for the adaptive hinter is given in \cref{alg:adaptive_hinting}.

\begin{algorithm}[tbh]
\caption{Learning to hint with \DORMP($q$=2) hint learner, \DORMP base learner, and delay $D\geq 0$}
\label{alg:adaptive_hinting}
\begin{algorithmic}[1]
\STATE Subgradient vector: $\g_{-D}, \cdots \g_0 =  \boldzero\in\reals^d$
\STATE Meta-subgradient vector: $\metagrad_{-D}, \cdots \metagrad_{0} = \boldzero\in\reals^m$
\STATE Initial instantaneous regret:
$\rr_{-D} = \boldzero\in\reals^d$ 
\STATE Initial instantaneous meta-regret: $\rho_{-D} = \boldzero\in\reals^{m}$ 
\STATE Initial hint $\h_0 = \mathbf{0} \in \reals^{d}$
\STATE Initial orthant meta-vector: $\tilde{\omega}_0 = \boldzero \in\reals^m$
\FOR{$t = 1, \dots, T$}
\STATE // \texttt{Update online hinter using \DORMP with $q=2$} 
\STATE Find optimal unnormalized hint combination vector $\omegaorth_t = \max(\mathbf{0}, \omegaorth_{t-1} + \rho_{t-D-1})$ 
\STATE Normalize: $\omega_{t} = \begin{cases} 
\boldone/m & \text{if } \omegaorth_{t} = \boldzero \\
\omegaorth_{t}/\inner{\boldone}{\omegaorth_{t}} & \text{otherwise}
\end{cases}$
\STATE Receive hint matrix: $H_t \in \R^{d \times m}$ in which each column is a hint for $\sum_{s=t-D}^t \rr_s$
\STATE Output hint $\h_t = H_t \omega_t$
\STATE // \texttt{Update \DORMP base learner and get next play} 
\STATE Output $\ww_t = \DORMP(\g_{t-D-1},\h_t)$
\STATE Receive $\g_{t-D} \in \R^d$ and pay $\langle \g_{t-D}, \ww_{t-D} \rangle$ 
\STATE Compute instantaneous regret
$\rr_{t-D} = \boldone\inner{\g_{t-D}}{\ww_{t-D}} - \g_{t-D}$
\STATE Compute hint meta-subgradient $\metagrad_{t-D} \in \subdiff l_{t-D}(\omega_{t-D}) \in\reals^m$ as in \cref{eq:hinting-gradient}
\STATE Compute instantaneous hint regret
$\rho_{t-D} = 
\boldone\inner{\metagrad_{t-D}}{\omega_{t-D}} - \metagrad_{t-D}$ 
\ENDFOR
\end{algorithmic}
\end{algorithm}

\subsection{Proof of \cref{entropy_properties}: Negative entropy properties} \label{proof_entropy_properties}
The expression of the Fenchel conjugate for $\lam > 0$ is derived by solving an appropriate constrained convex optimization problem for $\ww = \simplexd$, as shown in \citet[Section 6.6]{Orabona2019AMI}. The value of $\ww^*(\theta,\lam) \in \subdiff (\lam\reg)_{\wset}^*(\theta)$ uses the properties of the Fenchel conjugate \citep[Theorem 5.5]{rockafellar1970convex,Orabona2019AMI} and is shown in \citet[Theorem 6.6]{Orabona2019AMI}.

\subsection{Proof of \cref{adahedge_del_setting}: \AdaHedgeD $\delta_t$}  \label{proof_adahedge_del_setting}
First suppose $\lam_t > 0$. 
The first term in the $\min$ of \AdaHedgeD's $\delta_t$ setting is derived as follows: 
\begin{talign} 
\delta_t^{(1)} 
    &\defeq \obj_{t+1}(\ww_t,\lam_t) - \obj_{t+1}(\bar\ww_t,\lam_t) \qtext{by definition \cref{def-deltat}} \\
    &= \obj_{t-D}(\ww_t,\lam_t) + \inner{\h_t}{\ww_t}+ \inner{\g_{t-D:t}-\h_t}{\ww_t} - \inf_{\ww\in\wset} \obj_{t+1}(\ww,\lam_t) \qtext{by definition of $\bar\ww_t$} \\
    &= \obj_{t-D}(\ww_t,\lam_t) + \inner{\h_t}{\ww_t}+ \inner{\g_{t-D:t}-\h_t}{\ww_t} + \lam_t \reg_{\wset}^*(-\g_{1:t} / \lam_t) \qtext{by \cref{ftrl-obj-def}} \\
    &= \lam_t \reg_{\wset}^*(-\g_{1:t} / \lam_t) - \lam_t \reg_{\wset}^*((-\h_t - \g_{1:t-D-1}) / \lam_t) + \inner{\g_{t-D:t}-\h_t}{\ww_t} \\
    & \hspace{0.2cm} \qtext{because $\ww_t \in \argmin_{\ww\in\wset} \obj_{t-D}(\ww_t,\lam_t) + \inner{\h_t}{\ww_t}$} \\
    &=  \lam_t ( \ln(\sum_{j=1}^d \exp(-\g_{1:t, j} / \lam_t)) - \lam_t ( \ln(\sum_{j=1}^d \exp((-\g_{1:t-D-1, j} - \h_{t,j}) / \lam_t)) + \inner{\g_{t-D:t}-\h_t}{\ww_t}   \qtext{by \cref{entropy_properties}} \\
    &= \lam_t  \ln\left(\sum_{j=1}^d\frac{ \exp(-\g_{1:t, j} / \lam_t)}{\sum_{j=1}^d \exp((-\g_{1:t-D-1, j} - \h_{t,j}) / \lam_t)} \right)  + \inner{\g_{t-D:t}-\h_t}{\ww_t} \\
    &= \lam_t  \ln\left(\sum_{j=1}^d \frac{ \exp((-\g_{1:t-D-1, j} - \h_{t,j}) / \lam_t) \exp((\h_{t,j} -\g_{t-D:t, j} )/ \lam_t)}{\sum_{j=1}^d \exp((-\g_{1:t-D-1, j} - \h_{t,j}) / \lam_t)} \right) + \inner{\g_{t-D:t}-\h_t}{\ww_t} \\
    &= \lam_t  \ln\left(\sum_{j=1}^d \ww_{t,j} \exp((\h_{t,j} -\g_{t-D:t, j} )/ \lam_t) \right) + \inner{\g_{t-D:t}-\h_t}{\ww_t} \qtext{by the expression for $\ww_t$ in \cref{ftrl-obj-def}.}
\end{talign}
The expression for the third term in the $\min$ of \AdaHedgeD's $\delta_t$ setting follows from identical reasoning.

Now suppose $\lam_t=0$. 
We have
\begin{talign} 
\delta_t^{(1)} 
    &\defeq \obj_{t+1}(\ww_t,\lam_t) - \obj_{t+1}(\bar\ww_t,\lam_t) \qtext{by definition \cref{def-deltat}} \\
    &= \inner{\g_{1:t}}{\ww_t} - \inf_{\ww\in\wset} \obj_{t+1}(\ww,\lam_t) \qtext{by definition of $\bar\ww_t$} \\
    &= \inner{\g_{1:t}}{\ww_t} - \min_{j\in[d]} \g_{1:t,j}\qtext{by \cref{ftrl-obj-def}.} 
\end{talign}
Identical reasoning yields the advertised expression for the third term.

\section{Extension to Variable and Unbounded Delays}
\label{sec:variable_delays}
\newcommand{\last}{\textup{last}}
\newcommand{\first}{\textup{first}}

In this section we detail 
how our main results 
generalize to the case of variable and potentially unbounded delays.
For each time $t$, 
we define
$\last(t)$ as the largest index $s$ for which $\g_{1:s}$ is observable at time $t$ (that is, available for constructing $\ww_{t}$) 
and 
$\first(t)$ as the first time $s$ at which $\g_{1:t}$ is observable at time $s$ (that is, available for constructing $\ww_s$).

\subsection{Regret of \varDOOMD}
Consider the \DOOMD variable-delay generalization
\begin{align}
    &\ww_{t+1} = \argmin_{\ww\in\wset} \,\inner{\g_{\last(t)+1:\last(t+1)} +\h_{t+1}-\h_t}{\ww} + \Breg_{\lam\reg}(\ww,\ww_t)     \label[name]{var-doomd}\tag{DOOMD with variable delays}
    \qtext{with} \h_0 \defeq \boldzero \qtext{and arbitrary} \ww_0.
\end{align}
We first note that \varDOOMD is an instance of \SOOMD respectively with a ``bad'' choice of optimistic hint $\tildeg_{t+1}$  that deletes the unobserved loss subgradients $\g_{\last(t+1)+1:t}$. 
\begin{lemma}[\varDOOMD is \SOOMD with a bad hint]\label{var_doomd_is_soomd}
\varDOOMD is \SOOMD with 
$
\tildeg_{t+1} 
        = \tildeg_t + \g_{\last(t)+1:\last(t+1)} - \g_t + \h_{t+1}-\h_t
        = \h_{t+1} 
        + \sum_{s=1}^t \g_{\last(s)+1:\last(s+1)} - \g_s.
        = \h_{t+1} - \g_{\last(t+1)+1:t}.
$
\end{lemma}
The following result now follows immediately from \cref{soomd_regret,var_doomd_is_soomd}.
\begin{theorem}[Regret of \varDOOMD]\label{var_doomd_regret}
If $\reg$ is differentiable
and
$\h_{T+1} 
\defeq \g_{\last(T+1)+1:T}$, 
then, for all $\uu\in\wset$, the \varDOOMD iterates $\ww_t$ satisfy
\begin{talign}
\regret_T(\uu) 
    &\leq \Breg_{\lam \reg}(\uu,\ww_0)
    + \frac{1}{\lam} \sum_{t=1}^T \bbomd{t}^2,  \qtext{for} \\
 \bbomd{t}^2 
    &\defeq \huber(\staticnorm{\h_t - \sum_{s=\last(t)+1}^t\g_s }_*,\norm{\g_{\last(t)+1:\last(t+1)}+\h_{t+1}-\h_t}_*).
\end{talign}
\end{theorem}

\subsection{Regret of \varODAFTRL}
Consider the \ODAFTRL variable-delay generalization
\begin{align}
    \label[name]{var-odaftrl}\tag{ODAFTRL with variable delays}
    &\ww_{t+1} = \argmin_{\ww\in\wset} \,\inner{\g_{1:\last(t+1)} + \h_{t+1}}{\ww} + \lam_{t+1} \reg(\ww).
\end{align}
Since \varODAFTRL is an instance of \OAFTRL with 
$\tildeg_{t+1} = \h_{t+1} - \sum_{s=\last(t+1)+1}^t\g_s$,
the following result follows immediately from the \OAFTRL regret bound, \cref{oaftrl_regret}.

\begin{theorem}[Regret of \varODAFTRL]\label{var_odaftrl_regret}
If $\psi$ is nonnegative and $\lam_t$ is non-decreasing in $t$, then,  $\forall\uu \in \wset$, the \varODAFTRL iterates $\ww_t$ satisfy
\begin{talign}
\regret_T(\uu) 
    &\leq 
        \lambda_{T}\reg(\uu) + 
        \sum_{t=1}^T 
    \min(\frac{ \bbftrl{t}}{\lam_{t}}, \abftrl{t}) 
    \qtext{with} \\\label{var_abf_def}
    \bbftrl{t} & 
        \defeq \huber(\staticnorm{\h_t - \sum_{s=\last(t)+1}^t\g_s}_*, \dualnorm{\g_{t}}) \qtext{and} %
        \\
    \abftrl{t} &\defeq \diam{\wset} \min\big(\staticnorm{\h_t - \sum_{s=\last(t)+1}^t\g_s}, \dualnorm{\g_{t}} \big). 
\end{talign}
\end{theorem}

\subsection{Regret of \varDUB}
Consider the \DUB variable-delay generalization
\begin{align}
    \label[name]{var-dub}\tag{DUB with variable delays}
    \alpha \lambda_{t+1}
        = 
        2\max_{j\leq \last(t+1)-1} \abftrl{\last(j+1)+1:j} + \textstyle\sqrt{\sum_{i=1}^{\last(t+1)} \abftrl{i}^2 +2  \alpha \bbftrl{i}}.
\end{align}

\begin{theorem}[Regret of \varDUB] \label{var_dub_regret}
Fix $\alpha > 0$, and, 
for $\abftrl{t},\bbftrl{t}$ as in \cref{var_abf_def}, 
consider the \varDUB sequence.
If $\psi$ is nonnegative, then, for all $\uu \in \wset$, the \varODAFTRL iterates $\ww_t$ satisfy
\begin{talign}
&\regret_T(\uu) 
    \leq \big(\frac{\reg(\uu)}{\alpha}+1\big) \\
&\big(2\max_{t\in [T]} \abftrl{\last(t)+1:t-1} + \textstyle\sqrt{\sum_{t=1}^{T} \abftrl{t}^2 +2  \alpha \bbftrl{t}}\big)
\end{talign}
\end{theorem}
\begin{proof}
Fix any $\uu\in\wset$.  By \cref{var_odaftrl_regret}, \varODAFTRL admits the regret bound 
\begin{talign}
\regret_T(\uu) 
    \leq 
        \lambda_{T} \reg(\uu) + \sum_{t=1}^T \min(\frac{1}{\lam_t} \bbftrl{t}, \abftrl{t}).
\end{talign}
To control the second term in this bound, we apply the following lemma proved in \cref{proof_upper-bound-lam-bound}.

\begin{lemma}[\varDUB-style tuning bound] \label{var-upper-bound-lam-bound}
Fix any $\alpha > 0$ and any non-negative sequences $(a_t)_{t=1}^T$, $(b_t)_{t=1}^T$.
If $(\lambda_{t})_{t\geq 1}$ is non-decreasing and 
\begin{talign} 
\Delta_{t+1}^* 
    \defeq
        2\max_{j\leq \last(t+1)-1} a_{\last(j+1)+1:j} + \sqrt{\sum_{i=1}^{\last(t+1)} a_{i}^2 +2  \alpha b_i} 
    \leq
        \alpha \lambda_{t+1}
        \qtext{for each} t
\end{talign}
then
\begin{talign}
\sum_{t=1}^T \min(b_t / \lambda_t, a_{t})
    \leq 
        \Delta_{\first(T)}^*
    \leq
        \alpha \lambda_{\first(T)}.
\end{talign}
\end{lemma}
\end{proof}
Since $T \leq \first(T)$,  
$\lam_T \leq \lam_{\first(T)}$, 
and 
$\last(\first(T)) = T$, 
the result now follows by setting $a_t = \abftrl{t}$ and $b_t = \bbftrl{t}$, so that \begin{talign}
\regret_T(\uu) \leq  \lambda_{T} \reg(\uu) + \alpha \lam_{\first(T)} \leq (\reg(\uu) + \alpha) \lam_{\first(T)}.
\end{talign}

\subsection{Proof of \cref{var-upper-bound-lam-bound}: \varDUB-style tuning bound}  \label{proof_var-upper-bound-lam-bound}
We prove the claim 
\begin{talign}
\Delta_t \defeq \sum_{i=1}^t \min(b_i / \lambda_i, a_i)
    \leq
\Delta_{\first(t)}^*
    \leq 
        \alpha \lambda_{\first(t)}
\end{talign}
by induction on $t$. 

\paragraph{Base case} 
For $t=1$, since $\last(\first(t)) \geq t$, we have
\begin{talign}
\sum_{i=1}^t \min(b_i / \lambda_i, a_{i}) 
    &\leq
        a_1
    \leq 
        2\max_{j\leq t-1} a_{\last(j+1)+1:j} + \sqrt{\sum_{i=1}^{t} a_i^2 + 2\alpha b_i} \\
    &\leq 
        2\max_{j\leq \last(\first(t))-1} a_{\last(j+1)+1:j} + \sqrt{\sum_{i=1}^{\last(\first(t))} a_{i}^2 + 2  \alpha b_i}
    = 
        \Delta_{\first(t)}^*
    \leq 
        \alpha\lam_{\first(t)}
\end{talign} 
confirming the base case. 

\paragraph{Inductive step}
Now fix any $t + 1 \geq 2$ and suppose that 
\begin{talign}
\Delta_{i} 
\leq 
    \Delta_{\first(i)}^*
\leq 
    \alpha \lambda_{\first(i)}
\end{talign}
for all $1\leq i \leq t$. 
Since $\first(\last(i+1)) \leq i+1$ and $\lam_s$ is non-decreasing in $s$, 
we apply this inductive hypothesis to deduce that, for each $0 \leq i \leq t$, 
\begin{align}
\Delta_{i+1}^2 - \Delta_i^2
    &= \left(\Delta_i + \min(b_{i+1} / \lambda_{i+1} ,a_{i+1}) \right)^2 - \Delta_i^2
    = 2 \Delta_i \min(b_{i+1} / \lambda_{i+1},a_{i+1}) + \min(b_{i+1} / \lambda_{i+1},a_{i+1})^2\\
    &= 2 \Delta_{\last(i+1)} \min(b_{i+1} / \lambda_{i+1},a_{i+1}) + 2 (\Delta_i-\Delta_{\last(i+1)})\min(b_{i+1} / \lambda_{i+1},a_{i+1}) + \min(b_{i+1} / \lambda_{i+1},a_{i+1})^2 \\
    &= 2 \Delta_{\last(i+1)} \min(b_{i+1} / \lambda_{i+1},a_{i+1}) + 2 \sum_{j=\last(i+1)+1}^i \min(b_j/\lam_j, a_j)\min(b_{i+1} / \lambda_{i+1},a_{i+1}) + \min(b_{i+1} / \lambda_{i+1},a_{i+1})^2 \\
    &\leq 2 \alpha \lam_{\first(\last(i+1))} \min(b_{i+1} / \lambda_{i+1},a_{i+1}) + 2 a_{\last(i+1)+1:i}\min(b_{i+1} / \lambda_{i+1},a_{i+1}) + a_{i+1}^2 \\    
    &\leq 2 \alpha \lam_{i+1} \min(b_{i+1} / \lambda_{i+1},a_{i+1}) + 2 a_{\last(i+1)+1:i}\min(b_{i+1} / \lambda_{i+1},a_{i+1}) + a_{i+1}^2 \\    
    &\leq 2 \alpha b_{i+1} + a_{i+1}^2 + 2 a_{\last(i+1)+1:i}\min(b_{i+1} / \lambda_{i+1},a_{i+1}). %
\end{align}
Now, we sum this inequality over $i=0, \dots, t$, to obtain
\begin{talign}
\Delta^2_{t+1} 
    &\leq 
        \sum_{i=0}^{t} (2  \alpha b_{i+1} + a_{i+1}^2) + 2  \sum_{i=0}^{t} a_{\last(i+1)+1:i}\min(b_{i+1} / \lambda_{i+1},a_{i+1}) \\
    &=
        \sum_{i=1}^{t+1} (2  \alpha b_{i} + a_{i}^2) + 2  \sum_{i=1}^{t+1} a_{\last(i+1):i-1}\min(b_{i} / \lambda_{i},a_{i}) \\
    &\leq 
        \sum_{i=1}^{t+1} (a_{i}^2 + 2  \alpha b_i) + 2  \max_{j\leq t} a_{\last(j+1)+1:j} \sum_{i=1}^{t+1} \min(b_i / \lambda_{i}, a_{i}) \\
    &=
        \sum_{i=1}^{t+1} (a_{i}^2 + 2  \alpha b_i) + 2  \Delta_{t+1} \max_{j\leq t} a_{\last(j+1)+1:j} .
\end{talign}
We now solve this quadratic inequality, apply the triangle inequality, and invoke the relation 
$\last(\first(t+1)) \geq t+1$
to conclude that
\begin{talign}
\Delta_{t+1} 
    &\leq 
        \max_{j\leq t} a_{\last(j+1)+1:j}+ \half \sqrt{(2\max_{j\leq t} a_{\last(j+1)+1:j})^2 + 4\sum_{i=1}^{t+1} a_{i}^2 + 2  \alpha b_i} \\
    &\leq 
        2\max_{j\leq t} a_{\last(j+1)+1:j} + \sqrt{\sum_{i=1}^{t+1} a_{i}^2 + 2  \alpha b_i} \\
    &\leq 
        2\max_{j\leq \last(\first(t+1))-1} a_{\last(j+1)+1:j} + \sqrt{\sum_{i=1}^{\last(\first(t+1))} a_{i}^2 + 2  \alpha b_i}
    = 
        \Delta_{\first(t+1)}^*
    \leq 
        \alpha \lambda_{\first(t+1)}.
\end{talign}

\subsection{Regret of \varAdaHedgeD}
Consider the \AdaHedgeD variable-delay generalization
\begin{talign}
\label[name]{var-adahedged}\tag{AdaHedgeD with variable delays}
&\lam_{t+1} 
    = \frac{1}{\alpha}\sum_{s=1}^{\last(t+1)} \delta_s 
    \qtext{for}
\delta_t \qtext{defined in \cref{def-deltat}.}
\end{talign}

\begin{theorem}[Regret of \varAdaHedgeD] \label{var_adahedged_regret}
Fix $\alpha > 0$, and consider the
\varAdaHedgeD sequence.
If $\psi$ is nonnegative, then, for all $\uu \in \wset$, the \varODAFTRL iterates satisfy
\begin{talign}
&\regret_T(\uu) 
    \leq \big( \frac{\reg(\uu)}{\alpha} + 1 \big) \\
& \big(2\max_{t\in[T]} \abftrl{\last(t+1)+1:t} 
            + \sqrt{\sum_{t=1}^{T} \abftrl{t}^2 + 2 \alpha \bbftrl{t}}  \big).
\end{talign}
\end{theorem}

\begin{proof}
Fix any $\uu\in\wset$, and for each $t$, define $\lam'_{t+1} = \frac{1}{\alpha}\sum_{s=1}^{t} \delta_s$ 
so that $\alpha(\lam'_{t+1} - \lam'_t) = \delta_{t}$.
Since the \varAdaHedgeD regularization sequence $(\lam_t)_{t\geq 1}$ is non-decreasing, $\last(T) \leq T$, and hence $\lam_T \leq \lam'_{T+1}$, \cref{oaftrl_regret} gives the regret bound
\begin{talign}
\regret_T(\uu) 
    &\leq 
        \lam_{T} \reg(\uu) + \sum_{t=1}^T \delta_t 
    \leq
        \lam_{T} \reg(\uu) + \alpha\lam'_{T+1}
    \leq
        (\reg(\uu) + \alpha)\lam'_{T+1}
\end{talign}
and the proof of \cref{oaftrl_regret} gives the upper estimate \cref{delta_min_bound}:
\begin{talign}
\label{var_delta_a_b_bound}
\delta_t 
    \leq 
        \min\Big(\frac{\bbftrl{t}}{ \lam_t}, \abftrl{t} \Big)
    \qtext{for all} t \in [T].
\end{talign}
Hence, it remains to bound $\lam'_{T+1}$. 
We have
\begin{talign}
\alpha {\lam'_{T+1}}^2 
    &= 
        \sum_{t=1}^{T} \alpha({\lam'_{t+1}}^2 - {\lam'_{t}}^2) 
    = 
        \sum_{t=1}^{T} \left(\alpha(\lam'_{t+1} - \lam'_{t})^2 + 2\alpha(\lam'_{t+1} - \lam'_{t}) \lam'_{t} \right)\\
    & = 
        \sum_{t=1}^{T} \left(\delta_t^2/\alpha + 2 \delta_t \lam'_{t} \right)  \qtext{by the definition of $\lam'_{t+1}$}\\
    &= 
        \sum_{t=1}^{T} \left(\delta_t^2/\alpha 
        + 2 \delta_t \lam_{t}
        + 2 \delta_t (\lam'_{t} - \lam_{t})
        \right) \\%
    &\leq
        \sum_{t=1}^{T}\left(\delta_t^2/\alpha 
        + 2 \delta_t \lam_{t}
        + 2 \delta_t \max_{t\in[T]} (\lam'_{t} - \lam_{t})
        \right) \\
    &=
        \sum_{t=1}^{T}\left(\delta_t^2/\alpha 
        + 2 \delta_t \lam_{t}\right)
        + 2\alpha\lam'_{T+1} \max_{t\in[T]} (\lam'_{t} - \lam_{t}) \\   
    &=
        \sum_{t=1}^{T}\left(\delta_t^2/\alpha 
        + 2 \delta_t \lam_{t}\right)
        + 2\lam'_{T+1} \max_{t\in[T]} \delta_{\last(t+1)+1:t} \\  
    &\leq
        \sum_{t=1}^{T}\left(\abftrl{t}^2/\alpha 
        + 2 \bbftrl{t}\right)
        + 2\lam'_{T+1} \max_{t\in[T]} \abftrl{\last(t+1)+1:t} 
        \qtext{by \cref{var_delta_a_b_bound}.}
\end{talign}
Solving the above quadratic inequality for $\lam'_{T+1}$ and applying the triangle inequality, we find
\begin{talign}
\alpha\lam'_{T+1}
    &\leq \max_{t\in[T]} \abftrl{\last(t+1)+1:t} 
    + \half\sqrt{4(\max_{t\in[T]} \abftrl{\last(t+1)+1:t} )^2 + 4\sum_{t=1}^{T} \abftrl{t}^2 + 2 \alpha \bbftrl{t}} \\
    &\leq 2\max_{t\in[T]} \abftrl{\last(t+1)+1:t}  
    + \sqrt{\sum_{t=1}^{T} \abftrl{t}^2 + 2 \alpha\bbftrl{t}}.
\end{talign}
\end{proof}

\end{document}